\documentclass[twoside,11pt]{article}

\usepackage{blindtext}

\usepackage{thm-restate}
\usepackage{graphicx}
\usepackage{jmlr2e}

\usepackage[table]{xcolor}

\usepackage{color}
\usepackage{tikz}
\usetikzlibrary{automata,arrows,positioning,calc}
\usepackage{tkz-graph}
\usepackage{graphicx}

\usepackage{wrapfig}

\usepackage{algpseudocode}

\usepackage{amsmath, amssymb,mathtools}
\usepackage{thmtools}

\usepackage{url}

\usepackage{ifthen}

\usepackage{multirow}

\usepackage{float}
\usepackage[section]{placeins}
\usepackage{booktabs}
\usepackage{enumitem}
\setlist[enumerate]{itemsep=0mm}

\usepackage{lastpage}
\jmlrheading{27}{2026}{1-\pageref{LastPage}}{2/24; Revised 7/25}{6/26}{24-0260}{Ehsan Imani, Kai Luedemann, Sam Scholnick-Hughes, Esraa Elelimy, Martha White}

\ShortHeadings{Investigating the Histogram Loss in Regression}{Imani, Luedemann, Scholnick-Hughes, Elelimy, White}
\firstpageno{1}

\usepackage{support-caption}
\usepackage{subcaption}

\usepackage{custom_commands}

\allowdisplaybreaks

\begin{document}

\title{Investigating the Histogram Loss in Regression}

\author{\name Ehsan Imani \email imani@ualberta.ca \\
       \name Kai Luedemann\thanks{These authors contributed equally.} \thanks{The author is currently affiliated with the University of Tübingen and Zuse School ELIZA.} \email kluedema@ualberta.ca \\
       \name Sam Scholnick-Hughes\footnotemark[1] \email scholnic@ualberta.ca \\
       \name Esraa Elelimy\footnotemark[1] \email elelimy@ualberta.ca \\
       \name Martha White \email whitem@ualberta.ca \\
       \addr 
       Alberta Machine Intelligence Institute (Amii) and\\
       Reinforcement Learning and Artificial Intelligence Laboratory\\
       Department of Computing Science, University of Alberta\\
       Edmonton, Alberta, Canada T6G 2E8}

\editor{Kilian Weinberger}

\maketitle

\begin{abstract}It is becoming increasingly common in regression to train neural networks that model the entire distribution even if only the mean is required for prediction. This additional modeling often comes with performance gains, and the reasons behind the improvement are not fully known. This paper investigates a recent approach to regression, the \emph{histogram loss}, which involves learning the conditional distribution of the target variable by minimizing the cross-entropy between a target distribution and a flexible histogram prediction. The resulting loss corresponds to a classification loss: a cross-entropy  between the outputs and a smoothed label vector. We design theoretical and empirical analyses to determine why and when this performance gain appears and how different components of the loss contribute to it. Our results suggest that the benefits of learning distributions in this setup come from improvements in optimization rather than modeling extra information. We then demonstrate the viability of the histogram loss in common deep learning applications without the need for costly hyperparameter tuning.\footnote{An implementation is available at \url{https://github.com/marthawhite/Histogram_loss}.}
\end{abstract}


\section{Introduction}

In machine learning, the loss function is a crucial design choice that can influence the optimization landscape and the quality of solutions found by gradient descent. For regression tasks, where the goal is to predict a continuous target $Y$ from inputs $\xvec$, a common choice is the squared-error loss ($\ell_2$). This loss is theoretically motivated, as minimizing $\ell_2$ corresponds to maximum likelihood estimation when the conditional distribution of $Y \mid \xvec$ is Gaussian with fixed variance.

However, $\ell_2$ is not the only choice for regression. There are also statistically motivated alternatives. For example, loss functions such as $\ell_1$ loss, Huber loss, and Charbonnier loss have shown improvements in the presence of outliers \citep{huber2011robust,charbonnier1994two,barron2017amore}. In deep learning, the loss also plays a role in training by influencing the geometry of the optimization landscape. For nonconvex problems, a loss that is more easily optimized may allow gradient descent to converge to a better solution.

A more radical alternative is to reparameterize the model to capture a fuller picture of the output distribution. Examples are predicting the means, variances, and coefficients of a multi-modal mixture of Gaussian distributions \citep{bishop1994mixture}, weights of a histogram \citep{rothe2015dex}, quantiles of a distribution \citep{dabney2018distributional}, and a generative model for sampling from the distribution \citep{dabney2018implicit}. In the deep learning literature, such reparameterizations have been beneficial in applications such as age estimation, depth estimation, and value prediction for decision-making \citep{gao2017deep,reading2021categorical,bellemare2017distributional,hessel2017rainbow,bellemare2023distributional}. Understanding why these approaches can be superior to the squared loss, even when only the mean of the predicted distribution is needed, is an ongoing endeavor, and a common explanation is that the extra difficulty helps the network learn a better representation \citep{bellemare2017distributional,lyle2019comparative,lyle2021effect,lan2022generalization}.

\begin{wrapfigure}[15]{r}{0.49\textwidth}
\centering
 \begin{subfigure}[b]{0.5\textwidth}
         \includegraphics[width=\textwidth]{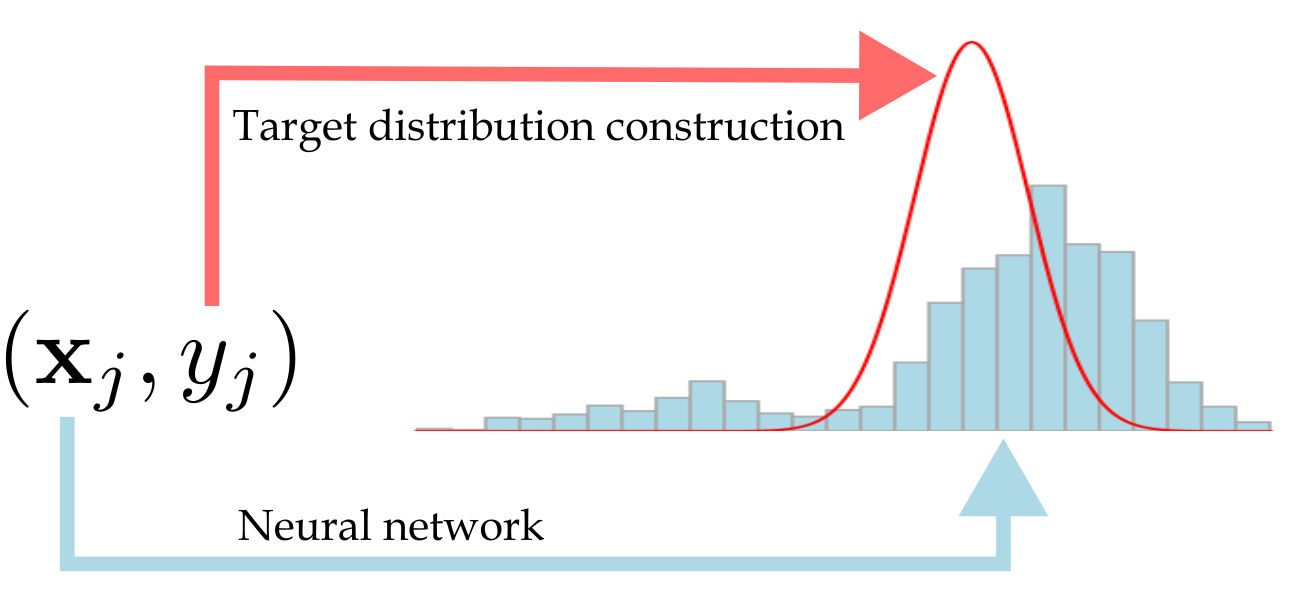}
\end{subfigure}
\caption{Training with HL. The red curve is the target distribution, and the blue histogram is the prediction distribution. The neural network is trained to minimize the cross-entropy between the two.}
\label{fig:hl_vis}
\end{wrapfigure}

A shorter version of this paper by \cite{imani2018improving} proposed a family of loss functions called the \emph{histogram loss} (HL). In short, a model trained with HL outputs a \textit{prediction distribution}, a histogram parameterized by a softmax function. Each scalar target in the training data is converted to a \textit{target distribution}, and the model aims to match the histogram to the target distribution by reducing the cross-entropy between the two. The process is shown in Figure \ref{fig:hl_vis}. This loss was motivated by benefits found in distributional reinforcement learning (RL) but investigated in supervised learning to better understand the phenomenon. 
\cite{imani2018improving} motivated that HL might provide performance improvements because of better-behaved gradient norms and provided a small-scale empirical study towards understanding why HL can outperform the squared error loss.
Recent work by~\cite{pmlr-v235-farebrother24a} has shown the utility of HL beyond supervised learning tasks. They showed that using HL for training the value function in RL resulted in a consistent performance improvement across many tasks and improved the scalability of deep RL methods with large networks.

In this paper, we expand the theoretical understanding of this loss, delve deeper empirically into the reasons behind the benefits of HL, explore the viability of HL in larger-scale deep learning tasks, and provide rules of thumb for its hyperparameters. We will also include the results by \cite{imani2018improving} for completeness. The following three paragraphs summarize the novel contributions.

For theoretical motivation, \cite{imani2018improving} proved an upper bound on the local Lipschitz constant for HL, showing it is smaller than for $\ell_2$; it is known that a smaller Lipschitz constant improves generalization performance with stochastic gradient descent. This previous work also highlighted a connection to maximum entropy reinforcement learning, which can promote efficient search. In this work, we focus on characterizing the \emph{discretization bias} introduced by moving from the $\ell_2$ loss to HL. For a larger number of output bins, the bias is negligible, and it increases with a small number of bins. We bound the bias based on bin width and further show that, even with a small number of bins, bias can be made small by appropriately selecting a variance parameter. We also motivate HL as a surrogate objective by showing that minimizing HL reduces an upper bound on the squared difference between the predictions and the targets.

We conduct a larger empirical study compared to \cite{imani2018improving} to test hypotheses regarding the performance of HL. On four regression data sets, we test
\begin{enumerate}
\item if the choice of target distribution has a major impact on the bias of HL that can explain the difference in error rates, 
\item if properties of the histogram prediction and target distribution give rise to a bias-variance trade-off,
\item if the challenge of predicting a full distribution forces the model to learn a better representation, 
\item if switching from the $\ell_2$ loss to HL 
makes the landscape smoother
\item if HL is more robust to corrupted targets in the data set, and 
\item if HL finds a model whose output is less sensitive to input perturbations. 
\end{enumerate}

Finally, we show the viability of HL on large-scale time series and value prediction tasks with different architectures and provide heuristics for choosing the hyperparameters of the loss, namely the number of bins and the amount of padding in the histogram and the degree of smoothing in the target distribution. We show that a fixed set of recommended hyperparameter values works well in a diverse set of experiments. Overall, this paper provides more insights into HL and clarifies some of the choices within this loss to make it easy to use more broadly for regression.

\section{Distributional Losses for Regression} \label{sec:dist_loss}

In this section we introduce the histogram loss (HL). At a high level, defining and using HL is simple: 1) the user selects the number of outputs $k$ (bins) to discretize the output space, 2) the regression labels $y$ in the data set are converted to classification labels (probability vectors of length $k$ that sum to one), and 3) the standard cross-entropy loss is used between the $k$ outputs and these new targets. The learned network is used for prediction by using the weighted average of outputs with their corresponding bin centers. We first introduce how this loss is constructed, explaining how to define these new classification labels. The construction is defined by the choice of target distribution. We advocate for a Gaussian target distribution as a suitable default in this work but include other examples (e.g., uniform) both to better understand the loss construction and because later we empirically compare these alternatives. This section only explains the construction of this loss; later sections motivate why this loss is useful.

\subsection{Defining the Histogram Loss}

Consider predicting a continuous target $Y$ with event space $\mathcal{Y}$, given inputs $\xvec$. Recall that directly predicting $Y$ is possible by minimizing the squared error on samples $(\xvec_j, y_j)$ where input $\xvec_j \in \RR^{\xdim}$ is associated with target $y_j \in \RR$. Instead of directly predicting $Y$, we can learn a distribution on $Y | \xvec$ whose expected value will be used as the prediction. Assume we have samples $(\xvec_j, \td_j)$, where each input $\xvec_j$ is associated with a target distribution with the probability density function (pdf) $\td_j: \mathcal{Y} \rightarrow [0,\infty)$. We would like to learn a parameterized prediction distribution $\pd: \mathcal{Y} \rightarrow [0,\infty)$, conditioned on $\xvec$, by minimizing a measure of difference between $\pd$ and $\td$.

\begin{wrapfigure}[20]{r}{0.35\textwidth}
\centering
 \begin{subfigure}[b]{0.35\textwidth}
         \includegraphics[width=\textwidth]{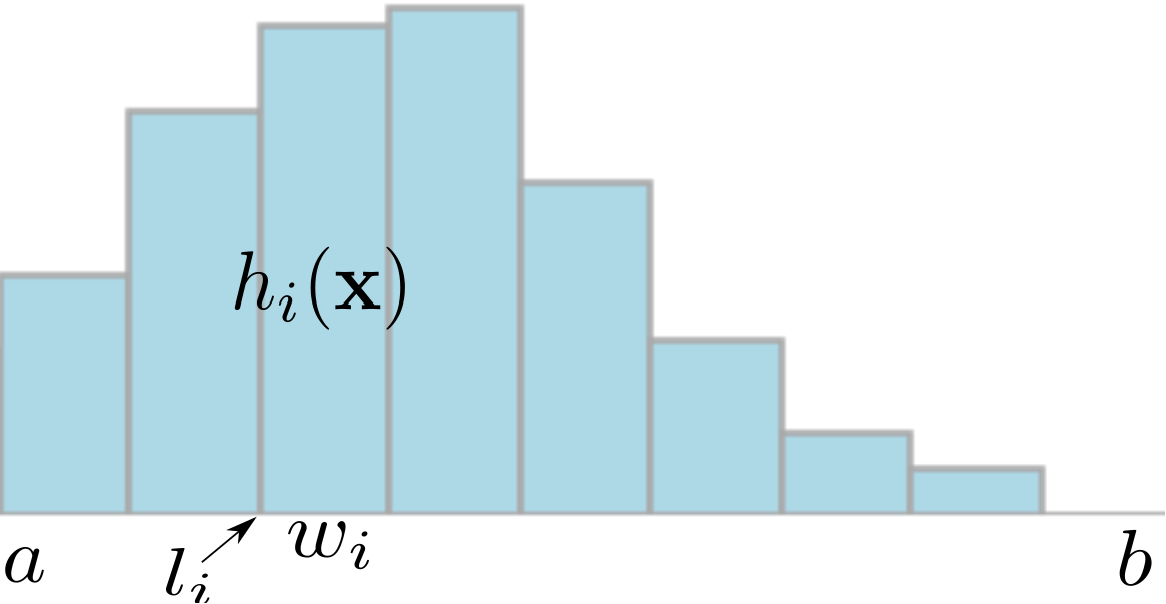}
\end{subfigure}
 \begin{subfigure}[b]{0.35\textwidth}
         \includegraphics[width=\textwidth]{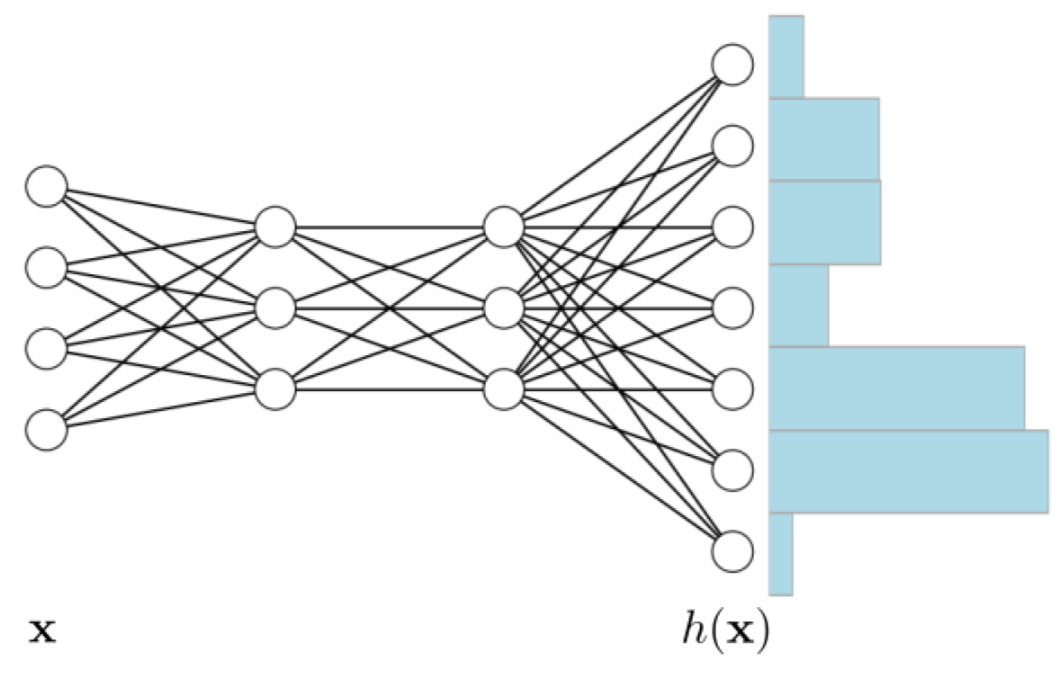}
\end{subfigure} 
\caption{(Top) A sample histogram, and (bottom) a neural network with a softmax output layer that represents a histogram.}
\label{fig:sample_histogram_nn}
\end{wrapfigure}

Three choices have to be made to learn a distribution on $Y | \xvec$ in this setting. First, we need a parameterization for $\pd$, which is a continuous pdf. Second, we require a function that measures how the prediction distribution is different from the target distribution. Third, we need a method to convert a regression data set with samples $(\xvec_j, y_j)$ with real-valued outputs to one with samples $(\xvec_j, \td_j)$ whose outputs are target distributions. Below we detail the first two choices, and in the next section we describe several methods for constructing target distributions.

We propose to restrict the prediction distribution $\pd$ to be a histogram density.
Assume an interval $[a,b]$ on the real line has been uniformly partitioned into $\nbins$ bins of width $w_i$, and let function $\hist: \mathcal{X} \rightarrow [0,1]^\nbins$ provide the $\nbins$-dimensional vector $\hist(\xvec)$ of coefficients indicating the probability the target is in each bin, given $\xvec$. The distribution $\pd$ corresponds to a (normalized) histogram and has density values $\hist_i(\xvec)/w_i$ per bin. A sample histogram is depicted in Figure \ref{fig:sample_histogram_nn} (top). The histogram prediction distribution $\pd$ can be parameterized with a softmax output layer in a neural network. The softmax layer has $\nbins$ units, and the value of unit $i$ represents $\hist_i(\xvec)$.  Figure \ref{fig:sample_histogram_nn} (bottom) shows a sample neural network whose output represents a histogram.

For the measure of difference, we use the KL-divergence. Suppose the target distribution has CDF $\tc$, and the support of this distribution is within the range $[a,b]$. The KL-divergence between $\td$ and $\pd$ is
\begin{align*}
D_{KL}(\td||\pd) &= - \int_{a}^{b} \td(y) \log \frac{\pd(y)}{\td(y)} dy \\
&= - \int_{a}^{b} \td(y) \log \pd(y) dy - \biggl(-  \int_{a}^{b} \td(y) \log \td(y) dy \biggr).
\end{align*}
Because the second term only depends on $\td$, the aim is to minimize the first term: the cross-entropy between $\td$ and $\pd$. 
This loss simplifies due to the histogram form on $\pd$:
\begin{align*}
- \int_{a}^{b} \td(y) \log \pd(y) dy &= - \sum_{i=1}^\nbins \int_{l_i}^{l_i + w_i} \td(y) \log \frac{\hist_i(\xvec)}{w_i} dy \\
&= - \sum_{i=1}^\nbins \log \frac{\hist_i(\xvec)}{w_i} \underbrace{(\tc(l_i + w_i) - \tc(l_i))}_{\distweight_i}.
\end{align*}
In the minimization, the width itself can be ignored because $\log \tfrac{\hist_i(\xvec)}{w_i} = \log \hist_i(\xvec) - \log w_i$, giving the histogram loss 
\begin{equation}
\text{HL}(\td,\pd) = - \sum_{i=1}^\nbins \distweight_i \log \hist_i(\xvec) \label{eq_hl}
.
\end{equation}
This loss has several useful properties. One important property is that it is convex in $\hist_i(\xvec)$; even if the loss is not convex in all network parameters, it is at least convex on the last layer. The other two benefits are due to restricting the form of the predicted distribution $\pd$ to be a histogram density. 
First, the divergence to the full distribution $\td$ can be efficiently computed. Second, the choice of $\td$ is flexible, as long as its CDF can be evaluated for each bin. The weighting $\distweight_i = \tc(l_i + w_i) - \tc(l_i)$ can be computed offline once for each sample, making it inexpensive to query repeatedly for each sample during training.

\subsection{Target Distributions}

The histogram loss requires a target distribution corresponding to each input. The method for constructing a target distribution from the target $y_j$ in the data set can be chosen upfront. Different choices simply result in different weightings in HL. Below, we consider some special cases that are of interest.

\myparagraph{Truncated Gaussian on $Y | \xvec$ (HL-Gauss).}
Consider a truncated Gaussian distribution, on support $[a,b]$, as the target distribution. The mean $\mu$ for this Gaussian is the data point $y_j$ itself, with fixed variance $\sigma^2$. The pdf $\td$ is
\begin{equation*}
\td(y) = \frac{1}{Z\sigma \sqrt{2 \pi}} e^{-\frac{(y-y_j)^2}{2\sigma^2}}
\end{equation*}
where $Z = \frac{1}{2} (\text{erf}\left(\frac{b-y_j}{\sqrt{2}\sigma}\right) - \text{erf}\left(\frac{a-y_j}{\sqrt{2}\sigma}\right))$,
and HL has
\begin{equation*}
\distweight_i = \tfrac{1}{2Z} \left(\text{erf}\left(\frac{l_i + w_i - y_j}{\sqrt{2}\sigma}\right) - \text{erf}\left(\frac{l_i - y_j}{\sqrt{2}\sigma}\right) \right)
.
\end{equation*}
This distribution enables smoothing over $Y$ through the variance parameter $\sigma^2$. We call this loss HL-Gauss, which is defined by the number of bins $\nbins$ and variance $\sigma^2$. Based on positive empirical performance, it is the main HL variant that we advocate for and analyze. Combined with Equation \eqref{eq_hl}, the final form for this loss is:
\begin{equation*}
    \text{HL-Gauss}(y_j, \pd) = - \sum_{i=1}^\nbins \tfrac{1}{2Z} \left(\text{erf}\left(\frac{l_i + w_i - y_j}{\sqrt{2}\sigma}\right) - \text{erf}\left(\frac{l_i - y_j}{\sqrt{2}\sigma}\right) \right) \log \hist_i(\xvec).
\end{equation*}

\myparagraph{Dirac delta on $Y | \xvec$ (HL-OneBin).} Consider Gaussians centered around data points $y_j$, with arbitrarily small variances $\tfrac{1}{2}a^2$: 
\begin{equation*}
\delta_{a,j}(y) = \frac{1}{a \sqrt{\pi}}\exp\left(-\tfrac{(y - y_j)^2}{a^2}\right). 
\end{equation*}
Let the target distribution have $\td(y) = \delta_{a,j}(y)$ for each sample. Define function $\distweight_{i,j}: [0,\infty) \rightarrow [0,1]$ as 
$\distweight_{i,j}(a) = \int_{l_i}^{l_i + w_i} \delta_{a,j}(y) dy$
.
For each $y_j$, as $a \rightarrow 0$, $\distweight_{i,j}(a) \rightarrow 1$ if $y_j \in [l_i, l_i + w_i]$ and $\distweight_{i,j}(a) \rightarrow 0$ otherwise. 
So, for $\text{bin}_j$ s.t. $y_j \!\in [l_{\text{bin}_j}, l_{\text{bin}_j} \!+\! w_{\text{bin}_j}]$, \begin{align*}
\!\!\lim_{a \rightarrow 0} \! HL(\delta_{a,j},\pdnox_{\xvec_j}) &= - \log \hist_{\text{bin}_j}(\xvec_j).
\end{align*}
The sum over samples for HL to the Dirac delta on $Y | \xvec$, then, corresponds to the negative log-likelihood for $\pd$
\begin{align*}
\argmin_{\hist_1, \ldots, h_\nbins} - \sum_{j=1}^\nsamples \log \hist_{\text{bin}_j}(\xvec_j)
&= \argmin_{\hist_1, \ldots, h_\nbins} - \sum_{j=1}^\nsamples \log \pdnox_{\xvec_j}(y_j)
.
\end{align*}
Such a delta distribution on $Y | \xvec$ results in one coefficient $\distweight_i$ being 1, reflecting the distributional assumption that $Y$ is certainly in a bin.
In the experiments, we compare to this loss, which we call HL-OneBin. This makes the final form of the loss:
\begin{equation*}
    \text{HL-OneBin}(y_j, \pd) = - \log \hist_{\text{bin}_j}(\xvec).
\end{equation*}

\myparagraph{Mixture of Dirac delta and uniform on $Y | \xvec$ (HL-Uniform).} Using a similar analysis to above, $\td(y)$ can be considered as a mixture between $\delta_{a,j}(y)$ and a uniform distribution. For a weighting of $\epsilon$ on the uniform distribution, the resulting loss HL-Uniform has $\distweight_i = \epsilon$ for $i \neq i_j$ and $\distweight_{i_j} = 1- (\nbins-1) \epsilon$. The loss can be expressed as:
\begin{equation*}
    \text{HL-Uniform}(y_j, \pd) = - (1-(k-1)\epsilon) \log \hist_{\text{bin}_j}(\xvec) - \sum_{i\neq \text{bin}_j} \epsilon \log \hist_{\text{bin}_j}(\xvec).
\end{equation*}

\section{Theoretical Analysis}  \label{sec:theoretical_analysis}

This section assembles our theoretical results on the histogram loss (HL). The first part considers the behavior of gradient descent optimization on HL and the $\ell_2$ loss. Gradient descent is a local search method, and problems like the existence of highly varying regions, unfavorable local minima and saddle points, and poorly conditioned coordinates can hinder the search. A complete characterization of the loss surface is a hard problem. Instead, we provide an upper bound on the norm of the gradient of HL, which suggests that the gradient varies less through training and provides more reliable optimization steps. We also point out how this norm relates to the generalization performance of the model. 
The second part shows a similarity between minimizing HL and entropy-regularized policy improvement based on work by \citet{norouzi2016reward}.

The third part discusses the bias of the minimizer of HL. The mean of the target generating distribution can be different from a coarse histogram prediction distribution that approximates it. We show that when HL-Gauss is minimized, this difference is bounded by half of the bin width as long as the target distribution has negligible probability beyond the support of the histogram. The last part shows that minimizing HL on the data reduces an upper bound on the difference between the targets and the mean of prediction distributions.

\subsection{Stable Gradients for the Histogram Loss} \label{opt_theory}

\future{Can this be a future to do? We should at least thing about if it makes the result meaningless, and see what Hardt says about it.}
\citet{hardt2015train} have shown that the generalization performance for stochastic gradient descent is bounded by the number of steps
that stochastic gradient descent takes during training, even for non-convex losses. 
The bound is also dependent on the properties of the loss.
In particular, it is beneficial to have a loss function with small Lipschitz constant $L$, which bounds the norm of the gradient.
Below, we discuss how HL with a Gaussian distribution (HL-Gauss) in fact promotes an improved bound on this norm, over both the $\ell_2$ loss and HL with all weight in one bin (HL-OneBin). 

\newcommand{\lip}{l}

In the proposition bounding the HL-Gauss gradient, we assume 
\begin{equation}
\hist_i(\xvec) = \tfrac{\exp(\phivec_\thetavec(\xvec)^\top \wvec_i)}{\sum_{j=1}^\nbins \exp(\phivec_\thetavec(\xvec)^\top \wvec_j)} \label{eq_softmax}
\end{equation}
for some function $\phivec_\thetavec: \mathcal{X} \rightarrow \mathcal{R}^k$ parameterized by a vector of parameters $\thetavec$. For example, $\phivec_\thetavec(\xvec)$ could be the last hidden layer in a neural network, with parameters $\thetavec$ for the entire network up to that layer.  
The proposition provides a bound on the gradient norm in terms of the \emph{current network parameters}. Our goal is to understand how the gradients might vary \emph{locally for the parameters}, as opposed to globally bounding the norm and characterizing the Lipschitz constant only in terms of the properties of the function class and loss.

\begin{restatable}[Local Lipschitz constant for HL-Gauss]{proposition}{proplipschitz}
Assume $\xvec, y$ are fixed, giving fixed coefficients $\distweight_i$ in HL-Gauss.
Let $\hist_i(\xvec)$ be as in \eqref{eq_softmax}, defined by the parameters $\wvec = \{\wvec_1, \ldots, \wvec_\nbins\}$ and $\thetavec$, providing the predicted distribution $\pd$. 
Assume for all $i$ that $\wvec_i^\top \phivec_\thetavec(\xvec)$ is locally $\lip$-Lipschitz continuous with respect to $\thetavec$: \begin{equation}
\|  \nabla_\thetavec (\wvec_i^\top \phivec_\thetavec(\xvec)) \| \le \lip \label{eq_lipschitz}.
\end{equation}
Then the norm of the gradient for HL-Gauss, with respect to all the parameters in the network $\{\thetavec, \wvec\}$, is bounded by
\begin{equation}
\|  \nabla_{\thetavec,\wvec} HL(\td, \pd)  \|
 \le \left(\lip +  \| \phivec_\thetavec(\xvec) \| \right)  \sum_{i=1}^\nbins | \distweight_i - \hist_i(\xvec)|  \label{eq_hl_lipschitz}.
\end{equation}
\end{restatable}

The results by \citet{hardt2015train} suggest it is beneficial for the local Lipschitz constant---or the norm of the gradient---to be small on each step.
HL-Gauss provides exactly this property. Besides the network architecture---which we are here assuming is chosen outside of our control---the HL-Gauss gradient norm is proportional to $|\distweight_i -\hist_i(\xvec)|$. This number is guaranteed to be less than 1 but is generally likely to be even smaller, especially if $\hist_i(\xvec)$ reasonably accurately predicts $\distweight_i$. Further, the gradients should push the weights to stay within a range specified by $\distweight_i$, rather than preferring to push some to be very small---close to 0---and others to be close to 1. For example, if $\hist_i(\xvec)$ starts relatively uniform, then the objective does not encourage predictions $\hist_i(\xvec)$ to get smaller than $\distweight_i$. If $\distweight_i$ are non-negligible, this keeps $\hist_i(\xvec)$ away from zero and the loss in a smaller range.  

This contrasts both the norm of the gradient for the $\ell_2$ loss and HL-OneBin. For the $\ell_2$ loss, $(f(\xvec) - y) \inlinevec{\nabla_\thetavec \wvec^\top \phivec_\thetavec(\xvec)}{\phivec_\thetavec(\xvec)}$ is the gradient, giving gradient norm bound $(\lip +  \| \phivec_\thetavec(\xvec) \|) | f(\xvec) - y|$. The constant $| f(\xvec) - y|$, as opposed to $\sum_{i=1}^\nbins |\distweight_i -\hist_i(\xvec)|$, can be much larger, even if $y$ is normalized between $[0,1]$, and can vary considerably more. 
HL-OneBin, on the other hand, shares the same constant as HL-Gauss but suffers from another problem. The Lipschitz constant $\lip$ in Equation \eqref{eq_lipschitz} will likely be larger because $\distweight_i$ is frequently zero and so pushes $\hist_i(\xvec)$ towards zero. This results in larger objective values and pushes $\wvec_i^\top \phivec_\thetavec(\xvec)$ to get larger to enable $\hist_i(\xvec)$ to get close to 1.

\subsection{Connection to Reinforcement Learning} 
HL can also be motivated through a connection to maximum entropy reinforcement learning. 
In reinforcement learning, an agent iteratively selects actions and transitions between states to maximize (long-term) reward. 
The agent's goal is to find an optimal policy in as few interactions as possible. To do so, the agent begins by exploring more to facilitate efficiently finding a better policy. Supervised learning can be expressed as a reinforcement learning problem \citep{norouzi2016reward}, where action selection conditioned on a state corresponds to making a prediction conditioned on a feature vector. An alternative view to minimizing prediction error is to search for a policy to make accurate predictions. 

One strategy to efficiently find an optimal policy is through a maximum entropy objective. The policy balances between selecting the action it believes to be optimal---making its current best prediction---and acting more randomly with high entropy. For a continuous action set $\mathcal{Y}$, the goal is to 
minimize the following objective
\begin{equation}
\int_{\mathcal{X}} p_s(\xvec) \Big[ - \tau \ent(\pd) -  \int_{\mathcal{Y}} \pd(y) r(y,y_i) dy \Big] d\xvec \label{eq_rl},
\end{equation}
where $\tau > 0$; 
$p_s$ is a distribution over states $\xvec$; $\pd$ is the policy or distribution over actions for a given $\xvec$; $\ent(\cdot)$ is the differential entropy---the extension of entropy to continuous random variables; and $r(y, y_i)$ is the reward function, such as the negative of the objective $r(y,y_i) = - \tfrac{1}{2} ( y - y_i )^2$. Minimizing \eqref{eq_rl} corresponds to minimizing the KL-divergence across $\xvec$ between $\pd$ and the exponentiated payoff distribution 
$\td(y) = \frac{1}{Z} \exp(r(y,y_i) / \tau)$
where $Z = \int \exp(r(y,y_i) / \tau)$ because 
\begin{align*}
D_{KL}(\pd || \td) 
&= -\ent(\pd) - \int \pd(y) \log \td(y) dy \\
&= - \ent(\pd) - \tau^{-1} \int \pd(y) r(y,y_i) dy + \log Z.
\end{align*}
The connection between HL and maximum-entropy reinforcement learning is that both are minimizing a divergence to this exponentiated distribution $q$.
HL, however, is minimizing $D_{KL}(\td || \pd)$ instead of $D_{KL}(\pd || \td)$. For example, Gaussian target distribution with variance $\sigma^2$ corresponds to minimizing $D_{KL}(\td || \pd)$ with $r(y,y_i) = - \tfrac{1}{2} ( y - y_i )^2$ and $\tau = \sigma^2$. 
These two KL-divergences are not the same, but a similar argument to \citet{norouzi2016reward} could be extended for continuous $y$, showing that $D_{KL}(\pd || \td)$ is upper-bounded by $D_{KL}(\td || \pd)$ plus variance terms. The intuition, then, is that minimizing HL is promoting an efficient search for an optimal (prediction) policy.

\future{Something we could add here is to actually extend their argument to continuous y. This seems nontrivial, and its not clear this explains the behaviour, so it might not be worth it. But, if doable, it would almost be of independent interest.}

\subsection{Bias of the Histogram Loss} \label{bias_theory}

Different forms of bias can be induced by components of the histogram loss, namely using histogram densities, using the KL-divergence, and the choice of target distribution. This section explores the effects of these components on the bias. First, we characterize the minimizer of HL and the effect of target distribution on it. Then, we show to what extent the mean of the minimizer differs from the Bayes optimal solution.

We make two assumptions in this section to simplify the analysis and highlight the key points. First, we assume that the model is flexible enough so that the predicted distribution for each input can be optimized independently. Therefore, we consider the case where all the data points in the data set have the same input and drop the subscripts that show dependence on $\xvec$. Second, we assume that the bins have equal width $w$ and the support of the histogram is unbounded, and therefore the target distribution is not truncated. For a fixed placement and indexing of the bins, we denote the set of distributions with this form as $\mathcal{P}({\NN})$ and the set of distributions without the histogram form restriction as $\mathcal{P}({\RR})$. Later in the experiments, we pad the support of the histogram to ensure that there is negligible truncation.

In the data, each input can be associated with one or multiple targets. More generally, we consider a \textit{target generating distribution} with pdf $\rtd: \mathcal{Y} \rightarrow [0,\infty)$ for an input. This is the underlying distribution that generates the targets given an input, and it is different from the target distribution that is used in HL. In the data set, targets $y_j$ are samples from $\rtd$ and each one is turned into a \textit{target distribution} with pdf $\td_j$ (e.g. a truncated Gaussian centered at $y_j$ if we use HL-Gauss, or a Dirac delta at $y_j$ if we use HL-OneBin) which is then used to train the model with HL. We use $\tdy_y$ to denote the pdf of the target distribution obtained from the target $y$. Finally, the model's predicted distribution for the input is denoted by $\pdnox$.

In this setup, the predictor that is Bayes optimal in squared error predicts $\EX_\rtd[y]$. We define the bias of a predicted distribution $\pdnox^*$ as $\text{Bias}(\pdnox^*, \rtd) = \EX_{\pdnox^*}[z] - \EX_\rtd[y]$. The following proposition characterizes the bias incurred from predicting a histogram. The proof is in Appendix \ref{sec:appdx_proofs}. 

\begin{restatable}[Bias Characterization]{proposition}{propbias}
The following statements hold:
\begin{enumerate}
    \item If the form of $\pdnox$ is not restricted, the minimizer of the KL-divergence is the result of smoothing $\rtd$ with $\tdy$. Formally, $\argmin_{\pdnox \in \mathcal{P}({\RR})} \EX_{y\sim \rtd} [D_{KL}(\tdy_y||\pdnox)] = s$ where $s(z) = \int \rtd(y) \tdy_y(z) \diff y$.
    \item The bias of $s$ is the expected difference between the scalar target and the mean of the target distribution constructed from it. If the target distribution is Gaussian, $s$ is unbiased. Formally, $\text{Bias}(s, \rtd) = \EX_{y\sim\rtd}[\EX_{z\sim\tdy_y}[z] - y]$.
    \item If the form of $\pdnox$ is restricted to a histogram with bins of equal widths, the minimizer of HL is $\pdnox^*$, where the probability in each bin is equal to the probability of $s$ in the range of that bin. Formally, $\argmin_{\pdnox \in \mathcal{P}({\NN})} \EX_{y\sim \rtd} [\text{HL}(\tdy_y||\pdnox)] =\pdnox^*$ where $\pdnox^*_i = \int_{l_i}^{l_i+w} s(z) \diff z$. 
    \item If the target distribution is Gaussian, the bias of $\pdnox^*$ is at most half of the bin width, and this bound is tight, i.e. $|\text{Bias}(\pdnox^*, \rtd)| \leq w/2$.
\end{enumerate}
\end{restatable}

Recall that this section does not explore the biases incurred due to truncation or the model's lack of capacity to fit different data points. The form of bias in this section is due to restricting the prediction to a histogram, and we call it the \textit{discretization bias}. This analysis characterizes the bias of the mean of the distribution that minimizes HL-Gauss and provides an upper bound on it that can be arbitrarily reduced by using smaller bins, assuming that the target distributions have negligible probability exceeding the support of the histogram. Note that the variance parameter affects both the quantified bias and the extent to which the assumption is satisfied. Section \ref{sec:bias_simulation} provides numerical simulations that show the discretization bias in HL-Gauss reduces as the variance parameter increases and the total bias becomes negligible with proper padding. We will also empirically investigate the bias in Section \ref{sec:experiments} on regression data sets.

\subsection{Bound on Prediction Error}

An often-sought property of surrogate loss functions is reducing an upper bound on the original objective. In this section we provide a result that shows, if the prediction distribution and the target distributions have a similarly bounded support, minimizing the KL-divergence between them reduces an upper bound on the difference between the means. This result motivates minimizing HL and using the mean of the prediction distribution as the final prediction.

\begin{restatable}[Bound on Prediction Error]{proposition}{proppredbound}
Assuming that for a data point, the target distribution $\td_y$ and the model's prediction distribution $\pd$ have supports bounded by the range $[a,b]$, the following bound holds:
\begin{equation} \label{eq:predbound}
(\EX_{\td_y}[z] - \EX_{\pd}[z])^2 \leq 4 \max(|a|,|b|)^2 \min\biggl(\frac{1}{2} D_{KL}(\td_y||\pd), 1 - e^{-D_{KL}(\td_y||\pd)} \biggr).
\end{equation}
\end{restatable}

The proof is in Appendix \ref{sec:appdx_proofs}, and the behavior of the two bounds is explained in Appendix \ref{sec:appdx_bounds}. A model trained with HL reduces $D_{KL}(\td||\pd)$ for each training data point. If the mean of the target distribution is close to the original label, the bound above shows that the mean of the predicted distribution is a good predictor for the original labels in the training set if the model achieves a low training loss.

\section{An Empirical Study on the Histogram Loss}  \label{sec:experiments}

In this section, we investigate the utility of HL-Gauss for regression compared to using the $\ell_2$ loss. We particularly investigate why the modification to this distributional loss improves performance by first comparing to baselines to test if it is due to each of the various ways in which the two losses differ and then designing experiments to explore other properties of HL.

\subsection{A Synthetic Experiment}\label{synth}

We start with an experiment on sine functions to understand the difference between HL-Gauss and the $\ell_2$ loss in a controlled setting. The goal is to see how the convergence rates of the two losses compare as the frequency or the offset of the target changes. The function we want to fit has an input $x \in \mathbb{R}$ and a deterministic target $\sin(\alpha x) + \beta$, where $\alpha, \beta \in \mathbb{R}$ determine the frequency and the offset, respectively. We pick 500 points on a grid from the range $[-\pi, \pi]$ and generate the corresponding targets. An example is depicted as the blue dots in Figure \ref{fig:synthetic} (Left). We use a neural network with two hidden layers of width 1024 and leaky ReLU activation, train it with the Adam optimizer, and report the MSE on the same 500 points during training. When comparing HL-Gauss and $\ell_2$, we try a wide range of step-size values in $\{10^{-5}, 10^{-4}, \cdots, 10^{-1}\}$ for each method and show the curve with the lowest area under the curve during 1000 iterations.

In the first part of the experiment, we fix $\beta$ to 0 and vary the frequency $\alpha \in \{1, 10, 20\}$. The range of the target is $[-1, 1]$, and for HL-Gauss, we discretize the range $[-1.5, 1.5]$ evenly to 100 bins and set $\sigma$ to twice the resulting bin width. As shown in Figure \ref{fig:synthetic} (Middle), for both losses training slows down with higher target frequencies and there is a clear difference between HL-Gauss and $\ell_2$ in high frequencies. We have used a sufficiently wide range of step-size values to ensure this difference cannot be eliminated with a larger step-size. The predictions obtained with HL-Gauss and $\ell_2$ for $\alpha=10$ are visualized in Figure \ref{fig:synthetic} (Left).

We now fix $\alpha$ to 10 and vary the offset $\beta \in \{0, 1, 10\}$. The range of the target is $[-1 + \beta, 1 + \beta]$ and the support of the histogram in HL-Gauss is $[-1.5, 11.5]$ regardless of the offset. The number of bins and $\sigma$ are similar to the previous part. The learning curves are shown in Figure \ref{fig:synthetic} (Right). The convergence rate of $\ell_2$ deteriorates as the offset rises while the behavior of HL-Gauss is not affected.

Overall, the results in this synthetic experiment suggest that, compared to the $\ell_2$ loss, HL-Gauss is especially helpful for convergence when the target has high frequency components or is far from zero.

\begin{figure*}[!ht]
\centering
\begin{subfigure}[b]{\figwidththree}
         \includegraphics[width=\textwidth]{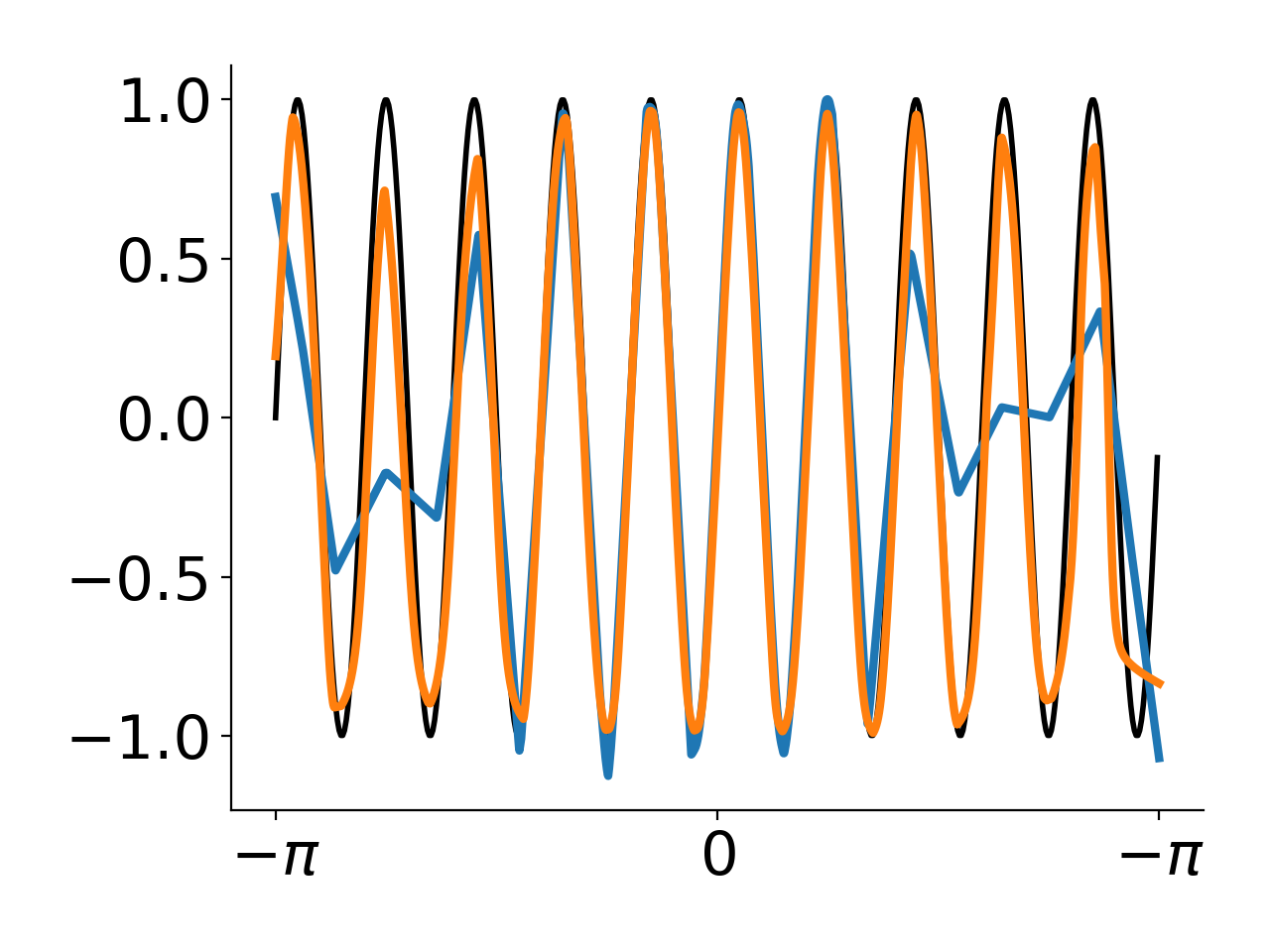}
\end{subfigure} 
 \begin{subfigure}[b]{\figwidththree}
         \includegraphics[width=\textwidth]{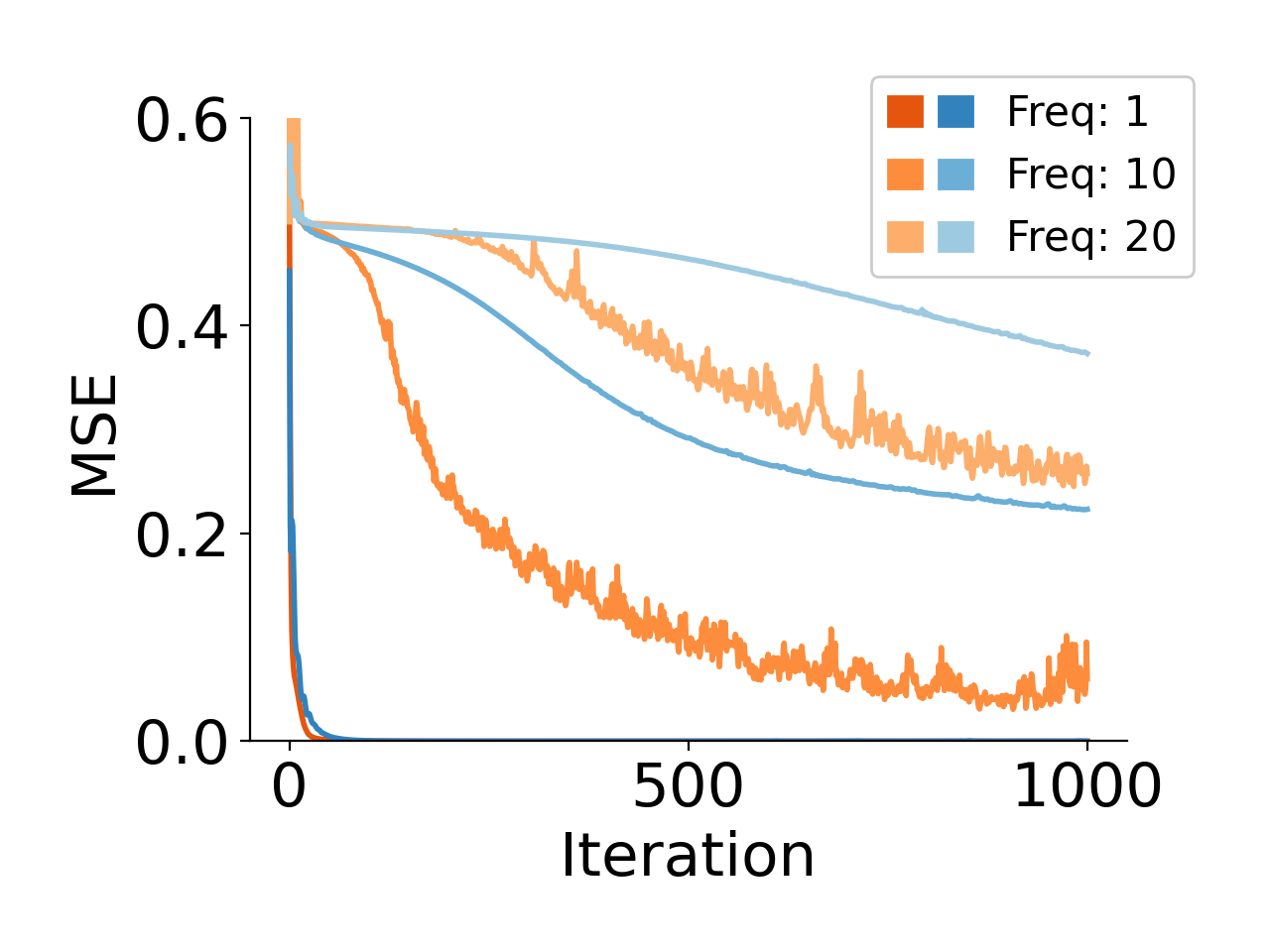}
\end{subfigure} 
  \begin{subfigure}[b]{\figwidththree}
         \includegraphics[width=\textwidth]{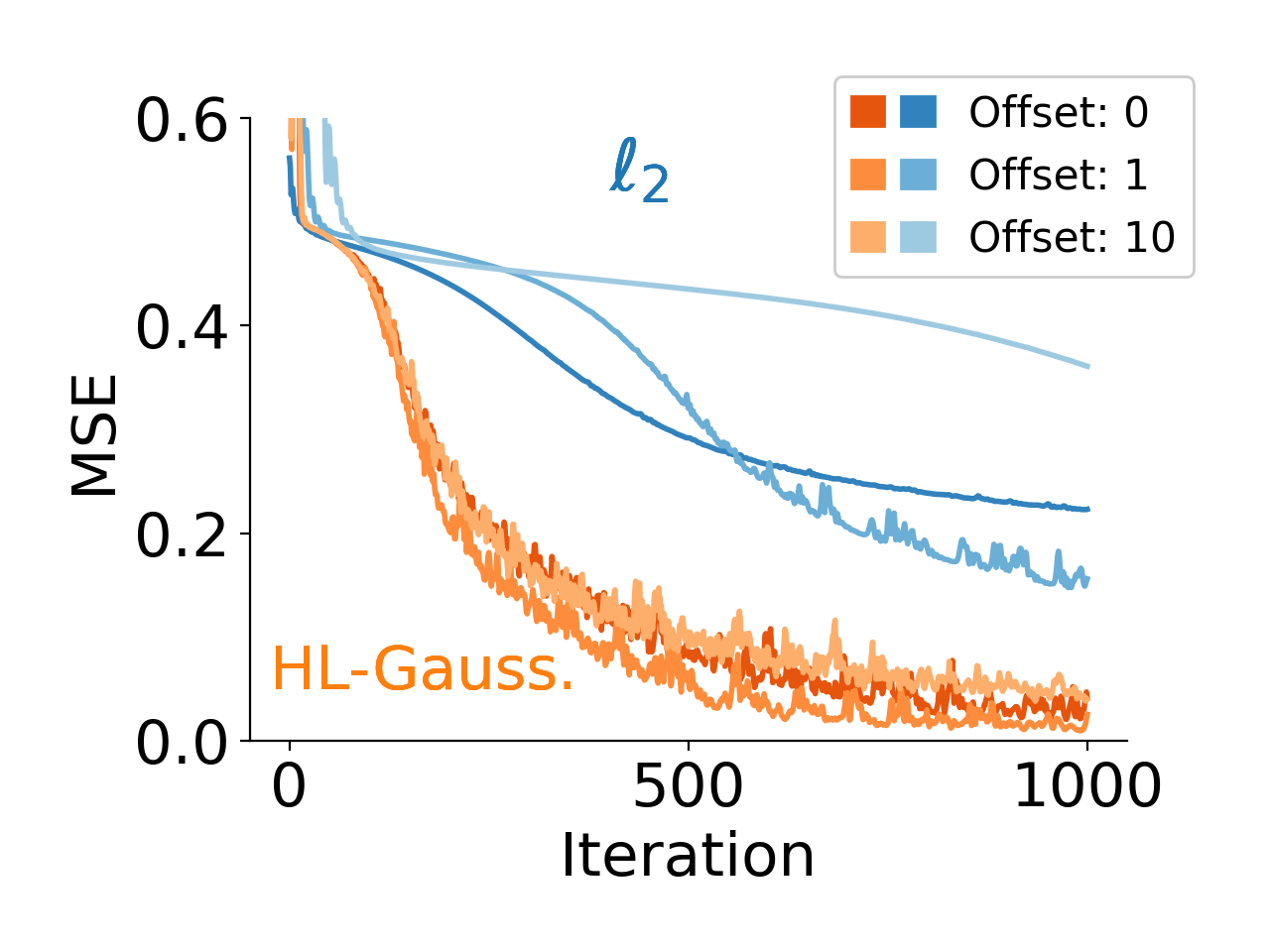}
\end{subfigure} 
\caption{(Left) A synthetic task with frequency 10 and offset 0. The black curve shows the training data, and the orange and blue curves show the predictions obtained after training with HL-Gauss and $\ell_2$, respectively. (Middle) Learning curves for different frequencies in $\{1, 10, 20\}$ and fixed offset $0$. Orange and blue denote HL-Gauss and $\ell_2$, and brighter shades denote higher frequencies. Each curve is averaged over 5 runs with different random initializations. (Right) Learning curves for different offsets in $\{0, 1, 10\}$ and fixed frequency $10$. Brighter shades show higher offsets. Altogether, HL-Gauss trains remarkably faster than $\ell_2$ when the target has high frequency or is far from zero.}
\label{fig:synthetic}
\end{figure*}

\subsection{Overview of Empirical Study}

Now, we turn to the empirical study on real-world data. The purpose of these experiments is to first show proof of concept studies and then provide insights on the role of factors that differentiate HL from the $\ell_2$ loss. The main goal is to understand when and why one would opt for HL. The data sets and the algorithms used in our experiments are described in Sections \ref{datasets} and \ref{algorithms} respectively.

Section \ref{overall} compares HL-Gauss, $\ell_2$, and several baselines, each of which is designed to differ in a particular aspect from HL-Gauss or $\ell_2$. The comparison shows the importance of each of these factors. These experiments attempt to clarify (a) if HL-Gauss is a reasonable choice for a regression task, (b) whether the difference between the performance of HL-Gauss and $\ell_2$ is merely because HL-Gauss is learning a flexible distribution, (c) if sensitivity of $\ell_2$ to outliers can explain the difference the performance of the two losses, and (d) if this difference is the result of data augmentation in the output space when using HL-Gauss.

In Section \ref{bias}, we investigate the impact of discretization bias on both HL-Gauss and HL-OneBin. Previously, we offered a theoretical discussion on the effect of the target distribution and the variance parameter on this bias in Section \ref{bias_theory}. The results in Section \ref{bias} show that this is indeed an important factor at play and can explain most of the difference between the behaviors of HL-OneBin and HL-Gauss. We further evaluate the models with different numbers of bins and variance parameter values to see if a bias-variance tradeoff is at work.

In reinforcement learning, learning the full distribution has been likened to auxiliary tasks that improve the performance by providing a better representation \citep{bellemare2017distributional}. Experiments in Section \ref{rep} test this hypothesis in a regression setting in three ways: (a) comparing the representations of a model that learns the distribution and one that learns the expected value, (b) evaluating a model with the main task of learning the expected value and the auxiliary task of learning the distribution, and (c) adding the auxiliary tasks of learning higher moments to a model that learns the expected value.

In Section \ref{opt} we compare the behavior of $\ell_2$, HL-OneBin, and HL-Gauss during the training to find if the difference in final performance is the result of different loss surfaces. This hypothesis is related to the gradient norms of different losses (described in Section \ref{opt_theory}) and has been posited by \citet{imani2018improving}.

Finally, Section \ref{outlier} measures the sensitivity of HL to target corruption and the sensitivity of the model to input perturbations. The first experiment tests to what extent the performance of a loss deteriorates in the presence of corrupted targets in the training data. The second experiment studies the sensitivity of a model's output to input perturbation, which, in classification problems, has been closely associated with poor generalization.

\subsection{Data Sets and Pre-Processing} \label{datasets}

Experiments are conducted on four regression data sets. \\
The \textbf{CT Position} data set is from CT images of patients \citep{graf20112d}, with 385 features 
and the target set to the relative location of the image. \\ 
The \textbf{Song Year} data set is a subset of The Million Song Dataset \citep{bertin2011million}, with 90 audio features for a song and a target corresponding to the release year. \\
The \textbf{Bike Sharing} data set \citep{fanaee2014event}, about hourly bike rentals for two years, has 16 features and the target set to the number of rented bikes. \\
The \textbf{Pole} data set \citep{Olson2017PMLB} describes a telecommunication problem and has 49 features. 

All data sets are complete, and there is no need for data imputation. All features are transformed to have zero mean and unit variance. We randomly split the data into train and test sets in each run. Root mean squared error (RMSE) and mean absolute error (MAE) are reported over 5 runs, with standard errors. Other ways of processing these data sets may result in a higher performance. The goal of this study, however, is comparing different methods and testing hypotheses rather than achieving state-of-the-art results on these tasks. More details about these data sets are provided in Table \ref{tab:Dataset Table} and Figure \ref{fig:histograms} in the appendix. We will show the results on CT Position data set in the main paper and defer the results on the other data sets to Appendix \ref{sec:appdx_exp_details}.

\subsection{Algorithms} \label{algorithms}
We compared several regression strategies, distribution learning approaches, and several HL variants. All the approaches---except for linear regression---use the same neural network, with differences only in the output layer. 
Unless specified otherwise, all networks using HL have 100 bins. The support of the histogram is chosen by the data set target range, and 10 bins are kept for padding on each side to minimize the effect of truncation. Hyperparameters for comparison algorithms are chosen according to the best Test MAE. Network architectures were chosen according to best Test MAE for $\ell_2$. More details about the architectures and the range of hyperparameters are provided in Appendix \ref{sec:appdx_exp_details}.\\
\textbf{Linear Regression} is included as a baseline, using ordinary least squares with the inputs.\\
\textbf{Squared-error} ($\ell_2$) is the neural network trained using the $\ell_2$ loss. The targets are normalized to range $[0,1]$, which was needed to improve stability and accuracy.\\
\textbf{$\mathbf{\ell_2}$+Noise} is the same as $\ell_2$, except Gaussian noise is added to the targets as a form of augmentation.\\
\textbf{$\mathbf{\ell_2}$+Clipping} is the same as $\ell_2$, but with gradient-norm clipping during training.\\
\textbf{Absolute-error} ($\ell_1$) is the neural network trained using the $\ell_1$ loss.\\
\textbf{MDN} is a Mixture Density Network \citep{bishop1994mixture} that models the target distribution as a mixture of Gaussian distributions.\\
\textbf{$\mathbf{\ell_2}$+Softmax} uses a softmax-layer with $\ell_2$ loss, $\sum_{i=1}^\nbins (f_i(\xvec_j) \cen_i - y_j)^2$ for bin centers $\cen_i$, with otherwise the same settings as HL-Gauss. \\
\textbf{$\mathbf{\ell_2}$+Softmax+Mean} uses a softmax-layer with $\ell_2$ loss on the mean of the distribution, $(\sum_{i=1}^\nbins f_i(\xvec_j) \cen_i - y_j)^2$ for bin centers $\cen_i$, with otherwise the same settings as HL-Gauss. \\
\textbf{HL-OneBin} is HL with a Dirac delta target distribution.\\
\textbf{HL-Uniform} is HL with a target distribution that mixes between a delta distribution and the uniform distribution, with a weighting of $\epsilon$ on the uniform and $1-\epsilon$ on the delta.\\
\textbf{HL-Gauss} is HL with a truncated Gaussian distribution as the target distribution. The parameter $\sigma$ is set to the width of the bins. Importantly, we do not tune any of HL-Gauss's hyperparameters to individual data sets.

\subsection{Overall Results} \label{overall}

We first report the relative performance of different baselines compared to HL-Gauss in Table \ref{tab:overall_ctscan}. The error rates of HL-Uniform are plotted separately in Figure \ref{fig:uniform_ctscan} to show that, regardless of the choice of the parameter in the loss, a uniform target distribution cannot help performance. We focus on Test MAE as the evaluation criterion due to considerably lower variability in the results and its wide use in practice and present Train MAE, Train RMSE, and Test RMSE to provide a complete picture.

\begin{table*}[!t]
\vspace{-0.1cm}
\centering
\begin{tabular}{lllbl}                                                                                                  \hline                                                                                                                   Method     & Train MAE                         & Train RMSE                        & Test MAE                             & Test RMSE                            \\                                                                      \hline                                                                                                                   Lin Reg    & $ 6.073${\scriptsize $(\pm 0.007)$} & $ 8.209${\scriptsize $(\pm 0.006)$} & $ 6.170${\scriptsize $(\pm 0.025)$} & $ 8.341${\scriptsize $(\pm 0.025)$} \\                                                                       $\ell_2$ & $ 0.140${\scriptsize $(\pm 0.036)$} & $ 0.183${\scriptsize $(\pm 0.048)$} & $ 0.176${\scriptsize $(\pm 0.034)$} & $ 0.267${\scriptsize $(\pm 0.039)$} \\                                                                       $\ell_2$+Noise  & $ 0.114${\scriptsize $(\pm 0.011)$} & $ 0.152${\scriptsize $(\pm 0.015)$} & $ 0.152${\scriptsize $(\pm 0.010)$} & $ 0.311${\scriptsize $(\pm 0.076)$} \\                                                                       $\ell_2$+Clip   & $ 0.111${\scriptsize $(\pm 0.004)$} & $ 0.143${\scriptsize $(\pm 0.005)$} & $ 0.148${\scriptsize $(\pm 0.004)$} & $ 0.231${\scriptsize $(\pm 0.006)$} \\                                                                       $\ell_1$         & $ 0.121${\scriptsize $(\pm 0.007)$} & $ 0.164${\scriptsize $(\pm 0.007)$} & $ 0.162${\scriptsize $(\pm 0.006)$} & $ 0.389${\scriptsize $(\pm 0.038)$} \\                                                                       MDN        & $ 0.114${\scriptsize$(\pm0.004)$} & $ 0.152${\scriptsize$(\pm0.005)$} & $ 0.153${\scriptsize$(\pm0.004)$} & $ 0.308${\scriptsize$(\pm0.035)$} \\ $\ell_2$+Softmax & $ 0.065${\scriptsize $(\pm 0.006)$} & $ 0.094${\scriptsize $(\pm 0.008)$} & $ 0.105${\scriptsize $(\pm 0.006)$} & $ 0.301${\scriptsize $(\pm 0.067)$} \\
$\ell_2$+Softmax+Mean & $0.115${\scriptsize $(\pm 0.029)$} & $0.261${\scriptsize $(\pm 0.141)$} & $0.122${\scriptsize $(\pm 0.028)$} & $0.302${\scriptsize $(\pm 0.129)$}\\
HL-OneBin  & $ 0.309${\scriptsize $(\pm 0.000)$} & $ 0.364${\scriptsize $(\pm 0.006)$} & $ 0.335${\scriptsize $(\pm 0.004)$} & $ 0.660${\scriptsize $(\pm 0.099)$} \\                                                                       HL-Gauss.  & $ 0.061${\scriptsize $(\pm 0.006)$} & $ 0.164${\scriptsize $(\pm 0.090)$} & $ 0.098${\scriptsize $(\pm 0.005)$} & $ 0.274${\scriptsize $(\pm 0.090)$} \\                                                                      \hline                                                                                                                  \end{tabular}
\caption{Overall results on CT Scan. HL-Gauss achieved the lowest Test MAE. Among the other methods, $\ell_2$+softmax yielded an error rate close to that of HL-Gauss, and the rest of the baselines performed worse. There was a large gap between the performance of HL-OneBin and HL-Gauss.}\label{tab:overall_ctscan}
\end{table*}

The overall conclusion is that HL-Gauss does not harm performance and can often considerably improve performance over alternatives.

Another observation is that simply modeling the output distribution cannot achieve HL-Gauss's performance. This is made clear by comparing HL-Gauss and MDN. MDN learns the output distribution, but as a mixture of Gaussian components rather than a histogram. In our experiments, MDN consistently underperformed HL-Gauss. Note that the error rate reported for MDN is the one observed after fixing numerical instabilities and tuning the number of components.

A related idea to learning the distribution explicitly is to use data augmentation in the target space as an implicit approach to minimizing divergences to distributions. 
We therefore also compared to directly modifying the labels and gradients, with $\ell_2+$Noise and $\ell_2+$Clipping. These approaches do perform slightly better than regression for some settings but do not achieve the same gains as HL-Gauss. The conclusion is that HL-Gauss's performance cannot be attributed to data augmentation in the targets.

A well-known weakness of the $\ell_2$ loss is its sensitivity to outliers. To find if the difference between the performance of the $\ell_2$ loss and HL-Gauss can be explained by the presence of outliers, we can compare HL-Gauss with the $\ell_1$ loss, which is robust to outliers. The model trained with $\ell_1$ still underperformed HL-Gauss on three data sets, which suggests that the gap between the performance of HL-Gauss and the $\ell_2$ loss is not merely due to the presence of outliers.

Choices of target distribution other than a Gaussian distribution are not effective. HL-OneBin and HL-Uniform appear to have fewer advantages, as shown in the results, and both can actually do worse than $\ell_2$. A uniform target distribution worsened the error on all data sets. An important artifact of label smoothing in HL-Uniform is that it biases the mean of the distribution. Since the mean of a uniform distribution is the center of the range, mixing the target distribution with a uniform distribution pulls the mean towards the center of the range. Recall from Section \ref{sec:theoretical_analysis} that HL will suffer from large bias if the mean of the target distribution is far from the target.

\begin{wrapfigure}[18]{r}{0.39\textwidth}
\centering
\includegraphics[width=0.39\textwidth]{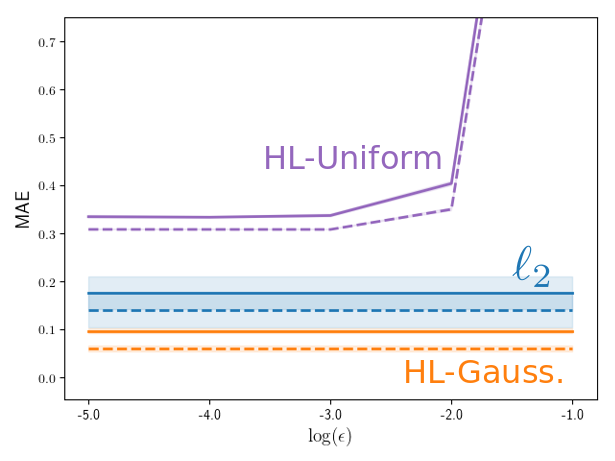}
    \caption{HL-Uniform results on CT Scan. Dotted and solid lines show train and test errors respectively. The parameter $\epsilon$ is the weighting on the uniform distribution and raising it only impaired performance.}
    \label{fig:uniform_ctscan}
\end{wrapfigure}
 
Finally, $\ell_2$+softmax performed close to HL-Gauss on the CT Position data set and slightly worse on the other tasks. While this model does not estimate the target distribution by minimizing the KL-divergence, it still benefits from the softmax nonlinearity in the output layer like HL-Gauss. The comparison between $\ell_2$+softmax and HL-Gauss shows that this softmax nonlinearity appears to be beneficial, but it cannot totally explain HL-Gauss's performance. In our experiments on the three other data sets in Appendix \ref{sec:appdx_results}, $\ell_2$+softmax consistently underperformed HL-Gauss. Of additional note is that $\ell_2$+softmax+mean is slightly worse than $\ell_2$+softmax. The $\ell_2$+softmax+mean can be viewed as just adding the softmax nonlinearity to the NN and still having a scalar final output to predict the target, whereas $\ell_2$+softmax has a squared loss per bin and is more similar to HL-Gauss.  

The comparisons in this section address several questions regarding HL-Gauss's performance. HL-Gauss can often work better than $\ell_2$, and this rise in performance cannot be solely attributed to learning a flexible distribution, using data augmentation in the labels, being robust to outliers in the data set, or employing a softmax nonlinearity in the model.

\subsection{Bias and the Choice of Target Distribution} \label{bias}

\begin{wrapfigure}[17]{r}{0.39\textwidth}
    \centering
    \includegraphics[width=0.39\textwidth]{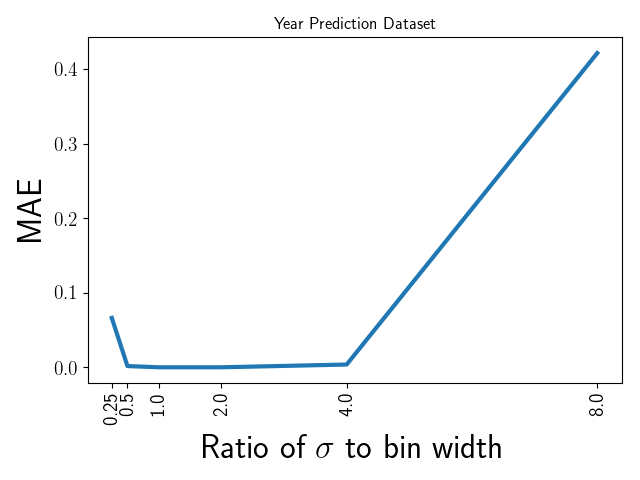}
    \caption{MAE between the means of HL targets and the original labels. It can be seen that extreme values of $\sigma$ on either side biased the mean of the target distributions.}
    \label{fig:label_bias}
\end{wrapfigure}

Characterizing the discretization bias in section \ref{sec:theoretical_analysis} showed a link between the bias of the loss and the choice of $\sigma$ in HL-Gauss. This parameter can also affect the bias from truncation as it controls the probability in the tails of the target distribution that exceed the support of the histogram. This section empirically investigates the effect of the target distribution on the bias of the loss. We first show how different values of $\sigma$ result in different levels of bias. Then, we compare the target distributions of HL-Gauss and HL-OneBin with another target distribution which reduces the bias to zero. Finally, we evaluate HL-OneBin and HL-Gauss using different parameter values to find if these parameters control a bias-variance trade-off.

Weightings in HL-Gauss (which represent the area under the target distribution's pdf in the range of each bin) are computed before training, and the histogram is trained to mimic these weightings. To obtain low bias, we want the mean of the trained histogram to be close to the original target. To analyze how close the loss is to this situation in practice, we plot the mean absolute error between the original targets in our data sets and the mean of the histograms that perfectly match weightings obtained from those targets. Figure \ref{fig:label_bias} shows this error for HL-Gauss with 100 bins and different choices of $\sigma$.

This simple analysis suggests that choosing a Gaussian target distribution with a carefully tuned variance parameter (rather than Dirac delta) can come with the extra benefit of reducing bias. To see to what extent the gap between HL-OneBin and HL-Gauss's performance is because of this bias and to find if histogram losses in general are suffering from a large bias, we introduce a target distribution that reduces the discretization bias to zero (regardless of the number of bins) and compare it to HL-OneBin and HL-Gauss in the ordinary setting of 100 bins. As we showed in Section \ref{sec:theoretical_analysis}, if the target distribution $\tdy_y$ is chosen so that, for each sample, the expected value of the closest histogram density to it is exactly $y$, discretization bias is zero. Assume $\cen_i$ is the last bin center before $y$ (so $y$ is somewhere between $\cen_i$ and $\cen_{i+1}$). The idea is to use a histogram density with probability $1 - \frac{y-\cen_i}{w}$ in bin $i$ and probability $\frac{y-\cen_i}{w}$ in bin $i+1$ as the target distribution.\footnote{This is close to the projection operator by \citet{bellemare2017distributional}, although that projection was introduced to fix the problem of disjoint support for discrete distributions. \citet{rowland2018analysis} also studied this projection operator and found its ability of preserving the expected value to be beneficial in reinforcement learning.} The expected value of this target distribution is always $y$, so this target distribution eliminates discretization bias. We call the histogram loss with this choice of target distribution \textit{HL-Projected} and compare it with HL-OneBin and HL-Gauss on the four data sets.

\begin{table}[!ht]
    \centering
    \begin{tabular}{lllll}                                                                                                  \hline                                                                                                                   Method       & CT Scan                         & Song Year                        & Bike Sharing                             & Pole                            \\                                                                    \hline                                                                                                                   HL-OneBin    & $ 0.335${\scriptsize $(\pm 0.004)$} &  $ 5.906${\scriptsize $(\pm 0.020)$} & $ 26.689${\scriptsize $(\pm 0.280)$} &  $ 1.264${\scriptsize $(\pm 0.019)$}  \\                                                                     HL-Projected & $ 0.103${\scriptsize $(\pm 0.003)$} & $ 5.917${\scriptsize $(\pm 0.019)$} & $ 26.180${\scriptsize $(\pm 0.348)$} & $ 0.741${\scriptsize $(\pm 0.018)$} \\                                                                    HL-Gauss.    & $ 0.098${\scriptsize $(\pm 0.005)$} & $ 5.903${\scriptsize $(\pm 0.010)$} & $ 25.525${\scriptsize $(\pm 0.331)$} & $ 0.714${\scriptsize $(\pm 0.024)$} \\                                                                     \hline                                                                                                                  \end{tabular}
    \caption{Discretization bias experiment. The reported numbers are Test MAE. HL-Projected achieved a Test MAE close to that of HL-Gauss and performed noticeably better than HL-OneBin. The gap between HL-Gauss and HL-Projected is small but appears on all the four data sets, and the gap is a bit bigger for Test RMSE (see Table \ref{tab:proj_ctscan_appdx}).}
    \label{tab:proj_ctscan}

\end{table}

While the difference between the performance of HL-Projected and HL-Gauss is small, it is consistent on the four data sets. The comparison shows that using this target distribution instead of HL-OneBin often results in a noticeable reduction in the errors, but the performance of HL-Gauss remains unbeaten. Although HL-Projected removes discretization bias, it still does not have HL-Gauss's desirable property of punishing faraway predictions more severely. 
Overall, the bias of HL-Gauss with proper $\sigma$ and padding is negligible, as discussed further in Appendix \ref{sec:bias_simulation}.

The parameters in the loss, namely the number of bins and $\sigma$, can affect the bias. A hypothesis is that a good choice of parameters for HL can reduce overfitting and place the method in a sweet spot in a bias-variance trade-off. Two experiments were designed to find if there is a bias-variance trade-off at work. In the first experiment, we changed the number of bins while keeping the padding and target distribution $\sigma$ fixed. The second experiment studies the effect of changing $\sigma$ on the performance of HL-Gauss. Figure \ref{fig:bias_var_sigma_ctscan} (Left) evaluates HL-Gauss and HL-OneBin with different numbers of bins, and Figure \ref{fig:bias_var_sigma_ctscan} (Right) shows the effect of changing $\sigma$ on HL-Gauss.

\begin{figure*}[!ht]
\centering
 \begin{subfigure}[b]{\figwidthtwo}
         \includegraphics[width=\textwidth]{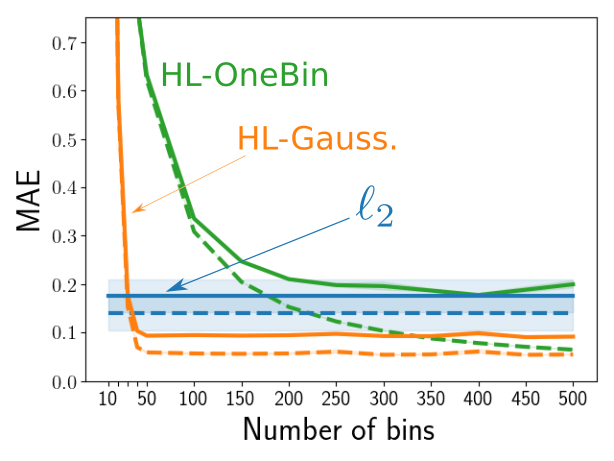}
\label{fig:bias_var_ctscan}
 \end{subfigure}
 \begin{subfigure}[b]{\figwidthtwo}
         \includegraphics[width=\textwidth]{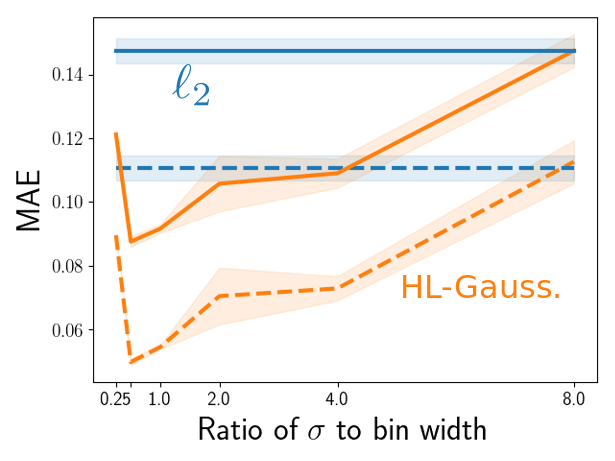}
\label{fig:sigma_ctscan}
 \end{subfigure} 
\caption{(Left) Changing the number of bins on CT Scan. Dotted and solid lines show train and test errors, respectively. The performance of $\ell_2$ is shown as a horizontal line for comparison. A small number of bins resulted in high train and test errors, indicating high bias. A higher number of bins generally did not result in a rise in test error, with the exception of HL-OneBin's test RMSE in Appendix \ref{sec:appdx_results}. (Right) Changing the parameter $\sigma$ on CT Scan. Extreme values of $\sigma$ resulted in bad performance. The error rates for train and test followed each other when changing this parameter.}
\label{fig:bias_var_sigma_ctscan}
\end{figure*}

Evaluating different numbers of bins shows that, although the quality of solution deteriorates when the bins are too few, there is no sign of overfitting in HL-Gauss with a higher number of bins. In the experiment on the $\sigma$ parameter, we observe a v-shaped pattern. Increasing $\sigma$ at first reduces both the train and test errors, and after a point, both errors begin to rise. Target distributions with higher $\sigma$ have longer tails, so these high errors may be the result of truncation. There seems to be no evidence that the choice of $\sigma$ caused overfitting, as both train and test error decreased and then increased together as we varied $\sigma$.

The analysis in this section shows the effect of the target distribution and the parameters of HL on the bias of the loss. Most of the gap between the performance of HL-Gauss and HL-OneBin can be explained by the high bias of HL-OneBin. Both the number of bins and $\sigma$ can affect the bias of HL, but reducing this bias generally does not result in overfitting.

\subsection{Quality of Learned Representations} \label{rep}

Learning a distribution, as opposed to a single statistic, provides a more difficult target---one that could require a better representation.
The hypothesis is that among the functions $f$ in the function class $\mathcal{F}$, there is a set of functions
that can predict the targets almost equally well. To distinguish among these functions, a wider range of tasks can make it more likely to select the true function, or at least one that generalizes better. We conducted three experiments to test the hypothesis that an improved representation is learned.

First, we trained with HL-Gauss and $\ell_2$ to obtain their representations. We tested (a) swapping the representations and re-learning only the last layer, (b) initializing with the other's representation, and (c) using the same fixed random representation for both. For (a) and (c), the optimizations for both are convex, since the representation is fixed. If the challenge of predicting a distribution in HL results in a better representation, one would expect the gap between $\ell_2$ and HL-Gauss to go away when each one is trained on or initialized with the other's representation.

\begin{table*}[!ht]
\centering
\begin{tabular}{llllll}\hline                                                                                                                              & Loss               & Default                              & Fixed                                & Initialized                          & Random                                \\                                                \hline                                                                                                                   Train MAE  & $\ell_2$         & $ 0.140${\scriptsize $(\pm 0.036)$} & $ 2.490${\scriptsize $(\pm 0.094)$} & $ 0.148${\scriptsize $(\pm 0.013)$} & $ 9.233${\scriptsize $(\pm 0.157)$}  \\                                                 Train MAE  & HL-Gauss. & $ 0.061${\scriptsize $(\pm 0.006)$} & $ 0.153${\scriptsize $(\pm 0.020)$} & $ 0.063${\scriptsize $(\pm 0.002)$} & $ 2.604${\scriptsize $(\pm 0.103)$}  \\                                                 Train RMSE & $\ell_2$         & $ 0.183${\scriptsize $(\pm 0.048)$} & $ 3.465${\scriptsize $(\pm 0.158)$} & $ 0.205${\scriptsize $(\pm 0.022)$} & $ 12.247${\scriptsize $(\pm 0.172)$} \\                                                 Train RMSE & HL-Gauss. & $ 0.164${\scriptsize $(\pm 0.090)$} & $ 0.219${\scriptsize $(\pm 0.027)$} & $ 0.085${\scriptsize $(\pm 0.004)$} & $ 5.756${\scriptsize $(\pm 0.232)$}  \\                                                  \rowcolor{aliceblue} Test MAE   & $\ell_2$         & $ 0.176${\scriptsize $(\pm 0.034)$} & $ 2.537${\scriptsize $(\pm 0.089)$} & $ 0.185${\scriptsize $(\pm 0.013)$} & $ 9.305${\scriptsize $(\pm 0.191)$}  \\                           \rowcolor{aliceblue}                      Test MAE   & HL-Gauss. & $ 0.098${\scriptsize $(\pm 0.005)$} & $ 0.187${\scriptsize $(\pm 0.019)$} & $ 0.101${\scriptsize $(\pm 0.002)$} & $ 2.727${\scriptsize $(\pm 0.113)$}  \\                                                 Test RMSE  & $\ell_2$         & $ 0.267${\scriptsize $(\pm 0.039)$} & $ 3.551${\scriptsize $(\pm 0.143)$} & $ 0.299${\scriptsize $(\pm 0.025)$} & $ 12.315${\scriptsize $(\pm 0.216)$} \\                                                 Test RMSE  & HL-Gauss. & $ 0.274${\scriptsize $(\pm 0.090)$} & $ 0.287${\scriptsize $(\pm 0.024)$} & $ 0.183${\scriptsize $(\pm 0.005)$} & $ 6.065${\scriptsize $(\pm 0.234)$}  \\                                                \hline                                                                                                                  \end{tabular} 
\caption{Representation results on CT Scan. We tested (a) swapping the representations and re-learning on the last layer (\textbf{Fixed}), (b) initializing with the other's representation (\textbf{Initialized}), (c) and using the same fixed random representation for both (\textbf{Random}) and only learning the last layer. Using HL-Gauss representation for $\ell_2$ (first column, Fixed) caused a sudden spike in error, even though the last layer in $\ell_2$ was re-trained. This suggests the representation is tuned to HL-Gauss. The representation did not even seem to give a boost in performance, as an initialization (second column, Initialization). Finally, even with the same random representation, where HL-Gauss cannot be said to improve the representation, HL-Gauss still obtained substantially better performance.}\label{tab:representation_ctscan}
\end{table*}

The results in Table \ref{tab:representation_ctscan} are surprisingly conclusive. They suggest that the performance gain is not due to better representations, and even under a random representation, HL-Gauss performs noticeably better than $\ell_2$.

In the second experiment, we tested a network called Multi-Task that predicted both the expected value (with the $\ell_2$ loss) and the distribution (using HL-Gauss). The predicted distribution was not used for evaluation and was only present during the training as an auxiliary task to improve the representation. If predicting the extra information in a distribution is the reason for the superior performance of HL-Gauss, a regression network with the auxiliary task of predicting the distribution is expected to achieve a similar performance. Figure \ref{fig:multitask_moments_ctscan} (Left) shows the results on the four data sets.

It can be seen from these results that raising the coefficient for the auxiliary task does not improve the performance, and the new model generally has higher errors than HL-Gauss. The conclusion is that the extra information required for predicting the distribution cannot explain the gap between the $\ell_2$ loss and HL-Gauss.

It might be argued that learning a histogram is not a suitable auxiliary task for regression, and the negative results above can be due to the incompatibility between predicting the mean with the $\ell_2$ loss and predicting the distribution with HL-Gauss. We trained the $\ell_2$ model with another auxiliary task in the third experiment. For the new task, predicting higher moments of the distribution, the target is raised to higher powers and the extra outputs try to estimate it by reducing the $\ell_2$ loss. These auxiliary tasks of predicting higher moments are only used to improve the representation during training and are discarded at inference time. The results are shown in Figure \ref{fig:multitask_moments_ctscan} (Right).

\begin{figure*}[!ht]
\centering
 \begin{subfigure}[b]{\figwidthtwo}
         \includegraphics[width=\textwidth]{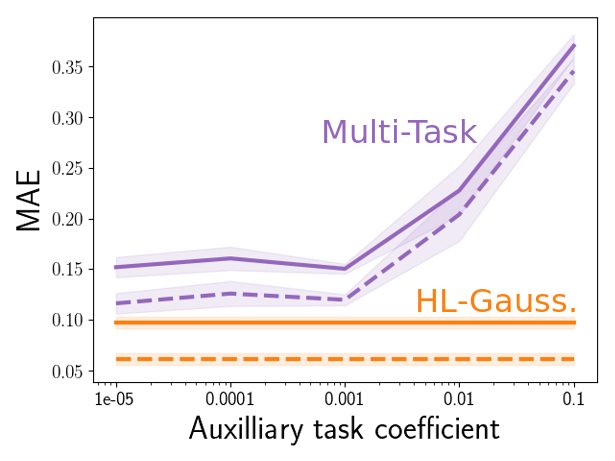}
\end{subfigure}
 \begin{subfigure}[b]{\figwidthtwo}
         \includegraphics[width=\textwidth]{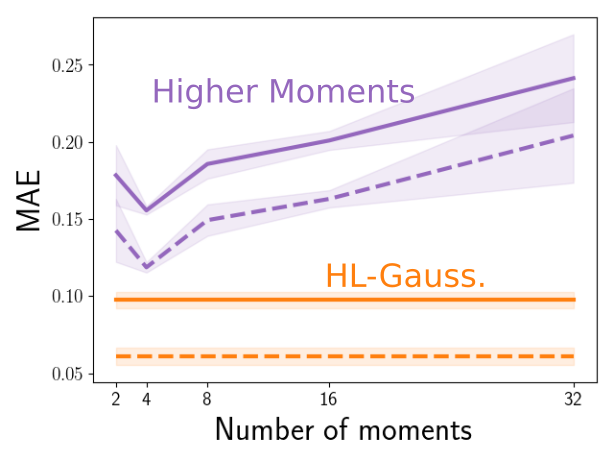}
\end{subfigure} 
\caption{(Left) Multi-Task Network results on CT Scan. The loss function is the mean prediction's squared error plus the distribution prediction's KL-divergence multiplied by a coefficient. The horizontal axis shows the coefficient for KL-divergence. Dotted and solid lines show train and test errors, respectively. There was not a substantial drop in the error rate when increasing the coefficient in the loss, and the performance only became worse. (Right) Higher Moments Network results on CT Scan. The horizontal axis shows the number of moments (including the mean). There was no substantial decrease in error when predicting higher moments. With four moments, there was a slight reduction in error, but the new model still performed worse than HL-Gauss in terms of Test MAE.}
\label{fig:multitask_moments_ctscan}
\end{figure*}

There is little gain in predicting higher moments, and the new model still underperforms HL-Gauss. Like the previous experiment, predicting this extra information about the output distribution appears to have little benefit as an auxiliary task.

The experiments in this section consistently reject the hypothesis that the challenge of predicting the distribution is the main factor behind HL-Gauss's performance. The performance cannot be achieved by exploiting HL-Gauss's representation or by predicting the histogram or higher moments of the distribution as auxiliary tasks.

\subsection{Optimization Properties} \label{opt}
In Section \ref{sec:theoretical_analysis}, we provided a bound on the gradient norm of HL that suggested ease of optimization with gradient descent. In this section, we show two experiments that study the optimization of the $\ell_2$ loss and HL empirically.

First, we recorded the training errors and gradient norms during training to see if using HL-Gauss is beneficial for optimization. Each gradient norm value was normalized by the difference between the training loss observed at that point in the training of the model and the minimum possible value for the loss.\footnote{The minimum value of the $\ell_2$ loss is zero. HL, however, is the cross-entropy to the target distribution, so its minimum value is the entropy of the histogram that best estimates the target distribution.} The goal was to find which model trains faster and has a smoother loss surface. Without the normalization of gradient norms, a simple scaling of the loss would shrink the fluctuations of gradient norms in the plot without making the optimization easier.

\begin{figure*}[!ht]
\centering
 \begin{subfigure}[b]{\figwidthtwo}
         \includegraphics[width=\textwidth]{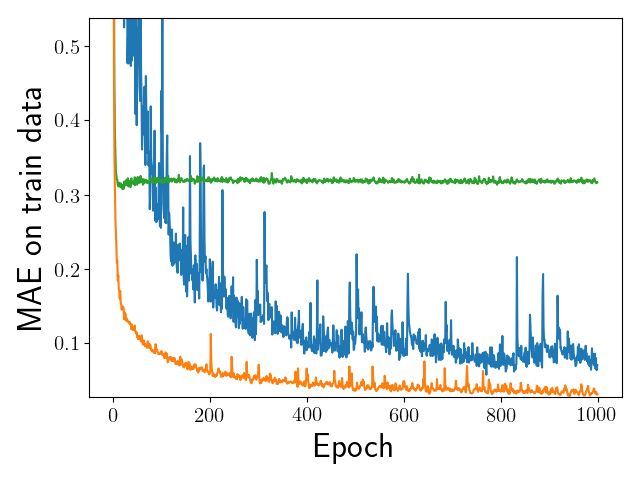}
\label{fig:mae-conv}
 \end{subfigure} 
 \begin{subfigure}[b]{\figwidthtwo}
         \includegraphics[width=\textwidth]{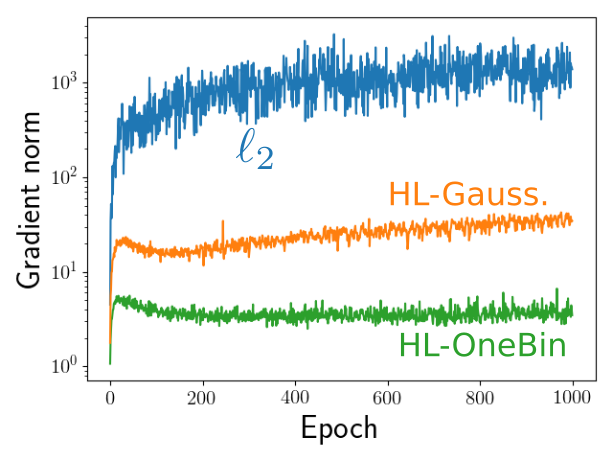}
\label{fig:grad-conv}
 \end{subfigure} 
\caption{Training process on CT Scan. (Left) HL-Gauss reduced the train errors much faster than $\ell_2$. (Right) The gradient norm of $\ell_2$ was highly varying through the training.}
\label{fig:conv_ctscan}
\end{figure*}

The results in Figure \ref{fig:conv_ctscan} suggest that HL-Gauss has more stable gradients than $\ell_2$ and generally reduces the loss by a large amount in the first few epochs. As discussed in Section \ref{sec:theoretical_analysis}, these properties can explain the performance of HL-Gauss. In comparison with HL-OneBin, however, HL-Gauss does not show a noticeable benefit in optimization.

The second experiment explores the role of $\sigma$ in HL-Gauss. The benefit of the parameter $\sigma$ in HL-Gauss can be twofold: (1) improving the final solution and (2) improving the optimization process. To compare the effect of these two properties, we design a network, called Annealing Network, that starts with a high $\sigma$ and gradually reduces it to a small $\sigma$ during training. The initial $\sigma$ is $8$ times the bin width, and during the first $20\%$ of epochs the current value of $\sigma$ is multiplied by a constant $\tau$ at each epoch. The training continues without further annealing once $20\%$ of the epochs are finished. The assumption is that the effect of $\sigma$ on optimization shows up during training. Therefore, if the effect on optimization is substantial, the model that has a high $\sigma$ earlier in training can benefit from it despite its small $\sigma$ at the end. Results are shown in Figure \ref{fig:anneal_ctscan}.

\begin{figure*}[!ht]
\centering
 \begin{subfigure}[b]{\figwidththree}
         \includegraphics[width=\textwidth]{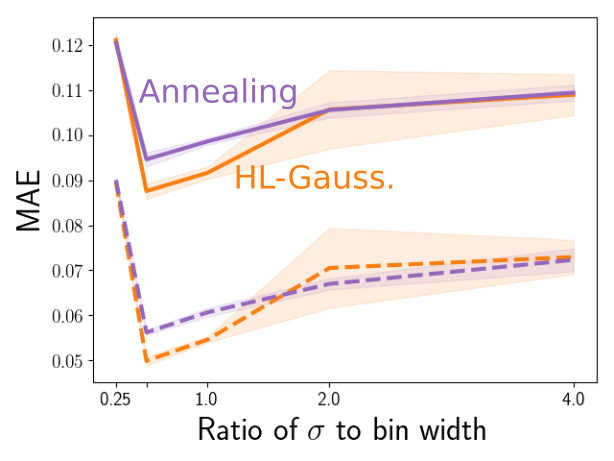}
\end{subfigure} 
 \begin{subfigure}[b]{\figwidththree}
         \includegraphics[width=\textwidth]{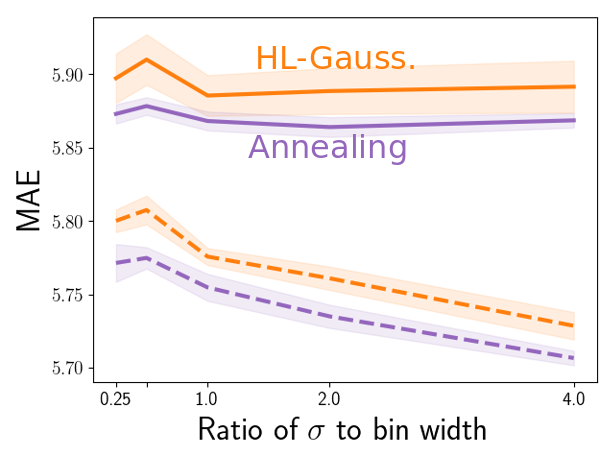}
\end{subfigure} 
  \begin{subfigure}[b]{\figwidththree}
         \includegraphics[width=\textwidth]{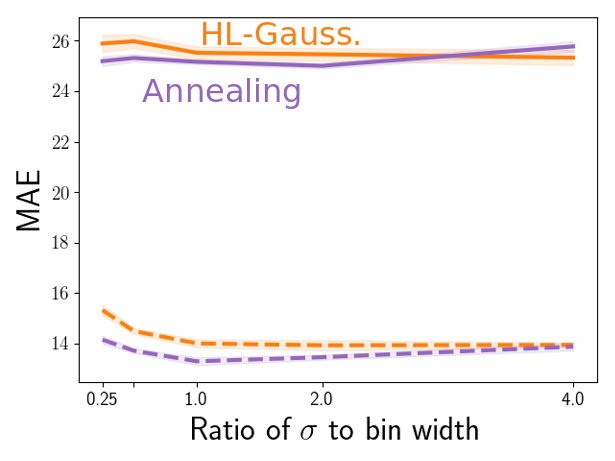}
\end{subfigure} 
\caption{Annealing results on (Left) CT Scan, (Middle) Song Year, and (Right) Bike Sharing. Dotted and solid lines show train and test errors, respectively. The X-axis for Annealing shows the final value of $\sigma$. HL-Gauss keeps $\sigma$ fixed to this value throughout training, while Annealing starts with a higher value of $\sigma$ and reduces it to this value. There is some benefit in starting with a higher $\sigma$ and reducing it on Song Year and Bike Sharing. See Appendix \ref{sec:appdx_results} for Pole and all RMSE results.}
\label{fig:anneal_ctscan}
\end{figure*}

There was some gain in performance when the model started with a large value of $\sigma$ and then reached a small value through annealing compared to training with a fixed small value of $\sigma$. So the parameter $\sigma$ has some impact on the optimization.

The takeaway is that using HL instead of the $\ell_2$ loss helps with optimization, which confirms the previous observation by \citet{imani2018improving}. Further, the choice of the target distribution and the parameter in the loss have some effect on the optimization.

\subsection{Corrupted Targets and Input Perturbations} \label{outlier}

We tested the $\ell_2$ loss, $\ell_1$ loss, HL-OneBin, and HL-Gauss on data sets with corrupted targets. At each level, we replaced a ratio of the training targets with numbers sampled uniformly from the range of targets in the data set. Results in Figure \ref{fig:outliers_sensitivity_ctscan} (Left) show that HL-Gauss and HL-OneBin are more robust to corrupted targets than $\ell_2$ but not as robust as $\ell_1$. These results show a situation where the difference between the performance of HL and the $\ell_2$ loss is more pronounced. Interestingly, HL can be worse than the $\ell_1$ loss in the presence of corrupted targets.

We then compared the sensitivity of the outputs of models trained with the $\ell_2$ loss, $\ell_1$ loss, HL-OneBin, and HL-Gauss to input perturbations. The measure we used was the Frobenius norm of the Jacobian of the model's output with respect to the input. Predictions in a regression problem are scalar values, so the Jacobian becomes a vector whose elements are the derivative of the output with respect to each input feature. A high value of this measure means that a slight perturbation in inputs will result in a drastic change in output. For a classification problem, there are theoretical and empirical results on the connection of this measure and generalization \citep{novak2018sensitivity,sokolic2017robust}. Also, a representation that is less sensitive to input perturbations with this measure has been shown to improve the performance of classifiers \citep{rifai2011contractive}. Following \cite{novak2018sensitivity}, we evaluated this measure at the test points and reported the average. Figure \ref{fig:outliers_sensitivity_ctscan} (Right) summarizes the results.

\begin{figure*}[!ht]
\centering
 \begin{subfigure}[b]{\figwidthtwo}
         \includegraphics[width=\textwidth]{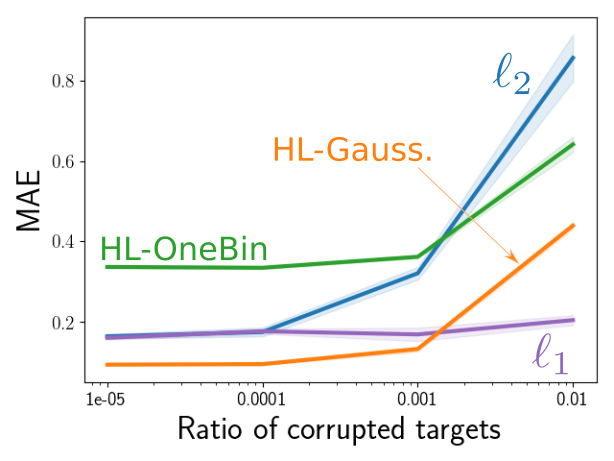}
\end{subfigure} 
  \begin{subfigure}[b]{\figwidthtwo}
         \includegraphics[width=\textwidth]{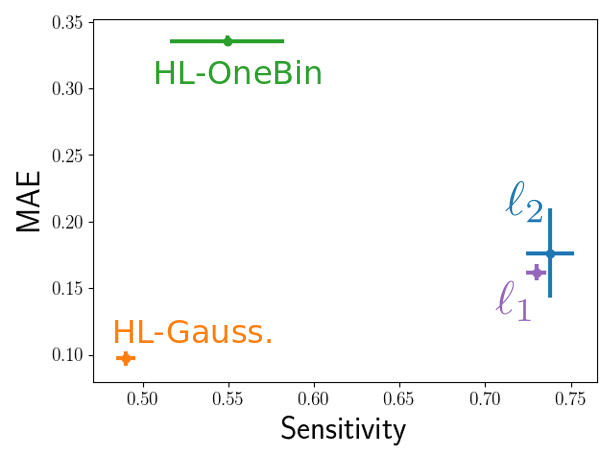}
\end{subfigure} 
\caption{(Left) Corrupted targets on CT Scan. The lines show test errors. The $\ell_1$ loss was robust to corrupted targets, and the $\ell_2$ loss was affected the most. (Right) Sensitivity results on CT Scan. The horizontal axis shows the sensitivity of the model's output to input perturbations (left means less sensitive), and the vertical axis shows the test error (lower is better). HL-Gauss showed less sensitivity and lower test errors than $\ell_1$ and $\ell_2$.}
\label{fig:outliers_sensitivity_ctscan}
\end{figure*}

The results indicate that predictions of the models that are trained with HL are less sensitive to input perturbations than the ones trained with $\ell_1$ or $\ell_2$.

\subsection{Summary of the Empirical Study}

In the empirical study, we observe that
\begin{enumerate}
    \item HL-Gauss, even without per-dataset tuning, is a viable choice in regression and can often outperform the $\ell_2$ loss across datasets (Section \ref{overall}).
    \item The softmax nonlinearity contributes to HL-Gauss's performance, as shown by the results of $\ell_2$+Softmax, though it does not fully account for the improved performance under HL-Gauss. (Section \ref{overall}).
    \item The Gaussian target distribution reduces the discretization bias compared to HL-OneBin, and the remaining bias is bounded and negligible with proper $\sigma$ and padding (Section \ref{bias}).
    \item The results are insensitive to the number of bins for $k > 30$ and to $\sigma$ within the interval $[0.5,4.0]$ times the bin width (Section \ref{bias}).
    \item A standard configuration of 100 bins with $\sigma$ set to twice the bin width and padding of three times $\sigma$ on each side performs well across all datasets without per-dataset tuning (Section \ref{bias}).
    \item Predicting the full distribution does not force better learned representations: HL-Gauss retains its advantages even with a fixed random representation (Section \ref{rep}).
    \item HL-Gauss produces more stable gradients and reduces training error faster than $\ell_2$ (Section \ref{opt}).
    \item HL-Gauss is more robust than $\ell_2$ to corrupted targets, though it is more sensitive to corrupted targets than $\ell_1$ (Section \ref{outlier}).
    \item HL-Gauss produces models whose outputs are less sensitive to input perturbations than both $\ell_1$ and $\ell_2$ (Section \ref{outlier}).

\end{enumerate}

\textbf{Limitations of the study:} These conclusions are given the empirical set-up used in this work. The architecture, stepsize, and epochs was chosen based on $\ell_2$ performance, and it is possible different conclusions could arise with per-loss hyperparmeter tuning. The study also assumed knowledge of the range of the target; though this is typically a reasonable assumption, results may be different if the range had to be estimated (incrementally). Finally, applications of regression are broad, and the data sets in our empirical study and applications sections only cover a small set. Generalizing our findings to other applications such as depth estimation \citep{diode_dataset} or survival prediction \citep{haider2020effective} requires additional investigation. Similarly, emergence of new architectures and foundation models \citep{bommasani2021opportunities} calls for more extensive studies to ask whether our observed patterns extend to the training and fine-tuning of these models.  
\section{Applications} \label{sec:applications}

This section explores the applicability of HL-Gauss beyond the previous simple regression data sets. We evaluate several network architectures on time-series prediction tasks using both HL-Gauss and $\ell_2$ loss. The network choices include linear models, long short-term memory networks (LSTMs, \citep{hochreiter1997long}), gated recurrent units (GRUs, \citep{chung2014empirical}), and vanilla transformers~\citep{vaswani2017transformer}. We then extend the experiments to state-value prediction for reinforcement learning, where we evaluate the performance of both losses on the Atari prediction task.

\subsection{Time Series Prediction}

We use the Electricity Transformer Dataset (ETD,~\citep{haoyi2021informer}) for the time series prediction experiments. The data set contains $6$ features relating to the electricity loads, in addition to the oil temperature, which is the prediction target in all of our experiments. The data is collected from two different transformers and has two variants: hourly data sets, which were collected from each transformer at an hourly rate, and 15-minute data sets, where the data collection happened $4$ times per hour. We consider the multi-step prediction task similar to previous work on this data set~\citep{wu2021autoformer,Zeng2022ltsflinear}, where given a history of length T, the goal is to predict the next k steps of the target (i.e., the oil temperature). In our experiment, we use $T=k=96$.

We investigated the performance of four different network architectures typically used for time-series predictions. These architectures are LSTM~\citep{hochreiter1997long}, GRUs ~\citep{chung2014empirical}, Vanilla Transformer~\citep{vaswani2017transformer}, and a simple linear model suggested by~\cite{Zeng2022ltsflinear}. For each architecture, we had two variations, one trained with HL-Gauss loss and one trained with $\ell_2$ loss. For HL-Gauss, we divided the output range to $100$
bins, set $\sigma$ to twice this bin width, and extended the support with additional padding as described in Appendix $\ref{sec:appdx_application_details}$. 
The same appendix section provides more details on the hyperparameters and the experiment setup.

\begin{table}
\resizebox{\textwidth}{!}{
\begin{tabular}{llcccc}
\toprule
Architecture & Loss & ETTh1 & ETTh2 & ETTm1 & ETTm2 \\
\midrule
Linear & $\ell_2$ & 0.07105 $\pm$ 0.001022 & 0.1624 $\pm$ 0.002041 & 0.03402 $\pm$ 0.001153 & 0.07743 $\pm$ 0.0009791 \\
 & HL-Gauss & 0.3826 $\pm$ 0.003224 & 0.3389 $\pm$ 0.003415 & 0.3243 $\pm$ 0.00557 & 0.2569 $\pm$ 0.00266 \\
\midrule
LSTM & $\ell_2$ & 0.775 $\pm$ 0.03489 & 3.194 $\pm$ 0.2182 & 0.3192 $\pm$ 0.1191 & 0.5605 $\pm$ 0.1366 \\
 & HL-Gauss & \textbf{0.6575 $\pm$ 0.04493} & \textbf{3.165 $\pm$ 0.4313} & \textbf{0.2437 $\pm$ 0.03809} & \textbf{0.3241 $\pm$ 0.08616} \\
\midrule
GRU & $\ell_2$ & 0.5414 $\pm$ 0.02092 & 1.208 $\pm$ 0.3515 & 0.399 $\pm$ 0.0299 & 0.3328 $\pm$ 0.05124 \\
 & HL-Gauss & 0.6733 $\pm$ 0.0863 & \textbf{0.8228 $\pm$ 0.4614} & \textbf{0.2468 $\pm$ 0.02231} & \textbf{0.3325 $\pm$ 0.1206} \\
\midrule 
Transformer & $\ell_2$ & 0.533 $\pm$ 0.02865 & 1.922 $\pm$ 0.3305 & 0.2427 $\pm$ 0.03834 & 0.2097 $\pm$ 0.03613 \\
 & HL-Gauss & 0.6389 $\pm$ 0.1288 & 4.082 $\pm$ 0.7923 & 0.2587 $\pm$ 0.06254 & 0.3529 $\pm$ 0.1232 \\
\bottomrule
\end{tabular}
}
\caption{Test MSE results on the ETD dataset. For each (architecture, loss, data set) tuple, the table shows the mean averaged over $5$ independent seeds and $95\%$ confidence interval.  {\bftab Bold} text indicates better performance for HL-Gauss. HL-Gauss outperforms $\ell_2$ for recurrent architectures (LSTM and GRU) but underperforms $\ell_2$ for linear architecture, and mostly ties on the vanilla transformer.}
\label{tab:ett_results}
\end{table}

Table~\ref{tab:ett_results} shows the test MSE for all the architectures on the 4 data sets in ETD, and we show the test MAE results in Table~\ref{tab:ett_results_mae} in Appendix~\ref{sec:appdx_application_details}. The results suggest that the benefit of HL-Gauss for time-series prediction tasks is architecture-dependent. Recurrent architectures (LSTM and GRU) consistently benefit from HL-Gauss, achieving significantly lower test MSE on most datasets. In contrast, the linear architecture did not benefit from having HL-Gauss. One possible explanation for the architecture dependence is that, for architectures with complex optimization dynamics such as LSTM and GRU, HL-Gauss helps stabilize the gradient dynamics (Section ~\ref{opt_theory}) and, as a result, improves their performance. On the other hand, simple models such as linear models are well-suited for direct $\ell_2$ optimization and thus do not benefit from HL-Gauss.

\renewcommand*{\thefootnote}{\arabic{footnote}}

\subsection{State Value Prediction}

A central problem in reinforcement learning (RL) is state-value prediction. An RL agent interacts with an environment modeled as a Markov decision process, and each action taken by the agent results in a transition, reward, and discounting or possible termination. The goal of state-value prediction is to predict the expected sum of discounted future rewards from a state.

\begin{figure}[!htbp]
    \centering
     \includegraphics[width=\textwidth]{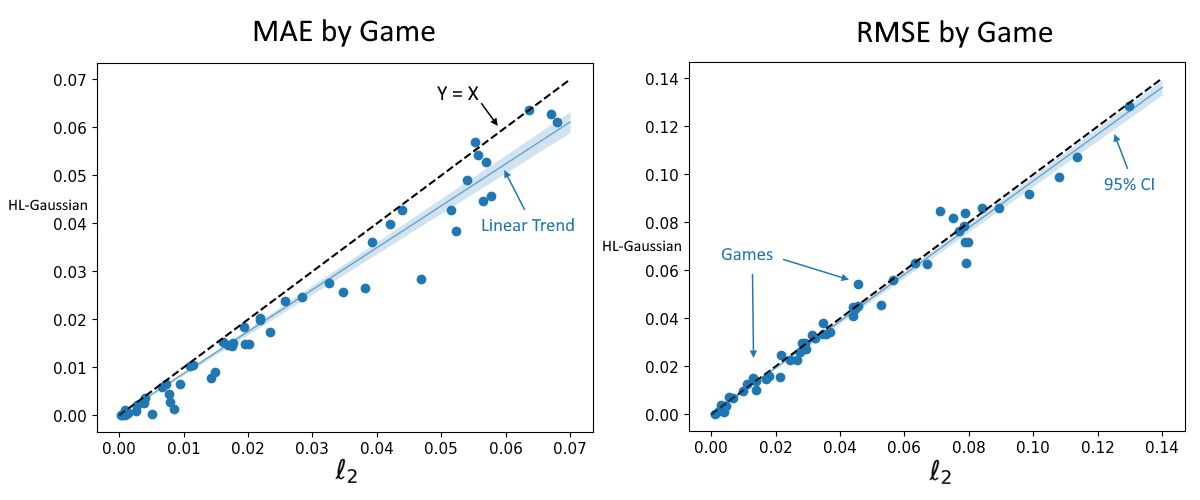}
    \caption{Atari results by game. Points below the dashed black line are games where HL-Gauss outperformed $\ell_2$. The estimated error ratio (blue line) is significant at 95\% confidence for both metrics, indicating that HL-Gauss has a multiplicative advantage across games.}
    \label{fig:atari_graphs}
\end{figure}

In our experiments, we predict state values for agents playing 58 different Atari 2600 games on the Arcade Learning Environment \citep{bellemare2013ale}. In this setting, the state is observed through frames from the screen, and the reward is based on the game score obtained. We use actions from the Atari Prediction Benchmark from \cite{javed2023atari}. This benchmark contains at least 10 million actions for each game. The actions were produced by a pre-trained rainbow DQN agent from the ChainerRL model zoo \citep{fujita2021chainerrl}. While the benchmark paper explores the partially-observable setting with a single frame per observation, we use a stack of 4 frames to make it fully-observable. For HL-Gauss, the bin width is computed using Equation \eqref{eq:hist_bins}. Again, we divide the output range to 100 bins, set $\sigma$ to twice this bin width, and then increase the bin width to add some padding as described in Appendix \ref{sec:appdx_application_details}. We use the first 5\% of the episodes as the test set and train the models for one epoch on the remaining episodes. We exclude 3 games from our analysis in which the agent received no rewards (namely Tennis, Venture, and Montezuma Revenge).

Figure \ref{fig:atari_graphs} shows the error of HL-Gauss against $\ell_2$ across the 55 remaining games. We see that HL-Gauss results in lower errors for most games. To evaluate the statistical significance, we model the error ratio and test for a consistent multiplicative advantage across games. Performing a $t$-test on the log-transformed error ratios, we find moderate ($p=0.0294$) and strong evidence ($p < 10^{-5}$) that HL-Gauss outperforms $\ell_2$ across Atari games in Test RMSE and Test MAE, respectively.

\section{Conclusion and Future Work}  \label{sec:conclusion}

We analyzed a recent approach to regression, the histogram loss (HL) \citep{imani2018improving}, which consists of the KL-divergence between a flexible histogram prediction and target distributions. Experiments on a variety of data sets and networks showed that this loss function often outperforms the oft-used $\ell_2$ loss with a fixed set of recommended hyperparameter values and without the need for careful hyperparameter sweeps. The recommended rule of thumb is using 100 evenly sized bins, setting $\sigma$ to twice the bin width, and then extending the range by three times $\sigma$ on each side padding.

Besides providing a solution for regression, HL offers a testbed for analyzing some of the open questions in previous work. A question that we explored in depth was one of the hypotheses put forward by \citet{bellemare2017distributional}: whether the performance gap between a model that learns the full distribution and one that only estimates the mean can be attributed to their representations, in a manner analogous to multi-task learning. Our experiments suggest that this is not the case in regression. Instead, the bound on the gradient norm of HL as well as its stability during training hint at a better-behaved loss surface as an important factor in effect. An additional factor in HL's performance is the softmax nonlinearity that models the histogram distribution, as we found that minimizing the $\ell_2$ loss between the target and the mean of a histogram distribution could yield an error rate close to that of HL.

Another question was the impact of the type of target distribution on performance. The method for constructing the target distribution has an effect on bias of the loss, and a Gaussian target distribution with a reasonable bandwidth can result in a lower bias than one-hot targets. However, it is also important to note that this is not the only benefit of a Gaussian target distribution, since even HL with an unbiased target distribution underperformed HL-Gauss.

Our results on the behavior of different models during optimization are still preliminary, and a more in-depth investigation is needed for a crisp conclusion. There is a growing literature on understanding loss surfaces of neural networks. Theoretical analyses are often insightful but require strong assumptions on the data or model \citep{soltanolkotabi2018theoretical,choromanska2015loss,jacot2018neural,arora2018optimization}. An alternative line of research compares loss surfaces empirically by techniques like projection on lower dimensions or slightly perturbing the parameters in random directions, often to compare the flatness of the solutions \citep{hochreiter1997flat,keskar2016large,li2018visualizing}. There are three important challenges in performing such comparison. First, intuitions about optimization in low-dimensional surfaces do not necessarily transfer to higher dimensions. Second, sometimes (like when comparing HL and $\ell_2$ loss) the surfaces compared are of different dimensionalities, and this difference can be a confounding factor. Third, a causal connection between the performance of the network and some measures of flatness is disputed, as solutions with similar performance but varying levels of flatness can be easily constructed \citep{dinh2017sharp}.

Despite the challenges above, investigating possible connections between performance gains in distribution learning and flatness is promising. Training regimes and optimizers that implicitly or explicitly enforce flatness have shown better final performance, an initial phase of fast learning, higher robustness to some degree of label noise, and less sensitivity to input perturbations \citep{keskar2016large,hoffer2017train,foret2020sharpness,jiang2019fantastic,novak2018sensitivity}. In our empirical study, we found HL-Gauss had these exact properties.

The experiment settings and baselines in this paper were focused on singling out and studying the points of difference between HL and the $\ell_2$ loss. Further experiments can be designed to evaluate intersections of these factors. As a simple example, one could switch the $\ell_2$ loss to the $\ell_1$ loss in our $\ell_2$+softmax baseline. The new model can benefit from robustness to outliers and the nonlinearity in the softmax output layer.

In terms of applications, a future direction would be finding problem-specific or even adaptive discretization methods. Breaking down the target range to even bins might not be reasonable in some cases. For example, consider the task of predicting the number of views of an online video. The presence of viral videos that heavily stretch the range of targets in the data set will face a model designer with a trilemma: (1) choosing a wide support that covers the targets and finely discretizing it which results in a network with an enormous last layer, (2) lowering the number of bins and turning to a coarse discretization which loses information about small differences between the targets, or (3) truncating the support which requires drawing a line between informative data points and outliers and may also excessively limit the range of values the network can predict. An investigation into best practices or even automatic discretization methods could make HL easier to use in such problems. 
\acks{We would like to thank the Natural Sciences and Engineering Research Council of Canada (NSERC), the Canada CIFAR AI Chair Program, and Amii for research funding and the Digital Research Alliance of Canada for the computational resources. K.L. is supported by the Konrad Zuse School of Excellence in
Learning and Intelligent Systems (\href{https://eliza.school}{ELIZA}) through the DAAD programme
Konrad Zuse Schools of Excellence in Artificial Intelligence, sponsored by the
Federal Ministry of Education and Research. We also thank Nick Janßen for pointing out an issue in a previous version of the time-series experiments.}

\clearpage
\appendix

\section{Proofs} \label{sec:appdx_proofs}

\proplipschitz*

\begin{proof}
Let $b_i = \phivec_\thetavec(\xvec)^\top \wvec_i$ and $e_i = \exp(b_i)$. Then, since $\hist_j(\xvec) =  \frac{e_j}{\sum_{l=1}^\nbins e_l}$, for $j \neq i$
\begin{align*}
\frac{\partial}{\partial b_i} \hist_j(\xvec) 
&= \frac{\partial}{\partial b_i} \frac{e_j}{\sum_{l=1}^\nbins e_l} 
= -\frac{e_j}{\left(\sum_{l=1}^\nbins e_l\right)^2} e_i   
= - \hist_j(\xvec) \hist_i(\xvec) 
\end{align*}
For $j = i$, we get
\begin{align*}
\frac{\partial}{\partial b_i} \hist_j(\xvec) 
&= \frac{e_i}{\sum_{l=1}^\nbins e_l} - \frac{e_i}{\left(\sum_{l=1}^\nbins e_l\right)^2} e_i
= \hist_i(\xvec)[1 - \hist_i(\xvec)]  
\end{align*}
Consider now the gradient of HL, with respect to $b_i$
\begin{align*}
\frac{\partial}{\partial b_i} \sum_{j=1}^\nbins \distweight_j \log \hist_j(\xvec) 
&= \sum_{j=1}^\nbins \distweight_j \frac{1}{\hist_j(\xvec)} \hist_j(\xvec) (1_{i = j} - \hist_i(\xvec))  \\
&= \sum_{j=1}^\nbins \distweight_j (1_{i = j} - \hist_i(\xvec) )
=  \distweight_i - \hist_i(\xvec) \sum_{i=1}^\nbins \distweight_i 
=  \distweight_i - \hist_i(\xvec)
\end{align*}
Then
\begin{align*}
\frac{\partial}{\partial \wvec_i} \sum_{j=1}^\nbins \distweight_i \log \hist_i(\xvec) 
&=  \left(\distweight_i - \hist_i(\xvec) \right) \phivec_\thetavec(\xvec)\\
\frac{\partial}{\partial \thetavec} \sum_{j=1}^\nbins \distweight_i \log \hist_i(\xvec) 
&=  \sum_{i=1}^\nbins  \left(\distweight_i - \hist_i(\xvec) \right) \nabla \wvec_i^\top \phivec_\thetavec(\xvec).
\end{align*}

The norm of the gradient of HL in Equation $\eqref{eq_hl}$, with respect to $\wvec$ which is composed of all the weights $\wvec_i \in \RR^\nbins$ is
\begin{align*}
 \Big\| \frac{\partial}{\partial \wvec} \sum_{j=1}^\nbins \distweight_j \log \hist_j(\xvec) \Big\|
&\le  \sum_{i=1}^\nbins \Big\| \frac{\partial}{\partial \wvec_i} \sum_{j=1}^\nbins \distweight_j \log \hist_j(\xvec) \Big\|\\
&= \sum_{i=1}^\nbins  \left\| (\distweight_i -\hist_i(\xvec) ) \phivec_\thetavec(\xvec) \right\|
\le  \sum_{i=1}^\nbins | \distweight_i -\hist_i(\xvec) | \| \phivec_\thetavec(\xvec) \| 
\end{align*}
Similarly, the norm of the gradient with respect to $\thetavec$ is
\begin{align*}
\Big\| \frac{\partial}{\partial \thetavec} \sum_{j=1}^\nbins \distweight_j \log \hist_j(\xvec)\Big\| 
&= \Big\| \sum_{i=1}^\nbins (\distweight_i -\hist_i(\xvec) ) \nabla_\thetavec \wvec_i^\top\phivec_\thetavec(\xvec) \Big\|\\
&\le \sum_{i=1}^\nbins \left\| (\distweight_i -\hist_i(\xvec) ) \nabla_\thetavec \wvec_i^\top\phivec_\thetavec(\xvec) \right\|
\le \sum_{i=1}^\nbins | \distweight_i -\hist_i(\xvec) | \lip 
\end{align*}
Together, these bound the norm $\|  \nabla_{\thetavec,\wvec} HL(\td, \pd)  \|$.
\end{proof}

\propbias*

\begin{proof}
The integrals in this proof are over the whole real line unless stated otherwise. We use Fubini's theorem when we change the orders of integrals. 

The function $s$ in Statement 1 is a proper pdf since it is positive at each point and sums to one:
\begin{align*}
    \int s(z) \diff z &= \int \int \rtd(y) \tdy_y(z) \diff y \diff z = \int \rtd(y) (\int \tdy_y(z) \diff z) \diff y = \int \rtd(y) \diff y = 1
\end{align*}
Recall that HL is the cross-entropy between the predicted distribution and the target distribution. We denote the cross-entropy between $\tdy_y$ and $\pdnox$ as $\ent(\tdy_y,\pdnox)$. We now find the expected value of the cross-entropy under $\rtd$, and the prediction distribution that minimizes it:
\begin{align*}
    \EX_{y\sim \rtd}[\ent(\tdy_y,\pdnox)] &= \int \rtd(y) \int \tdy_y(z) \log \pdnox(z) \diff z \diff y  = \int (\int \rtd(y) \tdy_y(z) \diff y) \log \pdnox(z) \diff z \\
    & = \int s(z) \log \pdnox(z) \diff z = \ent(s, \pdnox)
\end{align*}
So, as long as $s$ is in our function class for the prediction distribution, the prediction distribution that minimizes the average cross-entropy is not $\rtd$, but $s$ \textbf{(Statement 1)}. Formally
\begin{align*}
    \argmin_{\pdnox \in \mathcal{P}({\RR})} \EX_{y\sim \rtd} [D_{KL}(\tdy_y||\pdnox)] = \argmin_{\pdnox \in \mathcal{P}({\RR})} \EX_{y\sim \rtd}[\ent(\tdy_y,\pdnox)] = \argmin_{\pdnox \in \mathcal{P}({\RR})} \ent(s, \pdnox) = s
\end{align*}

In the case of the Dirac delta target distribution, $s$ is the same as $\rtd$ and for a Gaussian target distribution, $s$ is the result of Gaussian kernel smoothing on $\rtd$.

If $s$ is within our function class, although a Gaussian target distribution will change the minimizer from $\rtd$ to $s$, it will not bias the predicted mean, as we show below. Suppose $g(.|\mu,\sigma^2)$ is the pdf of a Gaussian distribution with mean $\mu$ and variance $\sigma^2$. Let us write the bias of $s$: 
\begin{align*}
    \EX_{s}[z] - \EX_{\rtd}[y] &= \int z s(z) \diff z - \EX_{\rtd}[y] 
    = \int z \int \rtd(y) \tdy_y(z) \diff y \diff z - \EX_{\rtd}[y]\\
    &= \int \rtd(y) \int z \tdy_y(z) \diff z \diff y - \EX_{\rtd}[y] 
    = \EX_{y\sim\rtd}[\EX_{z\sim\tdy_y}[z] - y] \\
    &= \int \rtd(y) \int z g(z|y, \sigma^2) \diff z \diff y - \EX_{\rtd}[y]
    = \int \rtd(y) y \diff y - \EX_{\rtd}[y]
    = 0
\end{align*}
The second line above proves \textbf{Statement 2} and the third line shows the claim for Gaussian target distribution.

The prediction function $\pdnox$ is restricted to a histogram in HL-Gauss. This restriction can bias the predictions in at least two ways. The first source of bias is due to the fact that $s$ is truncated to match the support of $\pdnox$. We do not explore truncation here, and assume that the support of $\pdnox$ is $(-\infty,\infty)$. The second source of bias is a result of using discrete bins for prediction which is present even in the case of infinite support. We call this bias \textit{discretization bias} and quantify it.

We consider the case where all bins have equal width $w$. The center of bin $i$ and the probability of a histogram distribution $\pdnox$ in bin $i$ are denoted by $\cen_i$ and $\pdnox_i$ respectively and $\mathcal{S}_i \coloneqq [\cen_i - \frac{w}{2}, \cen_i + \frac{w}{2})$ is the range of bin $i$. 

Note that in this case the cross-entropy between $\pdnox$ and $s$ is $\sum_i (\int_{\mathcal{S}_i} s(z) \diff z) \log(\pdnox_i)$. Thus the histogram density that minimizes the cross-entropy to $s$ is one where the probability in each bin is equal to the probability of $s$ in the range of that bin (\textbf{Statement 3}). We call this density $\pdnox^*$ and find the difference between the expected values of $\pdnox^*$ and $\rtd$.

\begin{align*}
    \EX_{\pdnox^*}[z] - \EX_\rtd[y] &=\sum_i \cen_i \int_{\mathcal{S}_i} s(z) \diff z - \int y \rtd(y) \diff y \\
    &=\sum_i \cen_i \int_{\mathcal{S}_i} \int \rtd(y) \tdy_y(z) \diff y \diff z - \int y \rtd(y) \diff y \\
    &=\int \rtd(y) \sum_i \cen_i \int_{\mathcal{S}_i} \tdy_y(z) \diff z \diff y - \int y \rtd(y) \diff y \\
    &=\int \rtd(y) \biggl(\sum_i \cen_i \int_{\mathcal{S}_i} \tdy_y(z) \diff z - y\biggr) \diff y
\end{align*}
The factor inside the parantheses is the difference between $y$ and the expected value of the histogram density with the lowest KL-divergence from $\tdy_y$. We will find this difference for the case of a Gaussian target distribution. This error depends on $y$ and, once this error is known, the highest discretization bias can be found by choosing a distribution $\rtd$ that has all its probability on the value of $y$ that results in the highest error. Also note that, if $\tdy_y$ is chosen so that the closest histogram density to it has the expected value of $y$, this error is zero and the whole discretization bias becomes zero. The proof below shows that discretization bias of HL-Gauss is bounded by $\pm \frac{w}{2}$ and depends on the choice of the parameter $\sigma$. 

We will denote with $\cen_0$ the center of the bin that contains $y$. Other bin centers will be denoted by $\cen_i$ where $\cen_i \coloneqq \cen_0 + iw$. Also, $\delta \coloneqq (y-\cen_0)$. Note that $\delta \in [-\frac{w}{2}, \frac{w}{2})$. Finally, CDF of a Gaussian with mean $\mu$ and variance $\sigma^2$ is denoted by $\Phi_{\mu}$ and $F_{\mu}(a,b) \coloneqq \Phi_{\mu}(b) - \Phi_{\mu}(a)$. (Subscripts to show the dependence of $\Phi$ and $F$ on $\sigma$ are dropped for convenience.)
\begin{align}
    &\sum_{i=-\infty}^{\infty} \cen_i \int_{\mathcal{S}_i} g(z|y,\sigma^2) \diff z - y 
    =\sum_{i=-\infty}^{\infty} (\cen_0 + iw) \int_{\mathcal{S}_i} g(z|y,\sigma^2) \diff z - y \nonumber \\
    &=\cen_0 \sum_{i=-\infty}^{\infty} \int_{\mathcal{S}_i} g(z|y,\sigma^2) \diff z + w\sum_{i=-\infty}^{\infty} i \int_{\mathcal{S}_i} g(z|y,\sigma^2) \diff z - y \label{eq:sum_s_i}
\end{align}
Recall that intervals $\mathcal{S}_i$ are disjoint and $\bigcup\limits_{i=-\infty}^{\infty}\mathcal{S}_i = (-\infty, \infty)$. So the series in the first term of \eqref{eq:sum_s_i} becomes one. Therefore
\begin{align}
    \eqref{eq:sum_s_i} =&(\cen_0 - y) + w\biggl(\sum_{i=-\infty}^{-1} i \int_{\mathcal{S}_i} g(z|y,\sigma^2) \diff z + \sum_{i=1}^{\infty} i \int_{\mathcal{S}_i} g(z|y,\sigma^2) \diff z\biggr) \nonumber \\
    =&-\delta + w\biggl(\sum_{i=-\infty}^{-1} i F_{y}(\cen_0 + iw -\frac{w}{2},\cen_0 + iw +\frac{w}{2}) \nonumber \\
    &+ \sum_{i=1}^{\infty} i F_{y}(\cen_0 + iw -\frac{w}{2},\cen_0 + iw +\frac{w}{2})\biggr) \nonumber \\
    =&-\delta + w\biggl(\sum_{i=1}^{\infty} (-i) F_{y}(\cen_0 - iw -\frac{w}{2},\cen_0 - iw +\frac{w}{2}) \nonumber \\
    &+ \sum_{i=1}^{\infty} i F_{y}(\cen_0 + iw -\frac{w}{2},\cen_0 + iw +\frac{w}{2})\biggr) \nonumber \\
    =&-\delta + w \biggl(\sum_{i=1}^{\infty} i (F_{y}(\cen_0 + iw -\frac{w}{2},\cen_0 + iw +\frac{w}{2}) \label{eq:series} \\ 
    &-F_{y}(\cen_0 - iw -\frac{w}{2},\cen_0 - iw +\frac{w}{2}))\biggr) \nonumber
\end{align}
Due to the symmetry of Gaussian distribution, $F_{\mu}(a,b) = F_{\mu}(2\mu-b, 2\mu-a)$. We define $a_i\coloneqq iw - \frac{w}{2}$ and $b_i\coloneqq iw + \frac{w}{2}$ and the series in \eqref{eq:series} becomes
\begin{align}
    &\sum_{i=1}^{\infty} i (F_{y}(\cen_0 + a_i,\cen_0 + b_i) -F_{y}(\cen_0 - b_i,\cen_0 - a_i)) \nonumber \\
    =&\sum_{i=1}^{\infty} i (F_{y}(\cen_0 + a_i,\cen_0 + b_i) - F_{y}(2y - \cen_0 + a_i, 2y - \cen_0 + b_i))\nonumber \\
    \stackrel{\delta := y - \cen_0}{=}&\sum_{i=1}^{\infty} i (F_{y}(y - \delta + a_i,y - \delta + b_i) - F_{y}(y + \delta + a_i, y + \delta + b_i))\nonumber \\
    =&\sum_{i=1}^{\infty} i (\Phi_{y}(y - \delta + b_i) - \Phi_{y}(y - \delta + a_i)   + \Phi_{y}(y + \delta + a_i) - \Phi_{y}(y + \delta + b_i))\nonumber \\
    =&\sum_{i=1}^{\infty} i (-F_{y}(y + b_i- \delta ,y + b_i + \delta) + F_{y}(y + a_i - \delta , y + a_i + \delta))\nonumber \\
    =&-\sum_{i=1}^{\infty} i F_{y}(y + b_i - \delta ,y + b_i + \delta) + \sum_{i=1}^{\infty} i F_{y}(y + a_i - \delta , y + a_i + \delta) \label{eq:before_index}
\end{align}
Since $iw + \frac{w}{2} = (i+1)w - \frac{w}{2}$, we can replace $b_i$ by $a_{i+1}$ and have
\begin{align}
    \eqref{eq:before_index} =&-\sum_{i=1}^{\infty} i F_{y}(y + a_{i+1}- \delta ,y + a_{i+1} + \delta) + \sum_{i=1}^{\infty} i F_{y}(y + a_i - \delta , y + a_i + \delta) \nonumber \\
    =&-\sum_{i=2}^{\infty} (i-1) F_{y}(y + a_i- \delta ,y + a_i + \delta) + \sum_{i=1}^{\infty} i F_{y}(y + a_i - \delta , y + a_i + \delta) \nonumber \\
    =&\sum_{i=2}^{\infty} F_{y}(y + a_i- \delta ,y + a_i + \delta) + F_{y}(y + a_1 - \delta , y + a_1 + \delta) \label{eq:telescopic}\\
    =&\sum_{i=1}^{\infty} F_{y}(y + a_i- \delta ,y + a_i + \delta)
    =\sum_{i=1}^{\infty} F_{0}(a_i- \delta , a_i + \delta)\nonumber\\
    =&\sum_{i=1}^{\infty} F_{0}(iw -\frac{w}{2}- \delta ,iw -\frac{w}{2} + \delta) \label{eq:series_final}
\end{align}
where we used the method of differences to obtain \eqref{eq:telescopic}. We can now plug \eqref{eq:series_final} into \eqref{eq:series} to have
\begin{align}
    \eqref{eq:series} =&-\delta + w \biggl(\sum_{i=1}^{\infty} F_{0}(iw -\frac{w}{2}- \delta ,iw -\frac{w}{2} + \delta)\biggr) \label{eq:final}
\end{align}
The sum in \eqref{eq:final} consists of the area under a Gaussian pdf in some intervals. Since $\delta \in [-\frac{w}{2}, \frac{w}{2}]$ the sum in the second term can be at most $0.5$ (when $\delta=\frac{w}{2}$ and the intervals connect to each other to form the positive half of a Gaussian pdf) and at least $-0.5$ (when $\delta=-\frac{w}{2}$ and the interval bounds are reversed). So, given the bounds on $\delta$ and the sum, we can see the overall error cannot exceed $[-w, w]$.

However, as long as $\delta > 0$, each term in the sum is positive and the sum is positive (because it is a part of a Gaussian pdf) and, when $\delta < 0$ the terms in the sum become negative (because the integral bounds are flipped). So the two terms in the sum will have opposite signs and will not add to each other, and the error is bounded by $[-\frac{w}{2},\frac{w}{2}]$. This is a tight bound since for example when $\sigma \rightarrow 0$ (and the loss becomes HL-OneBin) and $\delta = \pm \frac{w}{2}$ then $\pdnox^*$ will have all its density on the bin that contains $y$ and the error is $\cen_0 - y = -\delta = \mp \frac{w}{2}$ \textbf{(Statement 4)}.

\end{proof}

\proppredbound*

\begin{proof} We drop the subscript $y$ for brevity.
\begin{align}
    (\EX_{\td}[z] - \EX_{\pd}[z])^2 =& (\int_a^b (\td(z) - \pd(z)) z \diff z)^2 \leq \bigl(\int_a^b |\td(z) - \pd(z)| \diff z \max(|a|,|b|)\bigr)^2 \label{eq:predbound_appdx}
\end{align}
The inequality is tight when $a=-1, b=1, \td=\delta(-1), \pd=\delta(1)$. The integral on the right-hand side is twice the Total-Variation (TV) distance \citep[p.~437]{zhang2023mathematical}. Pinsker's inequality and Bretagnolle–Huber inequality relate TV distance and KL-divergence in these two ways:
\begin{align*}
    \frac{1}{2} \int_a^b |\td(z) - \pd(z)| \diff z \leq \sqrt{\frac{1}{2} D_{KL}(\td||\pd)}, \quad \frac{1}{2} \int_a^b |\td(z) - \pd(z)| \diff z \leq \sqrt{1 - e^{-D_{KL}(\td||\pd)}}
\end{align*}
\cite[p.~192]{lattimore2020bandit} and \cite{canonne2022short} show proofs for these bounds and discuss when each bound is preferred. Using these bounds we have
\begin{align*}
    \eqref{eq:predbound_appdx} &= 4\bigl(\frac{1}{2}\int_a^b |\td(z) - \pd(z)| \diff z\bigr)^2 \bigl(\max(|a|,|b|)\bigr)^2\\
    &\leq 4 \max(|a|,|b|)^2 \min\biggl(\frac{1}{2} D_{KL}(\td||\pd), 1 - e^{-D_{KL}(\td||\pd)} \biggr)
\end{align*}
\end{proof} 
\section{Prediction Bounds}\label{sec:appdx_bounds}

The two quantities in the minimum are obtained using Pinsker's inequality and Bretagnolle–Huber (BH) inequality and yield tighter bounds for small and large values of KL-divergence respectively \citep{lattimore2020bandit,canonne2022short}. To better understand these terms and the overall bound, consider the example where $a = -1, b = 1, \td_y = \delta(-1), \pd = \epsilon\delta(-1) + (1-\epsilon)\delta(1)$ for $\epsilon \in (0, 1)$. As $\epsilon$ rises, the two distributions become more similar to each other and both the squared error between their means (the left-hand side in Equation \eqref{eq:predbound}) and the KL-divergence between the distributions decreases. Figure \ref{fig:lower} shows the bounds obtained with Pinsker's inequality and BH inequality (the first and the second term in the minimum) as well as the left-hand side. Both bounds decrease to zero as the two distributions approach each other. Pinsker's inequality is preferred when the distributions are closer to each other. For small values of $\epsilon$, the bound using Pinsker's inequality exceeds a trivial bound of $(b-a)^2 = 4$, while the bound with BH inequality never exceeds this value.

\section{Bias Simulation} \label{sec:bias_simulation}

\begin{figure}[t]
    \centering
    \includegraphics[width=0.45\textwidth]{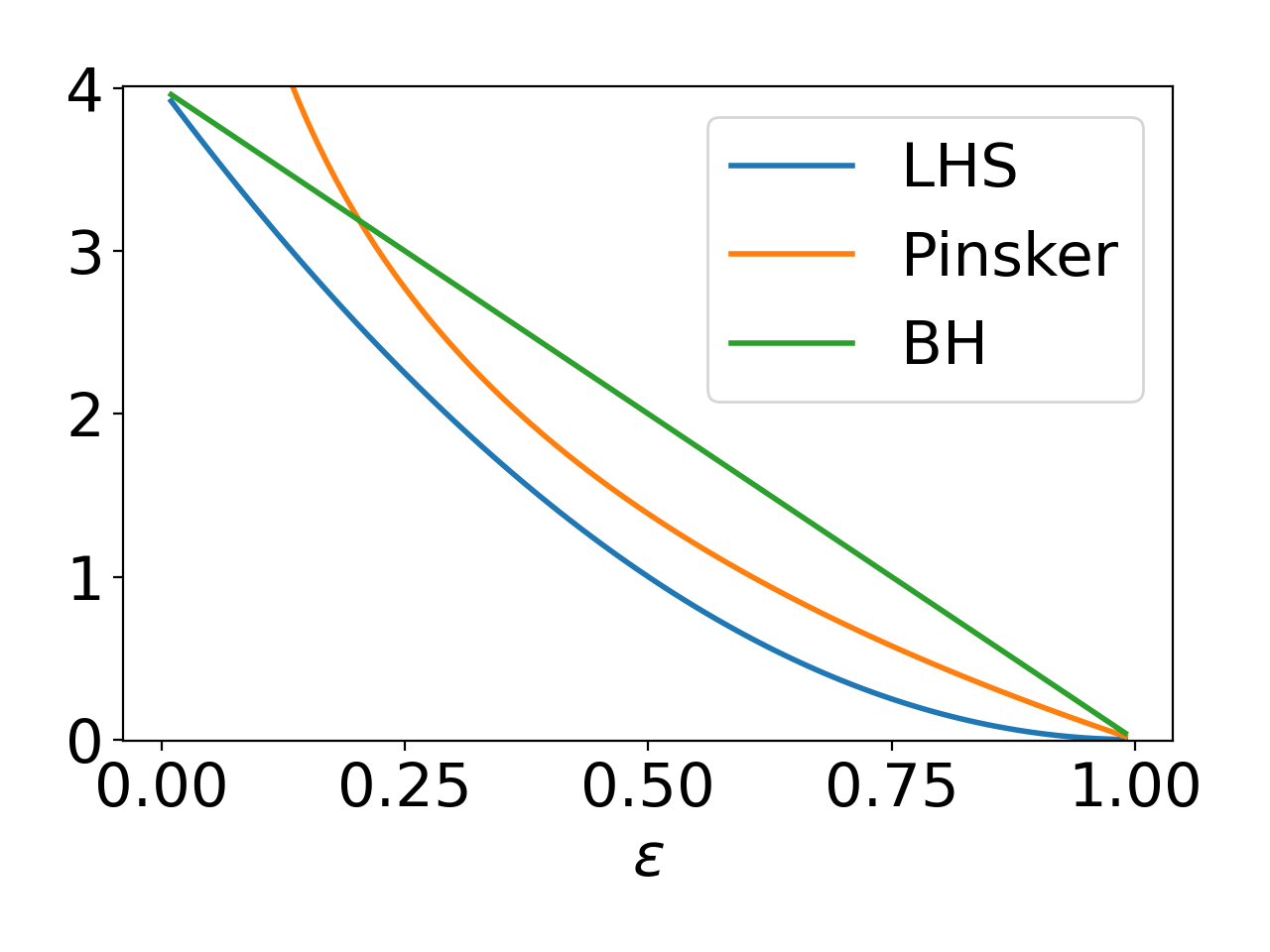}
    \caption{The behavior of the bound in Equation \eqref{eq:predbound} on an example with Dirac delta functions. LHS shows the left-hand side, and the other two curves show the bounds obtained with Pinsker's and BH inequality.}
    \label{fig:lower}
\end{figure}

When creating a discrete truncated Gaussian from a target y, bias is introduced as the mean of this distribution is typically not equal to y. The true mean of the histogram distribution, $\hat{y}$, can be found using the following equation:
\begin{equation}\label{hist_bias}
    \hat{y}=\sum_{i=1}^{k} c_i\left(\frac{\text{erf}\left(\frac{c_i+\frac{w}{2}-y}{\sqrt{2}\sigma}\right)-\text{erf}\left(\frac{c_i-\frac{w}{2}-y}{\sqrt{2}\sigma}\right)}{\text{erf}\left(\frac{b-y}{\sqrt{2}\sigma}\right)-\text{erf}\left(\frac{a-y}{\sqrt{2}\sigma}\right)}\right)
\end{equation}
where the truncated normal distribution is contained within the interval $(a, b)$, $k$ is the number of bins, $c_i$ is the value of the center of the $i$-th bin, $w$ is the width of all bins and $\sigma$ is the smoothness parameter of the truncated Gaussian. 

We used simulation to approximate the bias as a function of $\sigma$ and padding. We did this by generating simulated data, transforming it to a discretized truncated Gaussian distribution, and comparing the mean of this distribution to the original value. We first attempted to isolate the discretization bias by computing the total bias and using large padding, since the truncation bias vanishes as the support for the distribution grows to $(-\infty, \infty)$. By generating $10^5$ equally spaced data points in the range $[0, 1]$ divided into $k=100$ bins of equal width and varying $\sigma$ we can plot the discretization bias as a function of $\sigma_w$, defined as $\sigma/(1/k)$, and the relative position within a bin. We used padding of $\psi = 100$. Define $\psi_\sigma$ as $\psi/\sigma$. The bin width is then adjusted as in the following equation:
\begin{equation}\label{eq:hist_bins}
    w = \frac{y_{max} - y_{min}}{k - 2\sigma_w \psi_\sigma}
\end{equation}
Figure \ref{fig:bias_simulation} (left) displays the results. We define $\delta$ as the difference between the sample and the center of the bin and the offset as $\frac{\delta}{w}$. The X-axis ($\sigma_w$) has been log-transformed. As expected by the symmetry of the Gaussian distribution, we have zero bias for the center of the bin. The relationship with the bin offset converges to linear as $\sigma_w \to 0$, which corresponds to convergence towards HL-Onebin. The error also decreases to zero as $\sigma_w \to \infty$. Taking the mean absolute bias over all samples, we obtain Figure \ref{fig:bias_simulation} (middle). The bias reaches a minimum at around $\sigma_w \approx 1.35$, where it becomes dominated by 64-bit floating point precision error.

We can also observe the total error for different values of $\psi_\sigma$ in Figure \ref{fig:bias_simulation} (right). Fixing $\sigma_w = 2$, the discretization error is negligible. Again, we see squared-exponential decay as the padding increases until $\psi_\sigma = 8.5$, after which the bias is dominated by floating point precision error. Overall, for $\sigma_w > 0.8$ and $\psi_\sigma > 6$, we have bias less than $10^{-7}$ which is negligible when using 32-bit floating point numbers. Empirically, the effect of these parameters on model learning has a greater impact on performance, but these graphs can be used to approximate the bias after scaling by the bin width.

\begin{figure*}[!ht]
\centering
 \begin{subfigure}[b]{\figwidththree}
         \includegraphics[width=\textwidth]{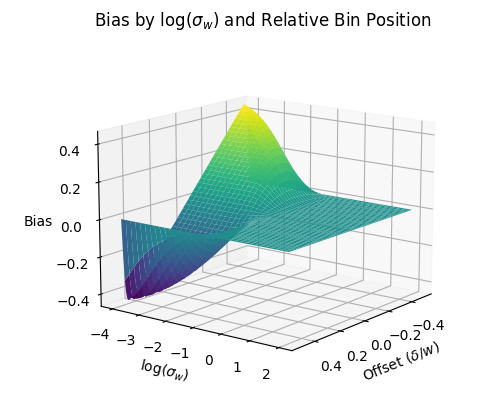}
\end{subfigure} 
 \begin{subfigure}[b]{\figwidththree}
         \includegraphics[width=\textwidth]{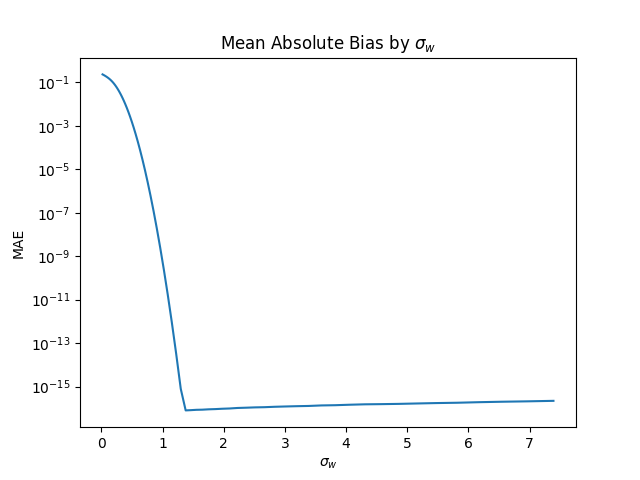}
\end{subfigure} 
  \begin{subfigure}[b]{\figwidththree}
         \includegraphics[width=\textwidth]{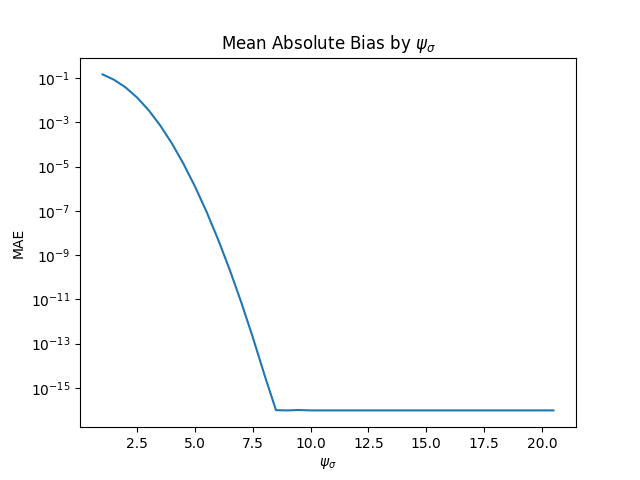}
\end{subfigure} 
\caption{Bias simulation results. (Left) Discretization bias, (Middle) Mean Absolute Discretization Bias, (Right) Mean Absolute Truncation Bias.}
\label{fig:bias_simulation}
\end{figure*} 
\section{Empirical Study - Details} \label{sec:appdx_exp_details}

\subsection{Data Set details}

Table \ref{tab:Dataset Table} and Figure \ref{fig:histograms} show the details of the four regression data sets.

\begin{table}[ht]
\begin{center}
\begin{tabular}{|c c c c c|}
 \hline
 Data set & \# train & \# test & \# feats & $Y$ range \\ [0.5ex] 
 \hline\hline
 Song Year & 412276 & 103069 & 90 & [1922,2011] \\ 
 \hline
 CT Position & 42800 & 10700 & 385 & [0,100] \\ [0ex] 
 \hline
 Bike Sharing & 13911 & 3478 & 16 & [0,1000] \\ [0ex] 
 \hline
  Pole & 12000 & 3000 & 49 & [0,100] \\ [0ex] 
 \hline
\end{tabular}
\caption{Overview of the data sets used in the experiments.}\label{tab:Dataset Table}
\end{center}
\end{table}

\begin{figure}[ht]
    \centering
\begin{minipage}{0.4\textwidth}
        \centering
        \includegraphics[width=0.9\textwidth]{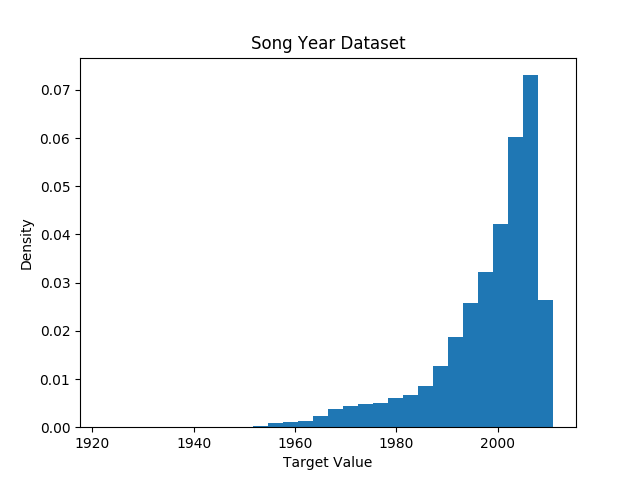}
    \end{minipage}
    \begin{minipage}{0.4\textwidth}
        \centering
        \includegraphics[width=0.9\textwidth]{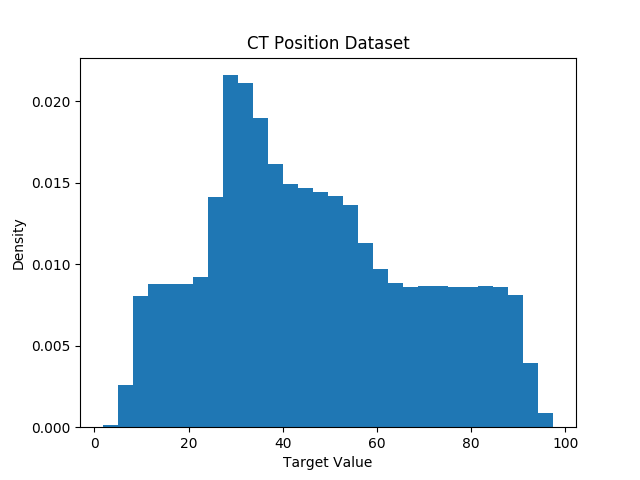}
    \end{minipage}
    \begin{minipage}{0.4\textwidth}
       	\centering
       	\includegraphics[width=0.9\textwidth]{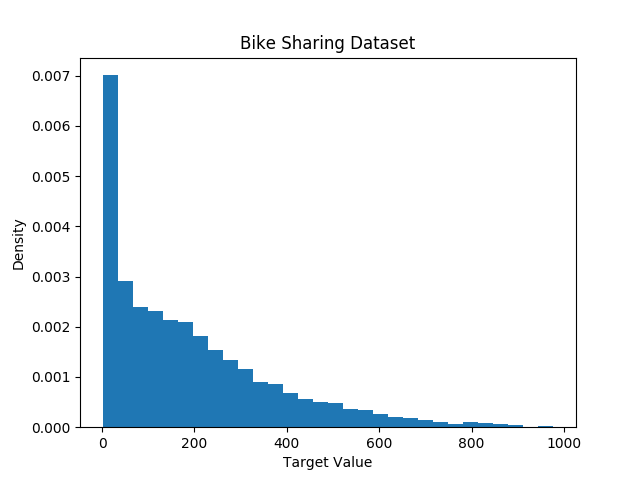}
    \end{minipage}
    \begin{minipage}{0.4\textwidth}
       	\centering
       	\includegraphics[width=0.9\textwidth]{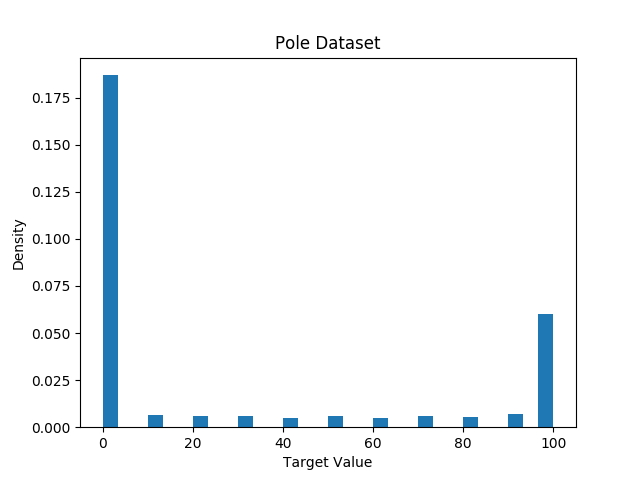}
    \end{minipage}
    \caption{Histograms of the target values for the four data sets. Each plot shows a histogram with 30 bins, where the area of a bin is the ratio of targets in the data set that fall in that bin. Note that it is the total area of bins, and not the sum of heights, that equals one.}
\label{fig:histograms}
\end{figure}

\subsection{Architectures and Hyperparameters}

The architecture for CT Position is 385-192-192-192-192-1, for Song Year is 90-45-45-45-45-1 (4 hidden layers of size 45), for Bike Sharing is 16-64-64-64-64-1, and for Pole is 49-24-24-24-1. All hidden layers employ ReLU activation. Network architectures were chosen according to best Test MAE for $\ell_2$, with depth and width varied across 7 different values.

Some of the baselines required extra hyperparameters. The range of these hyperparameters are described below:\\
\textbf{$\mathbf{\ell_2}$+Noise} The standard deviation of the noise is selected from $\{10^{-5}, 10^{-4}, 10^{-3}, 10^{-2}, 10^{-1}\}$. \\
\textbf{$\mathbf{\ell_2}$+Clipping} The threshold for clipping is selected from $\{0.01, 0.1, 1, 10\}$.\\
\textbf{HL-Uniform} $\epsilon \in \{10^{-5}, 10^{-4}, 10^{-3}, 10^{-2}, 10^{-1}\}$.\\
\textbf{MDN} The originally proposed model uses an exponential activation to model the standard deviations. However, inspired by \citet{lakshminarayanan2017simple}, we used softplus activation plus a small constant ($10^{-2}$) to avoid numerical instability. We selected the number of components from $\{2, 4, 8, 16 , 32\}$. Predictions are made by taking the mean of the mixture model given by the MDN. \\

We used Scikit-learn \citep{scikit-learn} for the implementations of Linear Regression, and Keras \citep{chollet2015keras} for the neural network models.
All neural network models are trained with mini-batch size 256 using the Adam optimizer \citep{kingma2014adam} with a learning rate 1e-3 and the parameters are initialized according to the method suggested by \citet{lecun1998efficient}. Dropout \citep{srivastava2014dropout} with rate $0.05$ is added to the input layer of all neural networks for CT Scan and Song Year tasks to avoid overfitting. We trained the networks for 1000 epochs on CT Position, 150 epochs on Song Year and 500 epochs on Bike Sharing and Pole. Dropout is disabled in the experiments in Section \ref{opt}. 
\section{Application - Details} \label{sec:appdx_application_details}

\subsection{Time Series Prediction}

In this section, we present the architectural details and hyperparameters used in the time-series prediction experiments.

The linear model extracts the target channel (oil temperature) and applies a single fully connected layer mapping the input sequence length to the prediction length, following the channel-independent design in ~\cite{Zeng2022ltsflinear}.
The LSTM and GRU models each stacks $2$ recurrent layers, each with $256$ hidden units, dropout of $0.1$, and layer normalization. The recurrent block is then followed by a linear projection mapping to the prediction length. The Transformer uses $3$ encoder blocks with model dimension $128$, $8$ attention heads, and a feed-forward dimension of $512$, followed by the same projection layer. For each architecture, the base model generates a feature vector that is then either passed to a regression head or to HL-Gauss head.

\begin{table}
\resizebox{\textwidth}{!}{
\begin{tabular}{llcccc}
\toprule
Architecture & Loss & ETTh1 & ETTh2 & ETTm1 & ETTm2 \\
\midrule
Linear & $\ell_2$ & 0.1988 $\pm$ 0.001382 & 0.3162 $\pm$ 0.002306 & 0.1365 $\pm$ 0.002068 & 0.2063 $\pm$ 0.001591 \\
 & HL-Gauss & 0.5548 $\pm$ 0.003119 & 0.4779 $\pm$ 0.003742 & 0.5303 $\pm$ 0.006077 & 0.4249 $\pm$ 0.003523 \\
\midrule
LSTM & $\ell_2$ & 0.7496 $\pm$ 0.02461 & 1.662 $\pm$ 0.06509 & 0.5056 $\pm$ 0.1031 & 0.6145 $\pm$ 0.08133 \\
 & HL-Gauss & 0.7604 $\pm$ 0.03374 & 1.664 $\pm$ 0.1179 & \textbf{0.4009 $\pm$ 0.06353} & \textbf{0.505 $\pm$ 0.1101} \\
\midrule
GRU & $\ell_2$ & 0.6165 $\pm$ 0.01295 & 0.9432 $\pm$ 0.1581 & 0.5491 $\pm$ 0.02453 & 0.4629 $\pm$ 0.04164 \\
 & HL-Gauss & 0.7948 $\pm$ 0.05213 & \textbf{0.7682 $\pm$ 0.2456} & \textbf{0.452 $\pm$ 0.02874} & \textbf{0.4432 $\pm$ 0.09381} \\
\midrule 
Transformer & $\ell_2$ & 0.6804 $\pm$ 0.01806 & 1.257 $\pm$ 0.1399 & 0.4316 $\pm$ 0.03694 & 0.3675 $\pm$ 0.03441 \\
 & HL-Gauss & 0.7532 $\pm$ 0.08686 & 1.921 $\pm$ 0.1787 & 0.4743 $\pm$ 0.08024 & 0.5415 $\pm$ 0.1324 \\

\bottomrule
\end{tabular}
}
\caption{Test MAE results on the ETD dataset. For each (architecture, loss, dataset) tuple, the table shows the mean averaged over $5$ independent seeds and $95\%$ confidence interval.  The {\bftab bold} indicates better performance for HL-Gauss. HL-Gauss outperforms $\ell_2$ for recurrent architectures (LSTM and GRU) but underperforms $\ell_2$ for linear, and mostly ties on the vanilla transformer.}
\label{tab:ett_results_mae}
\end{table}

We use the standard 12,4,4 month split (30 days per month) for training, validation, and test, following the convention in Informer~\citep{haoyi2021informer} and LTSF-Linear \citep{Zeng2022ltsflinear}. The input sequence length is 96 and the prediction horizon is 96 steps. The prediction task is univariate: all 7 input channels are used as features, and the target is the oil temperature (OT) column. All features are standardized to zero mean and unit variance using statistics computed from the training split. We train all networks with a batch size of 32 using the AdamW optimizer  with a fixed learning rate of $10^{-4}$ and a weight decay
$10^{-4}$. We set the maximum number of training epochs to 10 with early stopping based on the validation loss with patience 3, restoring the best weights.

For the HL-Gauss model, the bin width was determined as in Equation \eqref{eq:hist_bins} where $\sigma_w$ is the ratio of the $\sigma$ parameter for the truncated Gaussian distribution to the bin width, and $\psi_\sigma$ is the ratio of padding to bin width added to each side.  We used $k = 100$ bins with $\sigma_w = 2$ and $\psi_\sigma = 3$. The bin widths are calculated using the minimum and maximum observation over the training dataset.

\subsection{Value Prediction}

We first precomputed the values for each game. Then, we trained on the first 95\% of the game iterations in the actions file and tested on the remaining 5\% of the iterations. The samples were shuffled with a buffer size of 1000 within the training and testing sets. We scaled the targets to be in the range $[0, 1]$. The model backbone consists of 3 convolutional layers followed by a fully-connected layer. We used an ADAM optimizer with parameters $\alpha = 10^{-3}$, $\beta_1 = 0.9$, $\beta_2 = 0.999$, and $\epsilon = 10^{-7}$. All models were implemented using Keras (\cite{chollet2015keras}). 

For HL-Gauss We used $k = 100$ bins with $\sigma_w = 2$ and $\psi_\sigma = 4$ in Equation \eqref{eq:hist_bins}. The models were trained for 1-3 epochs, depending on the game. 
\clearpage
\section{Other Results} \label{sec:appdx_results}

This section provides the rest of the empirical results.

\subsection{Overall Results}

\begin{table*}[!ht]
\vspace{-0.4cm}
\centering
\begin{tabular}{lllbl}                                                                                                  \hline                                                                                                                   Method     & Train MAE                         & Train RMSE                        & Test MAE                             & Test RMSE                            \\                                                                      \hline                                                                                                                   Lin Reg    & $ 6.793${\scriptsize $(\pm 0.003)$} & $ 9.547${\scriptsize $(\pm 0.003)$} & $ 6.796${\scriptsize $(\pm 0.007)$} & $ 9.555${\scriptsize $(\pm 0.014)$} \\                                                                       $\ell_2$ & $ 5.758${\scriptsize $(\pm 0.026)$} & $ 8.137${\scriptsize $(\pm 0.014)$} & $ 6.016${\scriptsize $(\pm 0.022)$} & $ 8.690${\scriptsize $(\pm 0.015)$} \\                                                                       $\ell_2$+Noise  & $ 5.697${\scriptsize $(\pm 0.023)$} & $ 8.124${\scriptsize $(\pm 0.006)$} & $ 5.959${\scriptsize $(\pm 0.015)$} & $ 8.682${\scriptsize $(\pm 0.022)$} \\                                                                       $\ell_2$+Clip   & $ 5.749${\scriptsize $(\pm 0.032)$} & $ 8.130${\scriptsize $(\pm 0.012)$} & $ 6.009${\scriptsize $(\pm 0.032)$} & $ 8.687${\scriptsize $(\pm 0.015)$} \\                                                                       $\ell_1$         & $ 5.423${\scriptsize $(\pm 0.010)$} & $ 8.640${\scriptsize $(\pm 0.016)$} & $ 5.673${\scriptsize $(\pm 0.006)$} & $ 8.987${\scriptsize $(\pm 0.025)$} \\                                                                       MDN        & $ 5.858${\scriptsize$(\pm0.020)$} & $ 8.530${\scriptsize$(\pm0.008)$} & $ 5.935${\scriptsize$(\pm0.025)$} & $ 8.640${\scriptsize$(\pm0.020)$} \\  $\ell_2$+Softmax & $ 5.651${\scriptsize $(\pm 0.015)$} & $ 8.067${\scriptsize $(\pm 0.009)$} & $ 5.947${\scriptsize $(\pm 0.012)$} & $ 8.685${\scriptsize $(\pm 0.012)$} \\
$\ell_2$+Mean & $5.863${\scriptsize $(\pm 0.001)$} & $8.434${\scriptsize $(\pm 0.004)$} & $5.995${\scriptsize $(\pm 0.027)$} & $8.724${\scriptsize $(\pm 0.022)$}\\
HL-OneBin  & $ 5.810${\scriptsize $(\pm 0.010)$} & $ 8.487${\scriptsize $(\pm 0.002)$} & $ 5.906${\scriptsize $(\pm 0.020)$} & $ 8.636${\scriptsize $(\pm 0.020)$} \\                                                                       HL-Gauss.  & $ 5.789${\scriptsize $(\pm 0.004)$} & $ 8.440${\scriptsize $(\pm 0.004)$} & $ 5.903${\scriptsize $(\pm 0.010)$} & $ 8.621${\scriptsize $(\pm 0.016)$} \\                                                                      \hline                                                                                                                  \end{tabular}
\vspace{1pt}

\begin{tabular}{lllbl}                                                                                                  \hline                                                                                                                   Method     & Train MAE                           & Train RMSE                          & Test MAE                               & Test RMSE                              \\                                                              \hline                                                                                                                   Lin Reg    & $ 106.047${\scriptsize $(\pm 0.215)$} & $ 141.854${\scriptsize $(\pm 0.192)$} & $ 105.788${\scriptsize $(\pm 0.738)$} & $ 141.617${\scriptsize $(\pm 0.757)$} \\                                                               $\ell_2$ & $ 14.453${\scriptsize $(\pm 0.258)$}  & $ 19.487${\scriptsize $(\pm 0.116)$}  & $ 31.029${\scriptsize $(\pm 0.258)$}  & $ 48.205${\scriptsize $(\pm 0.545)$}  \\                                                               $\ell_2$+Noise  & $ 20.577${\scriptsize $(\pm 0.270)$}  & $ 28.604${\scriptsize $(\pm 0.325)$}  & $ 28.486${\scriptsize $(\pm 0.278)$}  & $ 44.777${\scriptsize $(\pm 0.747)$}  \\                                                               $\ell_2$+Clip   & $ 13.221${\scriptsize $(\pm 0.676)$}  & $ 18.011${\scriptsize $(\pm 0.783)$}  & $ 29.404${\scriptsize $(\pm 0.188)$}  & $ 46.309${\scriptsize $(\pm 0.403)$}  \\                                                               $\ell_1$         & $ 12.764${\scriptsize $(\pm 0.250)$}  & $ 21.456${\scriptsize $(\pm 0.356)$}  & $ 27.550${\scriptsize $(\pm 0.258)$}  & $ 44.717${\scriptsize $(\pm 0.674)$}  \\                                                               MDN        & $ 15.406${\scriptsize$(\pm0.192)$}  & $ 27.593${\scriptsize$(\pm0.348)$}  & $ 26.268${\scriptsize$(\pm0.342)$}  & $ 43.679${\scriptsize$(\pm0.605)$}  \\  $\ell_2$+Softmax & $ 11.991${\scriptsize $(\pm 0.277)$}  & $ 16.799${\scriptsize $(\pm 0.445)$}  & $ 27.827${\scriptsize $(\pm 0.253)$}  & $ 45.825${\scriptsize $(\pm 0.431)$}  \\
$\ell_2$+Mean & $14.522${\scriptsize $(\pm 0.353)$} & $20.499${\scriptsize $(\pm 0.518)$} & $29.474${\scriptsize $(\pm 0.527)$} & $47.651${\scriptsize $(\pm 0.990)$}\\
HL-OneBin  & $ 15.822${\scriptsize $(\pm 0.198)$}  & $ 26.065${\scriptsize $(\pm 0.335)$}  & $ 26.689${\scriptsize $(\pm 0.280)$}  & $ 45.150${\scriptsize $(\pm 0.857)$}  \\                                                               HL-Gauss.  & $ 14.335${\scriptsize $(\pm 0.152)$}  & $ 23.178${\scriptsize $(\pm 0.309)$}  & $ 25.525${\scriptsize $(\pm 0.331)$}  & $ 43.671${\scriptsize $(\pm 0.796)$}  \\                                                              \hline                                                                                                                  \end{tabular}

\vspace{1pt}

\begin{tabular}{lllbl}                                                                                                  \hline                                                                                                                   Method     & Train MAE                          & Train RMSE                         & Test MAE                              & Test RMSE                             \\                                                                  \hline                                                                                                                   Lin Reg    & $ 26.523${\scriptsize $(\pm 0.018)$} & $ 30.424${\scriptsize $(\pm 0.010)$} & $ 26.662${\scriptsize $(\pm 0.066)$} & $ 30.567${\scriptsize $(\pm 0.038)$} \\                                                                   $\ell_2$ & $ 0.767${\scriptsize $(\pm 0.016)$}  & $ 1.441${\scriptsize $(\pm 0.026)$}  & $ 1.131${\scriptsize $(\pm 0.024)$}  & $ 2.579${\scriptsize $(\pm 0.073)$}  \\                                                                   $\ell_2$+Noise  & $ 0.700${\scriptsize $(\pm 0.031)$}  & $ 1.406${\scriptsize $(\pm 0.039)$}  & $ 1.062${\scriptsize $(\pm 0.027)$}  & $ 2.529${\scriptsize $(\pm 0.054)$}  \\                                                                   $\ell_2$+Clip   & $ 0.714${\scriptsize $(\pm 0.018)$}  & $ 1.424${\scriptsize $(\pm 0.024)$}  & $ 1.069${\scriptsize $(\pm 0.017)$}  & $ 2.551${\scriptsize $(\pm 0.103)$}  \\                                                                   $\ell_1$         & $ 0.637${\scriptsize $(\pm 0.013)$}  & $ 1.661${\scriptsize $(\pm 0.029)$}  & $ 0.938${\scriptsize $(\pm 0.029)$}  & $ 2.584${\scriptsize $(\pm 0.034)$}  \\                                                                   MDN        & $ 0.648${\scriptsize$(\pm0.026)$}  & $ 1.712${\scriptsize$(\pm0.078)$}  & $ 0.895${\scriptsize$(\pm0.013)$}  & $ 2.480${\scriptsize$(\pm0.095)$}  \\ $\ell_2$+Softmax & $ 0.634${\scriptsize $(\pm 0.017)$}  & $ 1.445${\scriptsize $(\pm 0.019)$}  & $ 0.950${\scriptsize $(\pm 0.016)$}  & $ 2.481${\scriptsize $(\pm 0.055)$}  \\
$\ell_2$+Mean & $0.613${\scriptsize $(\pm 0.008)$} & $1.452${\scriptsize $(\pm 0.010)$} & $0.986${\scriptsize $(\pm 0.010)$} & $2.773${\scriptsize $(\pm 0.056)$}\\
HL-OneBin  & $ 0.899${\scriptsize $(\pm 0.015)$}  & $ 1.442${\scriptsize $(\pm 0.033)$}  & $ 1.264${\scriptsize $(\pm 0.019)$}  & $ 2.760${\scriptsize $(\pm 0.116)$}  \\                                                                   HL-Gauss.  & $ 0.347${\scriptsize $(\pm 0.018)$}  & $ 1.299${\scriptsize $(\pm 0.049)$}  & $ 0.714${\scriptsize $(\pm 0.024)$}  & $ 2.673${\scriptsize $(\pm 0.141)$}  \\                                                                  \hline                                                                                                                  \end{tabular}

\caption{Overall results on the other data sets. The tables show Song Year, Bike Sharing, and Pole from top to bottom. \textbf{Song Year:} $\ell_1$ outperformed HL-Gauss in terms of Test MAE, there was not a substantial difference between the performance of HL-OneBin and HL-Gauss, and the rest of the methods had higher error rates than HL-Gauss. The differences in error rates were small on this data set and even Linear Regression worked well. \textbf{Bike Sharing:} HL-Gauss had the lowest Test MAE on this data set and $\ell_2$ performed the worst among neural network methods. \textbf{Pole:} HL-Gauss achieved the lowest Test MAE. The other methods did not achieve a Test MAE close to that of HL-Gauss, and there was a noticeable gap between the performance of HL-OneBin and HL-Gauss.}
\label{tab:overall_yearpred}
\end{table*}

\begin{figure*}[!ht]
\centering
 \begin{subfigure}[b]{\figwidthtwo}
         \includegraphics[width=\textwidth]{figures/uniform,ctscan,mae.png}
\end{subfigure}
  \begin{subfigure}[b]{\figwidthtwo}
         \includegraphics[width=\textwidth]{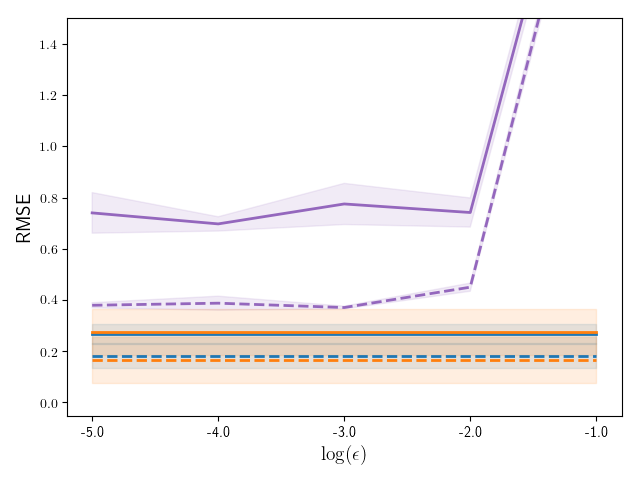}
\end{subfigure}
 
 \begin{subfigure}[b]{\figwidthtwo}
         \includegraphics[width=\textwidth]{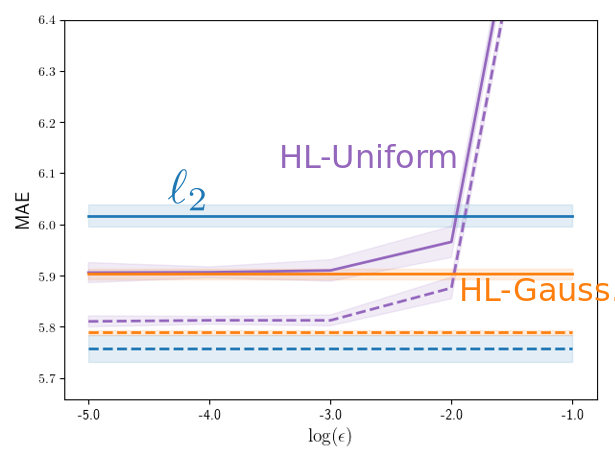}
\end{subfigure}
  \begin{subfigure}[b]{\figwidthtwo}
         \includegraphics[width=\textwidth]{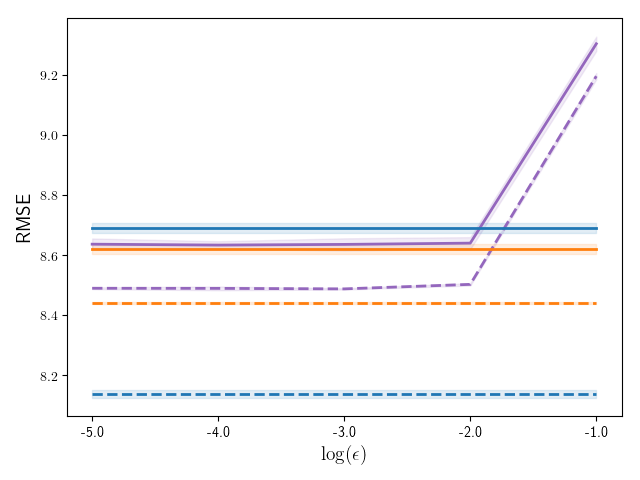}
\end{subfigure}
 
 \begin{subfigure}[b]{\figwidthtwo}
         \includegraphics[width=\textwidth]{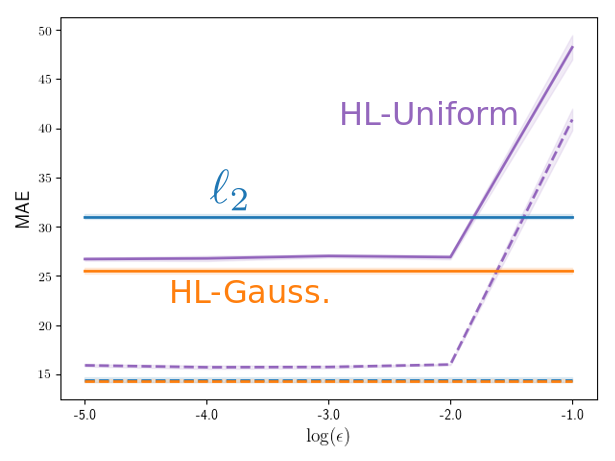}
\end{subfigure}
  \begin{subfigure}[b]{\figwidthtwo}
         \includegraphics[width=\textwidth]{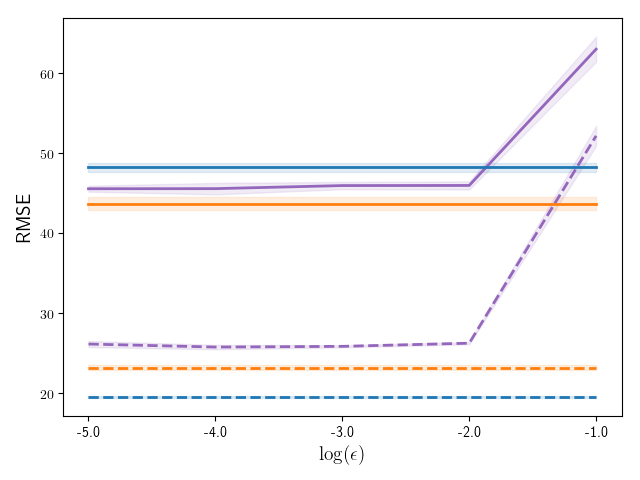}
\end{subfigure}
 
 \begin{subfigure}[b]{\figwidthtwo}
         \includegraphics[width=\textwidth]{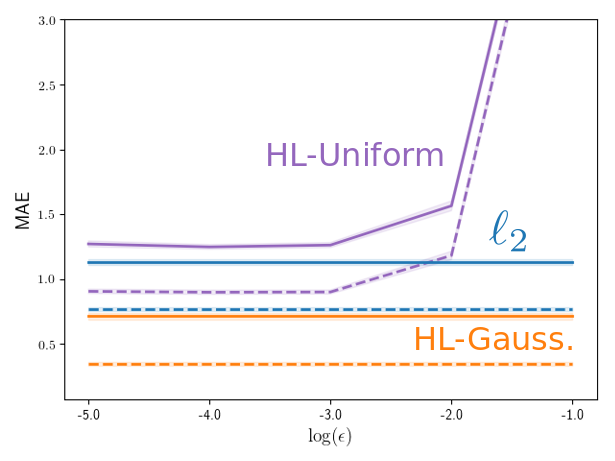}
\end{subfigure}
 \begin{subfigure}[b]{\figwidthtwo}
         \includegraphics[width=\textwidth]{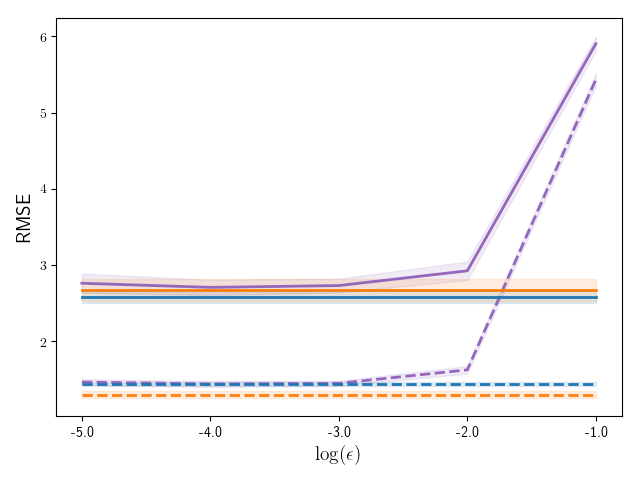}
\end{subfigure}
 
\caption{HL-Uniform results. The rows show CT Scan, Song Year, Bike, and Pole respectively. Dotted and solid lines show train and test errors respectively. The parameter $\epsilon$ is the weighting on the uniform distribution and raising it only impaired performance.}
\label{fig:uniform_ctscan_appdx}
\end{figure*}

\clearpage
\subsection{Bias and the Choice of Target Distribution}

\begin{table}[!ht]
    \centering
    \begin{tabular}{lllbl}                                                                                                  \hline                                                                                                                   Method       & Train MAE                         & Train RMSE                        & Test MAE                             & Test RMSE                            \\                                                                    \hline                                                                                                                   HL-OneBin    & $ 0.309${\scriptsize $(\pm 0.000)$} & $ 0.364${\scriptsize $(\pm 0.006)$} & $ 0.335${\scriptsize $(\pm 0.004)$} & $ 0.660${\scriptsize $(\pm 0.099)$} \\                                                                     HL-Projected & $ 0.053${\scriptsize $(\pm 0.001)$} & $ 0.074${\scriptsize $(\pm 0.002)$} & $ 0.103${\scriptsize $(\pm 0.003)$} & $ 0.462${\scriptsize $(\pm 0.065)$} \\                                                       HL-Gauss.    & $ 0.061${\scriptsize $(\pm 0.006)$} & $ 0.164${\scriptsize $(\pm 0.090)$} & $ 0.098${\scriptsize $(\pm 0.005)$} & $ 0.274${\scriptsize $(\pm 0.090)$} \\                                                                                  \hline                                                                                                                  \end{tabular}
    \vspace{1pt}
    \begin{tabular}{lllbl}                                                                                                  \hline                                                                                                                   Method       & Train MAE                         & Train RMSE                        & Test MAE                             & Test RMSE                            \\                                                                    \hline                                                                                                                   HL-OneBin    & $ 5.810${\scriptsize $(\pm 0.010)$} & $ 8.487${\scriptsize $(\pm 0.002)$} & $ 5.906${\scriptsize $(\pm 0.020)$} & $ 8.636${\scriptsize $(\pm 0.020)$} \\                                                                     HL-Projected & $ 5.815${\scriptsize $(\pm 0.012)$} & $ 8.477${\scriptsize $(\pm 0.003)$} & $ 5.917${\scriptsize $(\pm 0.019)$} & $ 8.629${\scriptsize $(\pm 0.016)$} \\                                                                     HL-Gauss.    & $ 5.789${\scriptsize $(\pm 0.004)$} & $ 8.440${\scriptsize $(\pm 0.004)$} & $ 5.903${\scriptsize $(\pm 0.010)$} & $ 8.621${\scriptsize $(\pm 0.016)$} \\                                                                    \hline                                                                                                                  \end{tabular}
    \vspace{1pt}
    \begin{tabular}{lllbl}                                                                                                  \hline                                                                                                                   Method       & Train MAE                          & Train RMSE                         & Test MAE                              & Test RMSE                             \\                                                                \hline                                                                                                                   HL-OneBin    & $ 15.822${\scriptsize $(\pm 0.198)$} & $ 26.065${\scriptsize $(\pm 0.335)$} & $ 26.689${\scriptsize $(\pm 0.280)$} & $ 45.150${\scriptsize $(\pm 0.857)$} \\                                                                 HL-Projected & $ 14.885${\scriptsize $(\pm 0.169)$} & $ 24.329${\scriptsize $(\pm 0.329)$} & $ 26.180${\scriptsize $(\pm 0.348)$} & $ 44.982${\scriptsize $(\pm 0.845)$} \\
    HL-Gauss.    & $ 14.335${\scriptsize $(\pm 0.152)$} & $ 23.178${\scriptsize $(\pm 0.309)$} & $ 25.525${\scriptsize $(\pm 0.331)$} & $ 43.671${\scriptsize $(\pm 0.796)$} \\                                                                 \hline                                                                                                                  \end{tabular}
    \vspace{1pt}
    \begin{tabular}{lllbl}                                                                                                  \hline                                                                                                                   Method       & Train MAE                         & Train RMSE                        & Test MAE                             & Test RMSE                            \\                                                                    \hline                                                                                                                   HL-OneBin    & $ 0.899${\scriptsize $(\pm 0.015)$} & $ 1.442${\scriptsize $(\pm 0.033)$} & $ 1.264${\scriptsize $(\pm 0.019)$} & $ 2.760${\scriptsize $(\pm 0.116)$} \\                                                                     HL-Projected & $ 0.364${\scriptsize $(\pm 0.020)$} & $ 1.340${\scriptsize $(\pm 0.055)$} & $ 0.741${\scriptsize $(\pm 0.018)$} & $ 2.753${\scriptsize $(\pm 0.068)$} \\
   HL-Gauss.    & $ 0.347${\scriptsize $(\pm 0.018)$} & $ 1.299${\scriptsize $(\pm 0.049)$} & $ 0.714${\scriptsize $(\pm 0.024)$} & $ 2.673${\scriptsize $(\pm 0.141)$} \\                                                                                                                                         \hline                                                                                                                  \end{tabular}
   
    \caption{Discretization bias experiment. The tables show CT-Scan, Song Year, Bike Sharing, and Pole from top to bottom. Overall, HL-Projected achieved a Test MAE close to that of HL-Gauss, and performed noticeably better than HL-OneBin.}
    \label{tab:proj_ctscan_appdx}

\end{table}

\begin{figure}[!ht]
    \centering
\begin{minipage}{\figwidththree}
        \centering
        \includegraphics[width=\textwidth]{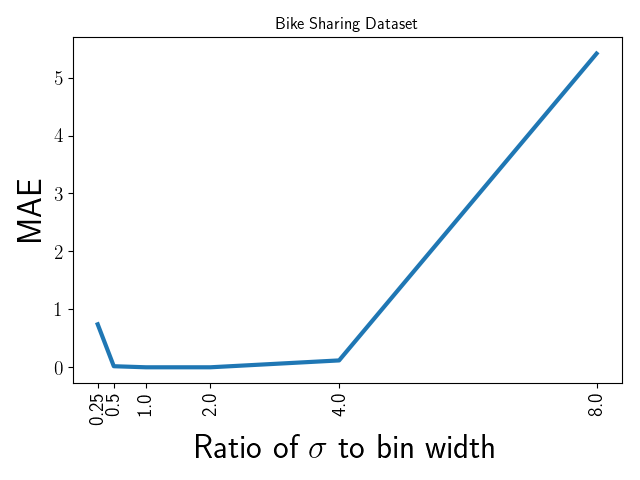}
    \end{minipage}
    \begin{minipage}{\figwidththree}
       	\centering
       	\includegraphics[width=\textwidth]{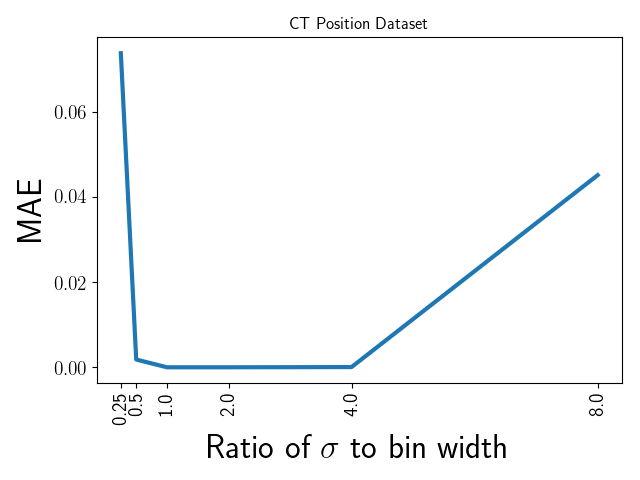}
    \end{minipage}
    \begin{minipage}{\figwidththree}
       	\centering
       	\includegraphics[width=\textwidth]{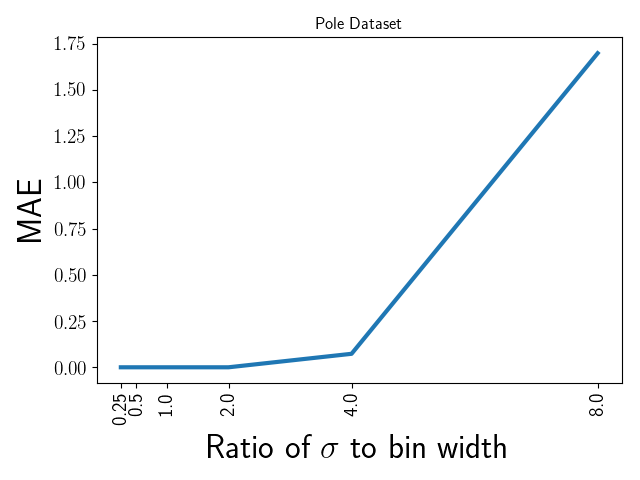}
    \end{minipage}
    \caption{MAE between the means of HL targets and the original labels. It can be seen that extreme values of $\sigma$ on either side biased the mean of the target distributions.}
\label{fig:label_bias_others}
\end{figure}

\begin{figure*}[!ht]
\centering
 \begin{subfigure}[b]{\figwidthtwo}
         \includegraphics[width=\textwidth]{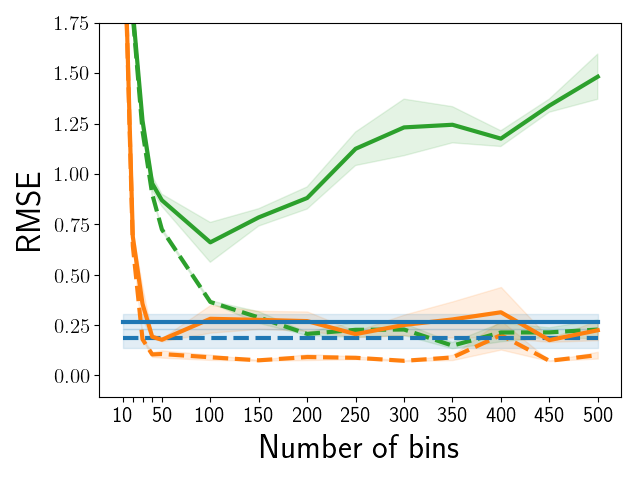}
\end{subfigure}
 \begin{subfigure}[b]{\figwidthtwo}
         \includegraphics[width=\textwidth]{figures/bias_var,ctscan,mae.png}
\end{subfigure}

  \begin{subfigure}[b]{\figwidthtwo}
         \includegraphics[width=\textwidth]{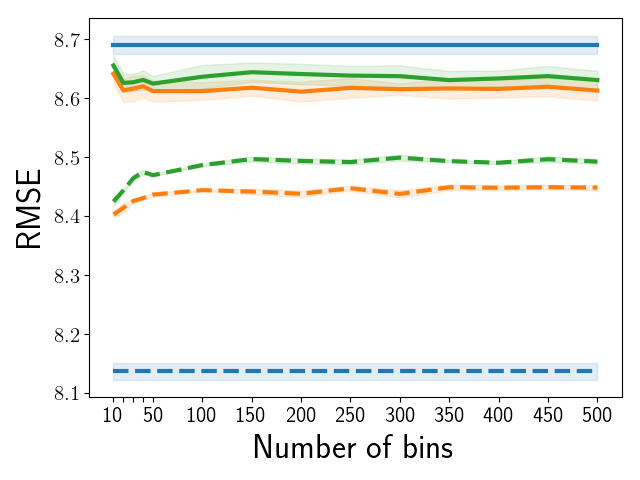}
\end{subfigure}
 \begin{subfigure}[b]{\figwidthtwo}
         \includegraphics[width=\textwidth]{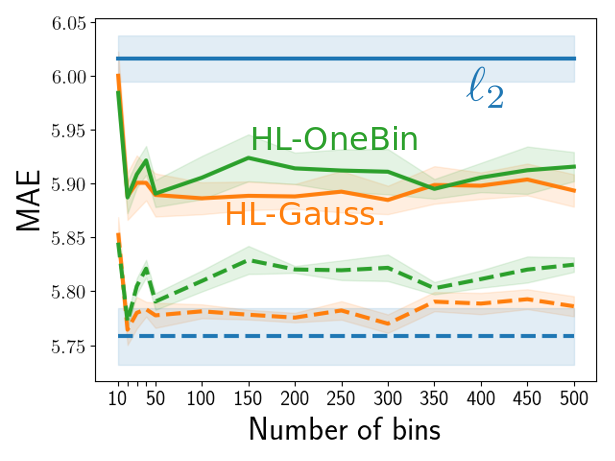}
\end{subfigure}

  \begin{subfigure}[b]{\figwidthtwo}
         \includegraphics[width=\textwidth]{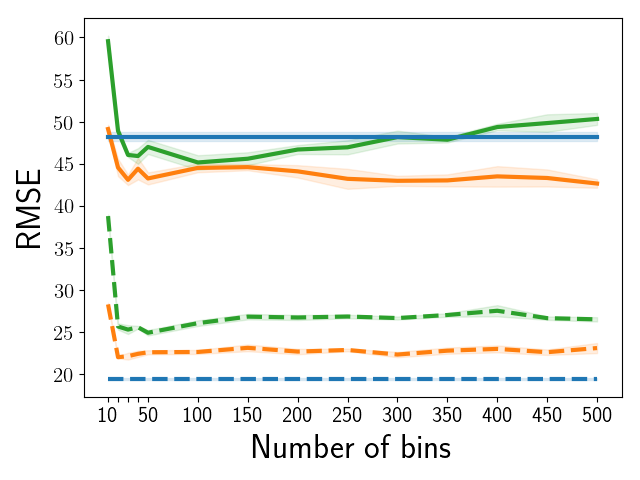}
\end{subfigure}
 \begin{subfigure}[b]{\figwidthtwo}
         \includegraphics[width=\textwidth]{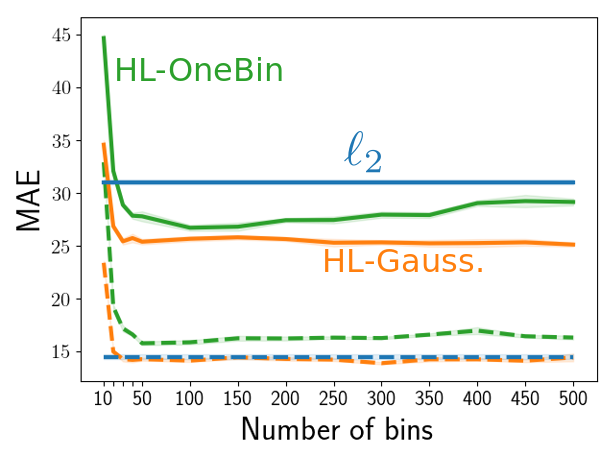}
\end{subfigure} 

  \begin{subfigure}[b]{\figwidthtwo}
         \includegraphics[width=\textwidth]{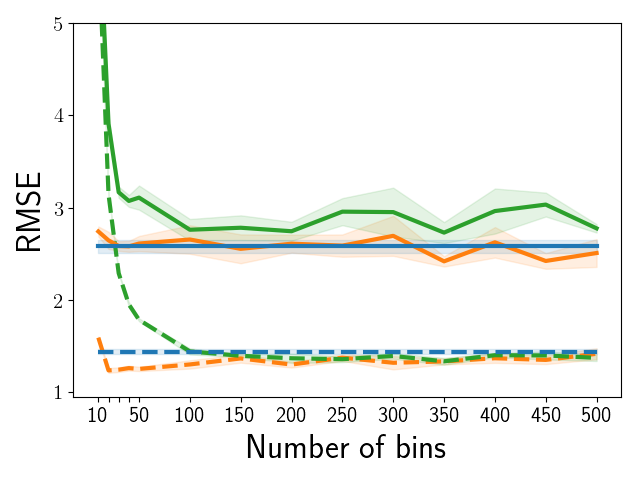}
\end{subfigure}
 \begin{subfigure}[b]{\figwidthtwo}
         \includegraphics[width=\textwidth]{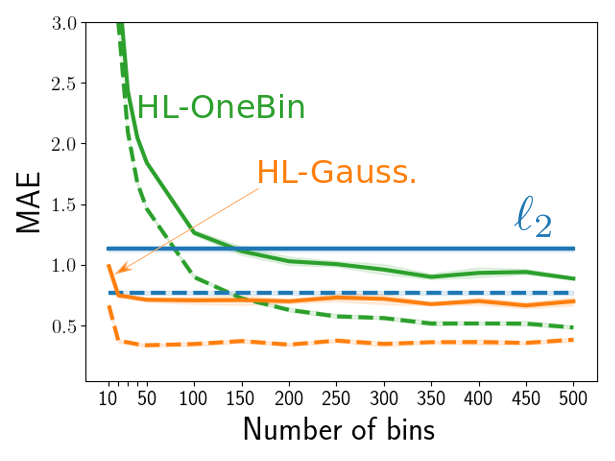}
\end{subfigure} 
\caption{Changing the number of bins. The rows show CT Scan, Song Year, Bike Sharing, and Pole respectively. Dotted and solid lines show train and test errors respectively. A small number of bins resulted in high train and test errors, indicating high bias. A higher number of bins generally did not result in a rise in test error, with the exception of HL-OneBin's test RMSE on CT Scan.}
\label{fig:bias_var_ctscan_appdx}
\end{figure*}

\begin{figure*}[!ht]
\centering
 \begin{subfigure}[b]{\figwidthtwo}
         \includegraphics[width=\textwidth]{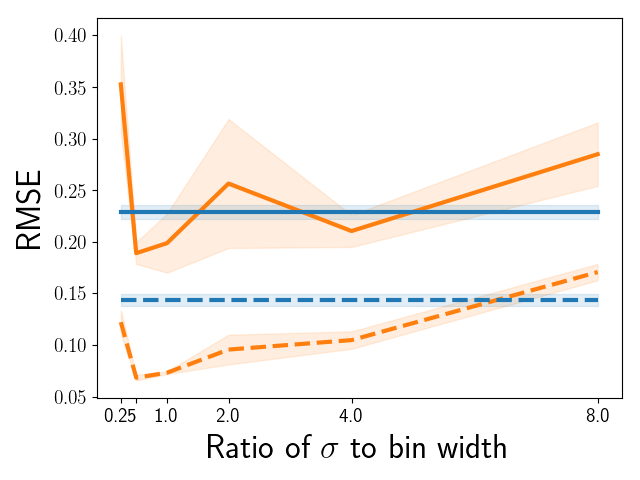}
\end{subfigure}
 \begin{subfigure}[b]{\figwidthtwo}
         \includegraphics[width=\textwidth]{figures/sigma,ctscan,mae.png}
\end{subfigure}

 \begin{subfigure}[b]{\figwidthtwo}
         \includegraphics[width=\textwidth]{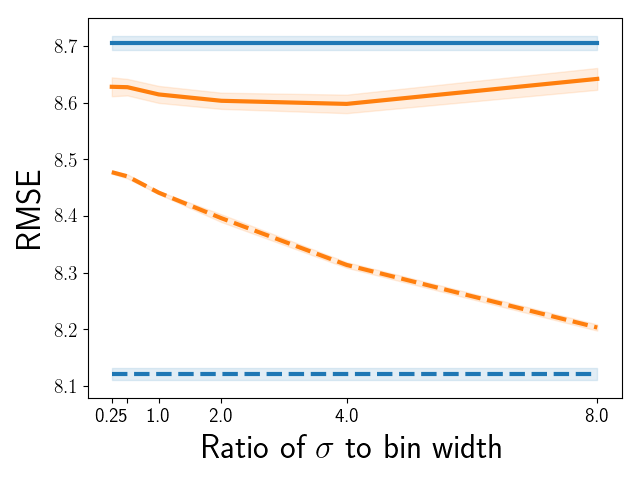}
\end{subfigure}
 \begin{subfigure}[b]{\figwidthtwo}
         \includegraphics[width=\textwidth]{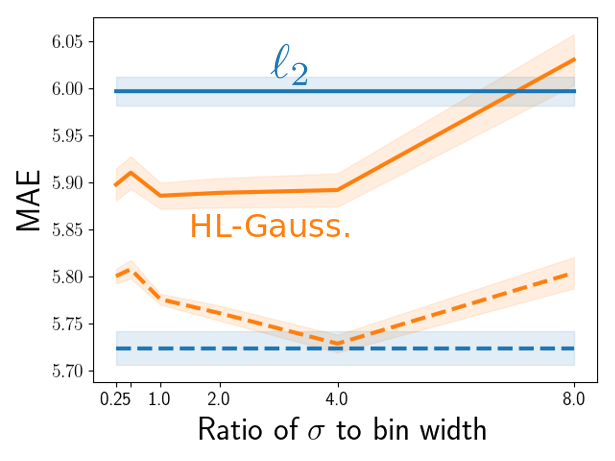}
\end{subfigure}

 \begin{subfigure}[b]{\figwidthtwo}
         \includegraphics[width=\textwidth]{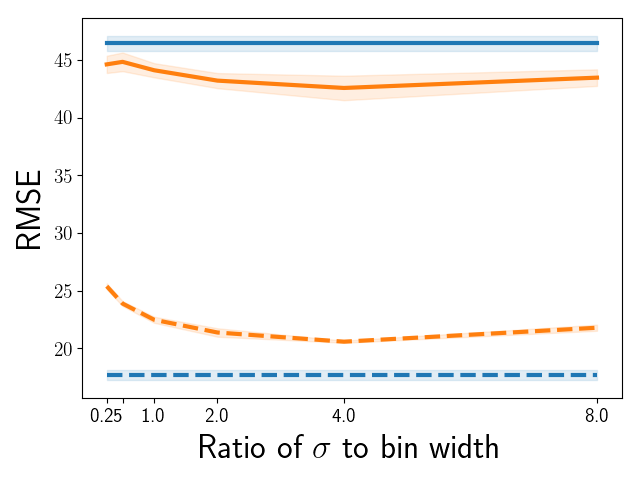}
\end{subfigure}
 \begin{subfigure}[b]{\figwidthtwo}
         \includegraphics[width=\textwidth]{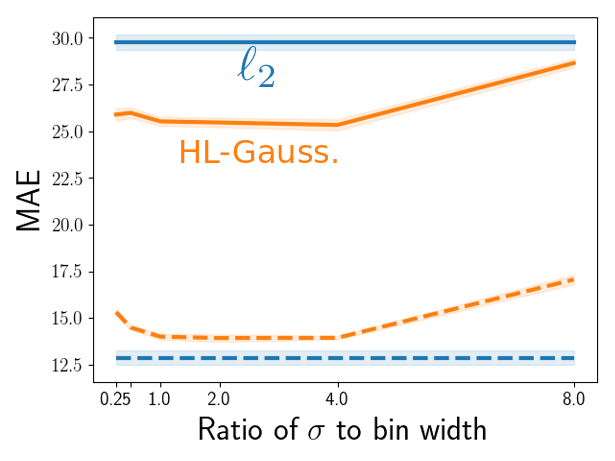}
\end{subfigure} 

 \begin{subfigure}[b]{\figwidthtwo}
         \includegraphics[width=\textwidth]{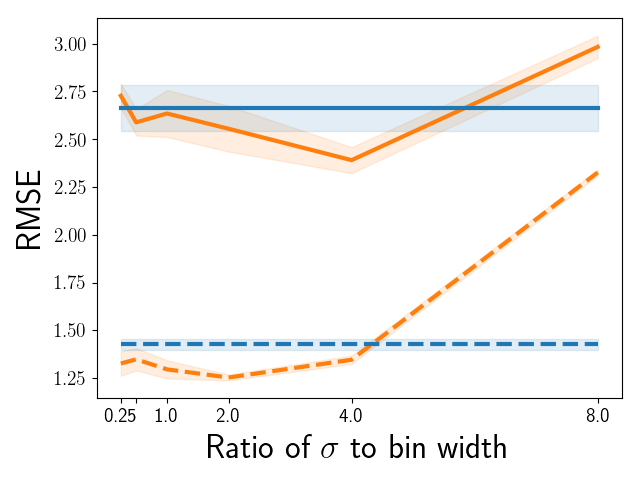}
\end{subfigure}
 \begin{subfigure}[b]{\figwidthtwo}
         \includegraphics[width=\textwidth]{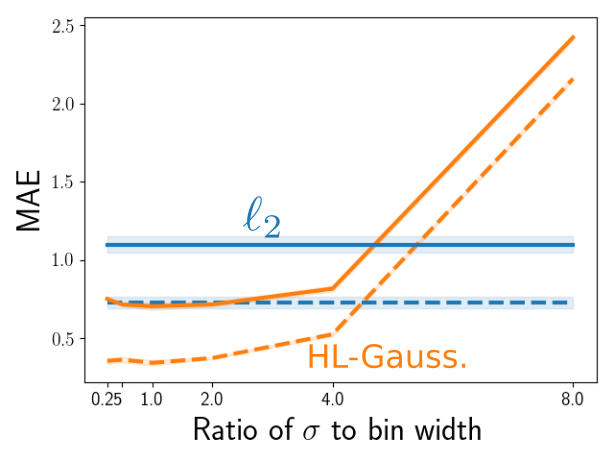}
\end{subfigure}
\caption{Changing the parameter $\sigma$. The rows show CT Scan, Song Year, Bike Sharing, and Pole respectively. Dotted and solid lines show train and test errors respectively. Generally, extreme values of $\sigma$, especially to the right, resulted in bad performance. The error rates for train and test followed each other when changing this parameter.}
\label{fig:sigma_ctscan_appdx}
\end{figure*}

\clearpage
\subsection{Representation}

\begin{table*}[!ht]
\centering
\begin{tabular}{llllll}                                                                                                 \hline                                                                                                                              & Loss               & Default                              & Fixed                                & Initialized                          & Random                                \\                                                \hline                                                                                                                   Train MAE  & $\ell_2$         & $ 5.758${\scriptsize $(\pm 0.026)$} & $ 6.268${\scriptsize $(\pm 0.055)$} & $ 5.743${\scriptsize $(\pm 0.015)$} & $ 8.060${\scriptsize $(\pm 0.024)$}  \\                                                 Train MAE  & HL-Gauss. & $ 5.789${\scriptsize $(\pm 0.004)$} & $ 5.682${\scriptsize $(\pm 0.014)$} & $ 5.781${\scriptsize $(\pm 0.010)$} & $ 7.942${\scriptsize $(\pm 0.021)$}  \\                                                 Train RMSE & $\ell_2$         & $ 8.137${\scriptsize $(\pm 0.014)$} & $ 8.797${\scriptsize $(\pm 0.030)$} & $ 8.131${\scriptsize $(\pm 0.003)$} & $ 10.730${\scriptsize $(\pm 0.016)$} \\                                                 Train RMSE & HL-Gauss. & $ 8.440${\scriptsize $(\pm 0.004)$} & $ 8.130${\scriptsize $(\pm 0.010)$} & $ 8.396${\scriptsize $(\pm 0.007)$} & $ 10.691${\scriptsize $(\pm 0.026)$} \\                                                 \rowcolor{aliceblue} Test MAE   & $\ell_2$         & $ 6.016${\scriptsize $(\pm 0.022)$} & $ 6.338${\scriptsize $(\pm 0.063)$} & $ 5.997${\scriptsize $(\pm 0.028)$} & $ 8.069${\scriptsize $(\pm 0.029)$}  \\                             \rowcolor{aliceblue}                    Test MAE   & HL-Gauss. & $ 5.903${\scriptsize $(\pm 0.010)$} & $ 5.954${\scriptsize $(\pm 0.011)$} & $ 5.900${\scriptsize $(\pm 0.016)$} & $ 7.952${\scriptsize $(\pm 0.017)$}  \\                                                 Test RMSE  & $\ell_2$         & $ 8.690${\scriptsize $(\pm 0.015)$} & $ 8.901${\scriptsize $(\pm 0.046)$} & $ 8.666${\scriptsize $(\pm 0.021)$} & $ 10.748${\scriptsize $(\pm 0.014)$} \\                                                 Test RMSE  & HL-Gauss. & $ 8.621${\scriptsize $(\pm 0.016)$} & $ 8.723${\scriptsize $(\pm 0.013)$} & $ 8.603${\scriptsize $(\pm 0.013)$} & $ 10.712${\scriptsize $(\pm 0.010)$} \\                                                \hline                                                                                                                  \end{tabular}

\vspace{1pt}

\begin{tabular}{llllll}\hline                                                                                                                              & Loss               & Default                               & Fixed                                 & Initialized                           & Random                                 \\                                            \hline                                                                                                                   Train MAE  & $\ell_2$         & $ 14.453${\scriptsize $(\pm 0.258)$} & $ 39.101${\scriptsize $(\pm 1.054)$} & $ 18.214${\scriptsize $(\pm 1.154)$} & $ 106.376${\scriptsize $(\pm 0.569)$} \\                                             Train MAE  & HL-Gauss. & $ 14.335${\scriptsize $(\pm 0.152)$} & $ 30.171${\scriptsize $(\pm 0.469)$} & $ 12.571${\scriptsize $(\pm 0.250)$} & $ 99.251${\scriptsize $(\pm 1.069)$}  \\                                             Train RMSE & $\ell_2$         & $ 19.487${\scriptsize $(\pm 0.116)$} & $ 53.545${\scriptsize $(\pm 1.264)$} & $ 23.833${\scriptsize $(\pm 1.137)$} & $ 142.523${\scriptsize $(\pm 1.052)$} \\                                             Train RMSE & HL-Gauss. & $ 23.178${\scriptsize $(\pm 0.309)$} & $ 42.140${\scriptsize $(\pm 0.664)$} & $ 20.320${\scriptsize $(\pm 0.518)$} & $ 136.065${\scriptsize $(\pm 1.283)$} \\                                             \rowcolor{aliceblue} Test MAE   & $\ell_2$         & $ 31.029${\scriptsize $(\pm 0.258)$} & $ 42.834${\scriptsize $(\pm 0.983)$} & $ 31.119${\scriptsize $(\pm 0.906)$} & $ 106.429${\scriptsize $(\pm 0.982)$} \\                        \rowcolor{aliceblue}                     Test MAE   & HL-Gauss. & $ 25.525${\scriptsize $(\pm 0.331)$} & $ 37.696${\scriptsize $(\pm 0.360)$} & $ 25.470${\scriptsize $(\pm 0.202)$} & $ 99.311${\scriptsize $(\pm 1.486)$}  \\                                             Test RMSE  & $\ell_2$         & $ 48.205${\scriptsize $(\pm 0.545)$} & $ 59.890${\scriptsize $(\pm 0.674)$} & $ 46.735${\scriptsize $(\pm 0.387)$} & $ 143.052${\scriptsize $(\pm 1.472)$} \\                                             Test RMSE  & HL-Gauss. & $ 43.671${\scriptsize $(\pm 0.796)$} & $ 56.350${\scriptsize $(\pm 0.525)$} & $ 43.809${\scriptsize $(\pm 0.593)$} & $ 136.941${\scriptsize $(\pm 1.639)$} \\                                            \hline                                                                                                                  \end{tabular}

\vspace{1pt}

\begin{tabular}{llllll}                                                                                                 \hline                                                                                                                              & Loss               & Default                              & Fixed                                & Initialized                          & Random                                \\                                                \hline                                                                                                                   Train MAE  & $\ell_2$         & $ 0.767${\scriptsize $(\pm 0.016)$} & $ 4.218${\scriptsize $(\pm 0.325)$} & $ 1.109${\scriptsize $(\pm 0.024)$} & $ 28.936${\scriptsize $(\pm 0.176)$} \\                                                 Train MAE  & HL-Gauss. & $ 0.347${\scriptsize $(\pm 0.018)$} & $ 1.218${\scriptsize $(\pm 0.052)$} & $ 0.312${\scriptsize $(\pm 0.010)$} & $ 18.428${\scriptsize $(\pm 0.969)$} \\                                                 Train RMSE & $\ell_2$         & $ 1.441${\scriptsize $(\pm 0.026)$} & $ 6.184${\scriptsize $(\pm 0.412)$} & $ 1.836${\scriptsize $(\pm 0.042)$} & $ 33.642${\scriptsize $(\pm 0.231)$} \\                                                 Train RMSE & HL-Gauss. & $ 1.299${\scriptsize $(\pm 0.049)$} & $ 2.899${\scriptsize $(\pm 0.143)$} & $ 1.212${\scriptsize $(\pm 0.027)$} & $ 26.795${\scriptsize $(\pm 0.959)$} \\                                                 \rowcolor{aliceblue} Test MAE   & $\ell_2$         & $ 1.131${\scriptsize $(\pm 0.024)$} & $ 4.396${\scriptsize $(\pm 0.329)$} & $ 1.437${\scriptsize $(\pm 0.053)$} & $ 29.221${\scriptsize $(\pm 0.254)$} \\                                     \rowcolor{aliceblue}            Test MAE   & HL-Gauss. & $ 0.714${\scriptsize $(\pm 0.024)$} & $ 1.398${\scriptsize $(\pm 0.021)$} & $ 0.645${\scriptsize $(\pm 0.009)$} & $ 18.285${\scriptsize $(\pm 1.062)$} \\                                                 Test RMSE  & $\ell_2$         & $ 2.579${\scriptsize $(\pm 0.073)$} & $ 6.423${\scriptsize $(\pm 0.415)$} & $ 2.882${\scriptsize $(\pm 0.109)$} & $ 33.944${\scriptsize $(\pm 0.299)$} \\                                                 Test RMSE  & HL-Gauss. & $ 2.673${\scriptsize $(\pm 0.141)$} & $ 3.441${\scriptsize $(\pm 0.093)$} & $ 2.373${\scriptsize $(\pm 0.055)$} & $ 26.621${\scriptsize $(\pm 1.043)$} \\                                                \hline                                                                                                                  \end{tabular}

\caption{Representation results on the other data sets. The tables show Song Year, Bike Sharing, and Pole from top to bottom. We tested (a) swapping the representations and re-learning on the last layer (\textbf{Fixed}), (b) initializing with the other's representation (\textbf{Initialized}), (c) and using the same fixed random representation for both (\textbf{Random}) and only learning the last layer. In general, $\ell_2$ underperformed HL-Gauss in all the three settings. Using HL-Gauss's fixed representation only made the performance of $\ell_2$ worse. As an initialization, it did not result in a considerable improvement.}
\label{tab:representation_yearpred}
\end{table*}

\begin{figure*}[!ht]
 \vspace{-0.1cm}
\centering
 \begin{subfigure}[b]{\figwidthtwo}
         \includegraphics[width=\textwidth]{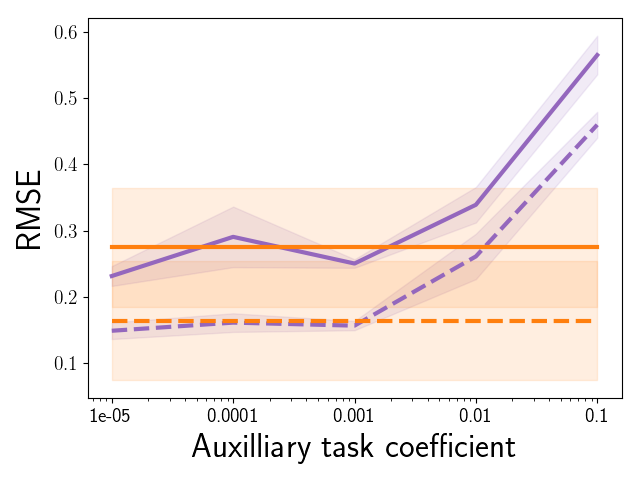}
\end{subfigure}
 \begin{subfigure}[b]{\figwidthtwo}
         \includegraphics[width=\textwidth]{figures/polyc,ctscan,mae.png}
\end{subfigure}

 \begin{subfigure}[b]{\figwidthtwo}
         \includegraphics[width=\textwidth]{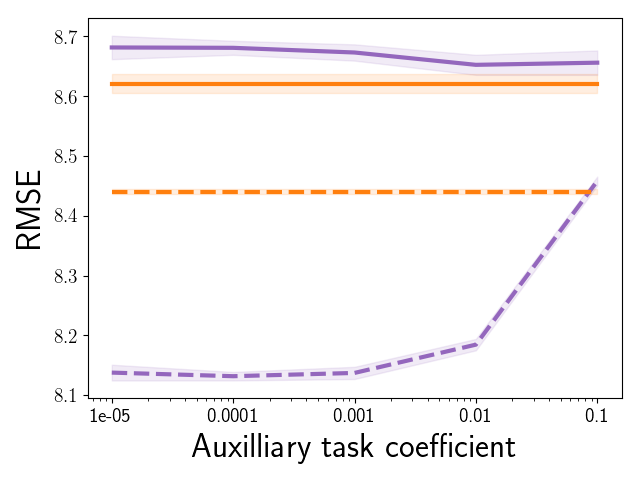}
\end{subfigure}
 \begin{subfigure}[b]{\figwidthtwo}
         \includegraphics[width=\textwidth]{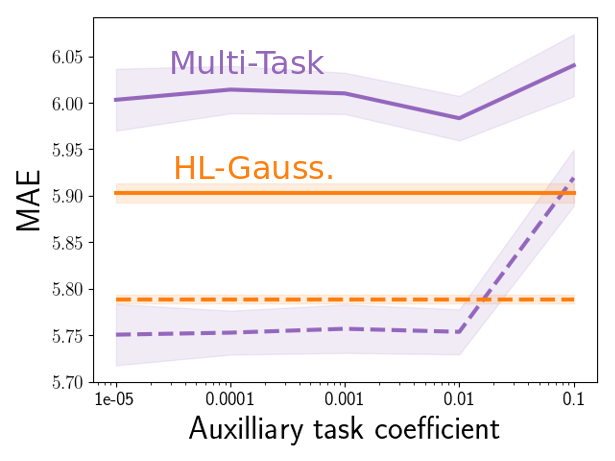}
\end{subfigure} 

 \begin{subfigure}[b]{\figwidthtwo}
         \includegraphics[width=\textwidth]{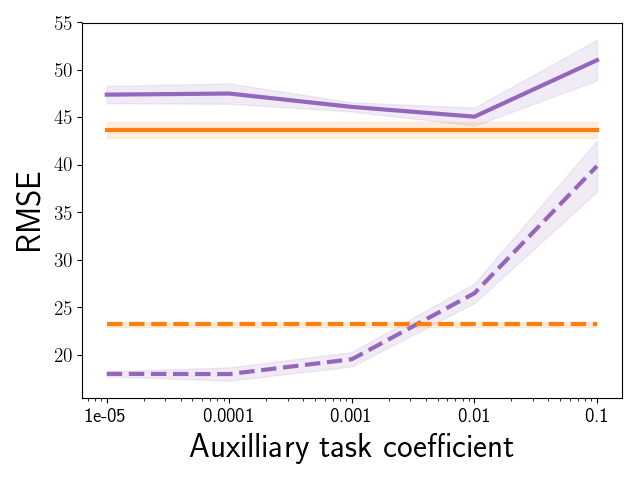}
\end{subfigure}
 \begin{subfigure}[b]{\figwidthtwo}
         \includegraphics[width=\textwidth]{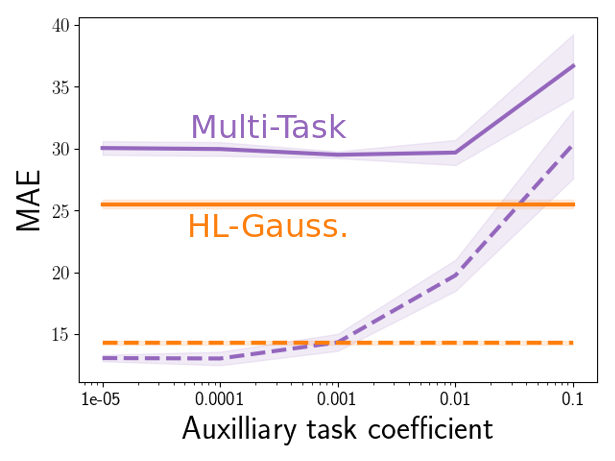}
\end{subfigure} 

 \begin{subfigure}[b]{\figwidthtwo}
         \includegraphics[width=\textwidth]{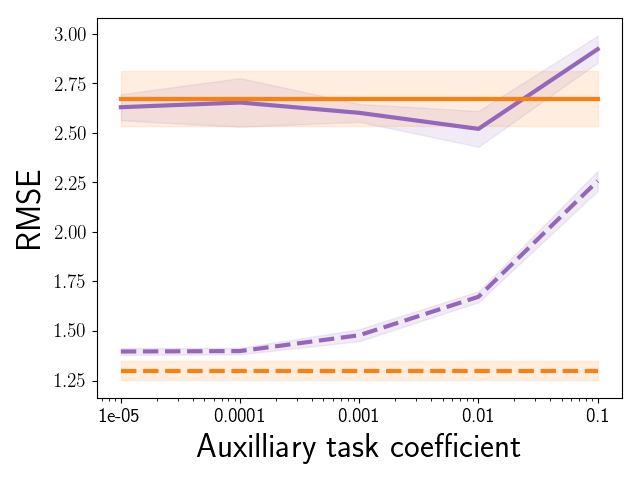}
\end{subfigure}
 \begin{subfigure}[b]{\figwidthtwo}
         \includegraphics[width=\textwidth]{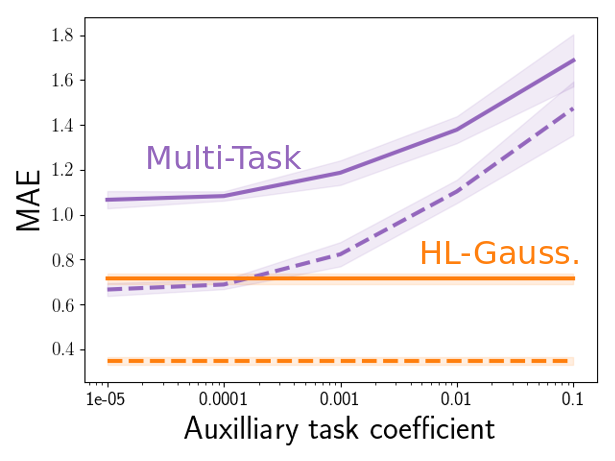}
\end{subfigure} 
\caption{Multi-Task Network results. The rows show CT Scan, Song Year, Bike Sharing, and Pole respectively. The loss function is the mean prediction's squared error plus the distribution prediction's KL-divergence multiplied by a coefficient. The horizontal axis shows the coefficient for KL-divergence. Dotted and solid lines show train and test errors respectively. There was not a substantial drop in the error rate when increasing the coefficient in the loss and the performance only became worse.}
\label{fig:multitask_ctscan_appdx}
\end{figure*}

\begin{figure*}[!ht]
\centering
 \begin{subfigure}[b]{\figwidthtwo}
         \includegraphics[width=\textwidth]{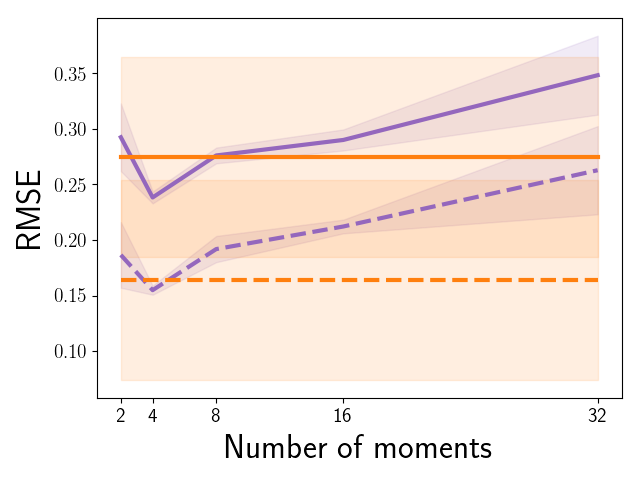}
\end{subfigure}
 \begin{subfigure}[b]{\figwidthtwo}
         \includegraphics[width=\textwidth]{figures/moments,ctscan,mae.png}
\end{subfigure}

  \begin{subfigure}[b]{\figwidthtwo}
         \includegraphics[width=\textwidth]{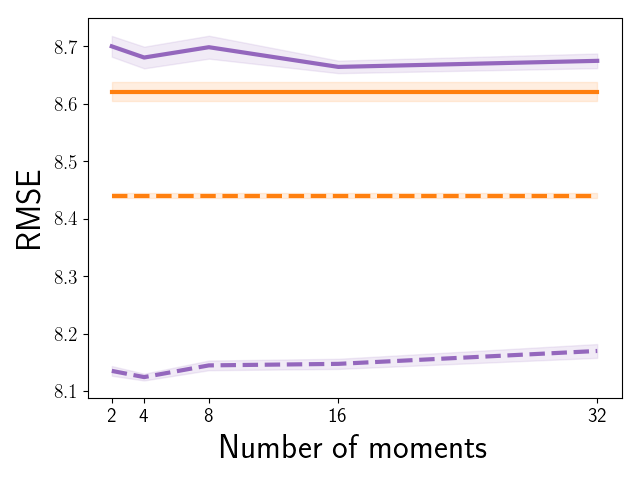}
\end{subfigure}
 \begin{subfigure}[b]{\figwidthtwo}
         \includegraphics[width=\textwidth]{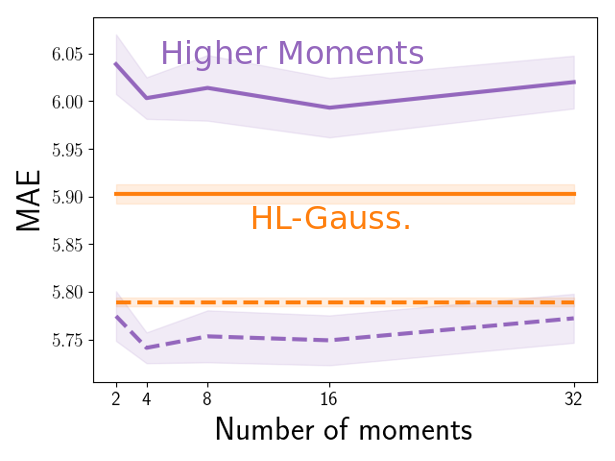}
\end{subfigure} 

  \begin{subfigure}[b]{\figwidthtwo}
         \includegraphics[width=\textwidth]{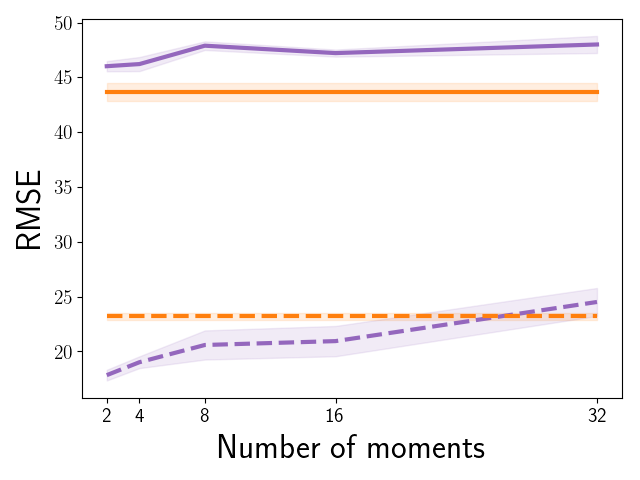}
\end{subfigure}
 \begin{subfigure}[b]{\figwidthtwo}
         \includegraphics[width=\textwidth]{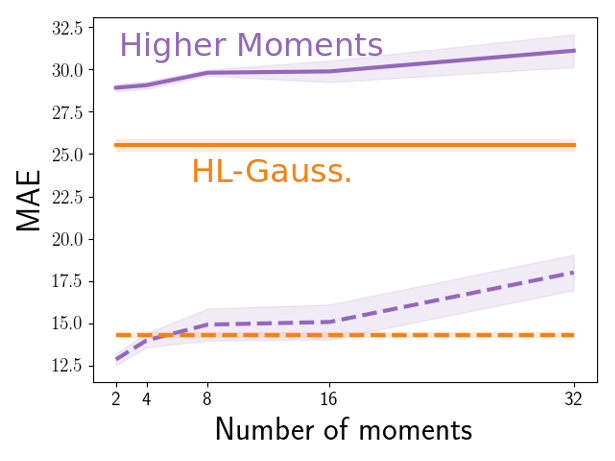}
\end{subfigure}

 \begin{subfigure}[b]{\figwidthtwo}
         \includegraphics[width=\textwidth]{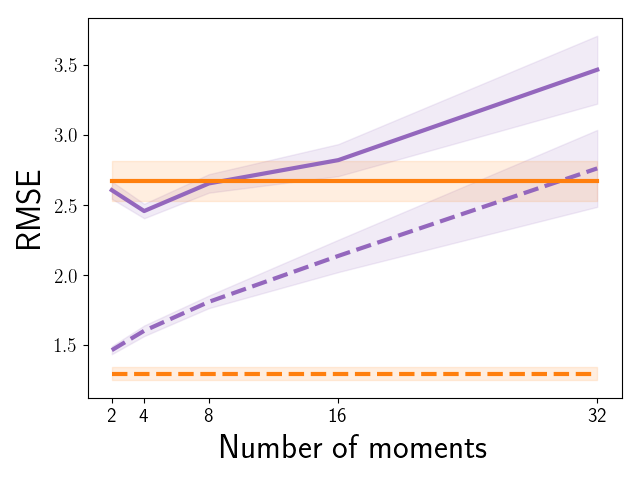}
\end{subfigure}
 \begin{subfigure}[b]{\figwidthtwo}
         \includegraphics[width=\textwidth]{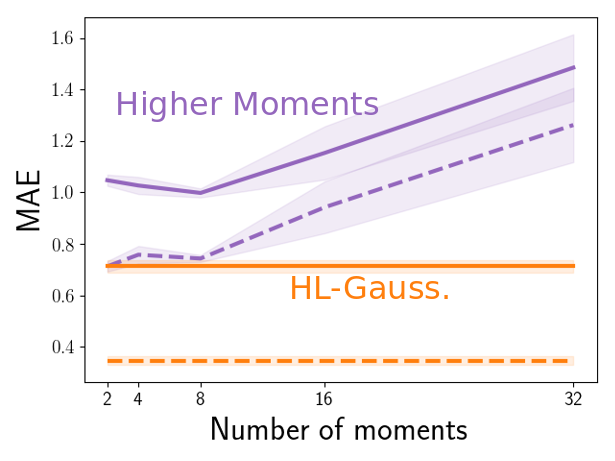}
\end{subfigure}
\caption{Higher Moments Network results. The rows show CT Scan, Song Year, Bike Sharing, and Pole respectively. The X-axis shows the number of moments (including the mean). Dotted and solid lines show train and test errors respectively. There was no substantial decrease in error when predicting higher moments. In some cases there was a slight reduction in error but the new model still performed worse than HL-Gauss in terms of Test MAE.}
\label{fig:moments_ctscan_appdx}
\end{figure*}

\clearpage
\subsection{Optimization Properties}

\begin{figure*}[!ht]
\centering
 \begin{subfigure}[b]{\figwidththree}
         \includegraphics[width=\textwidth]{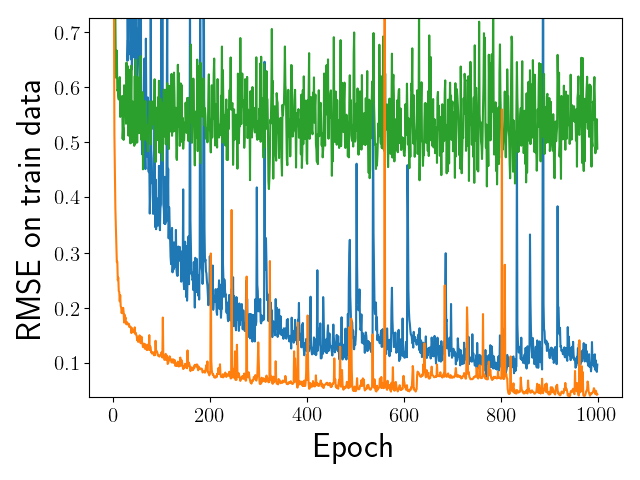}
\end{subfigure}
 \begin{subfigure}[b]{\figwidththree}
         \includegraphics[width=\textwidth]{figures/conv,ctscan,train_errors,mae.png}
\end{subfigure} 
 \begin{subfigure}[b]{\figwidththree}
         \includegraphics[width=\textwidth]{figures/conv,ctscan,gradients.png}
\end{subfigure}

 \begin{subfigure}[b]{\figwidththree}
         \includegraphics[width=\textwidth]{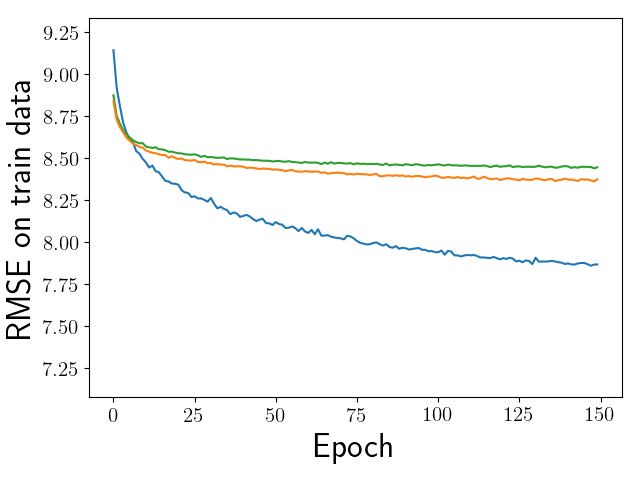}
\end{subfigure}
 \begin{subfigure}[b]{\figwidththree}
         \includegraphics[width=\textwidth]{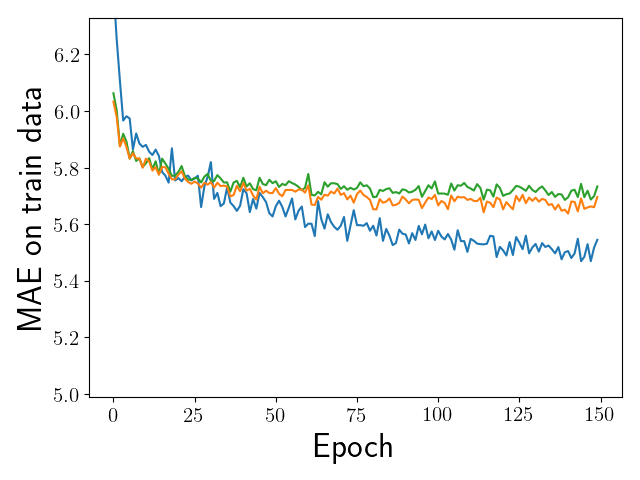}
\end{subfigure} 
 \begin{subfigure}[b]{\figwidththree}
         \includegraphics[width=\textwidth]{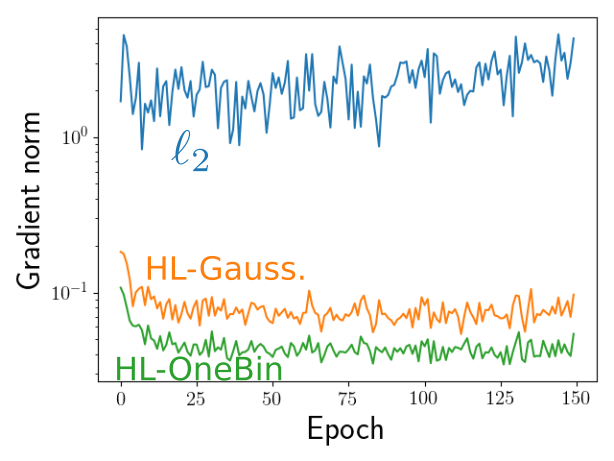}
\end{subfigure} 

 \begin{subfigure}[b]{\figwidththree}
         \includegraphics[width=\textwidth]{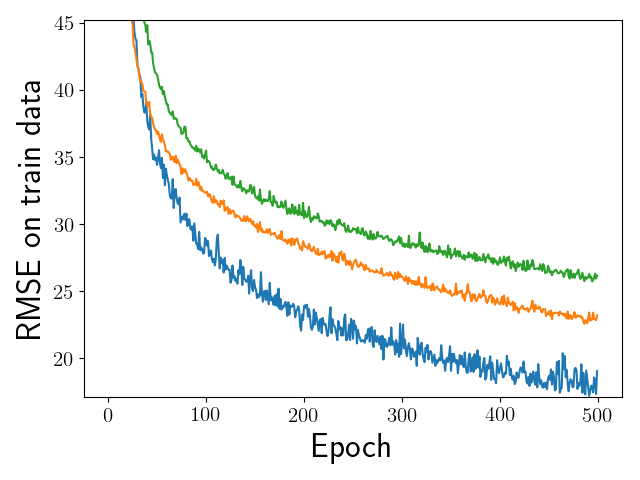}
\end{subfigure}
 \begin{subfigure}[b]{\figwidththree}
         \includegraphics[width=\textwidth]{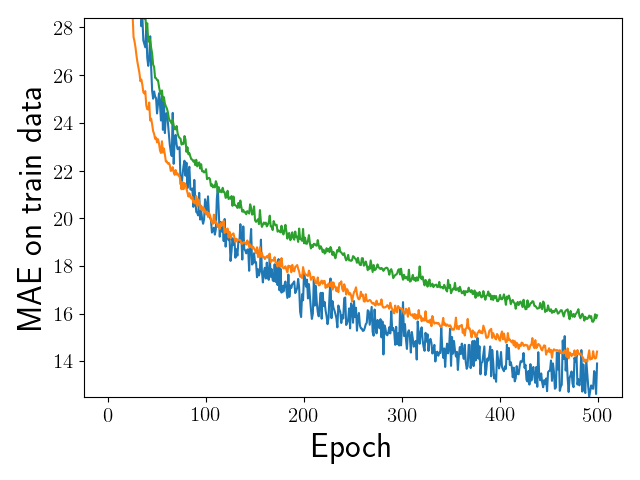}
\end{subfigure} 
 \begin{subfigure}[b]{\figwidththree}
         \includegraphics[width=\textwidth]{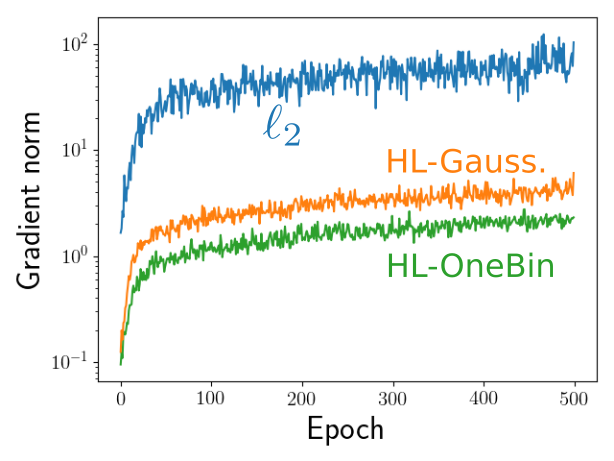}
\end{subfigure} 

 \begin{subfigure}[b]{\figwidththree}
         \includegraphics[width=\textwidth]{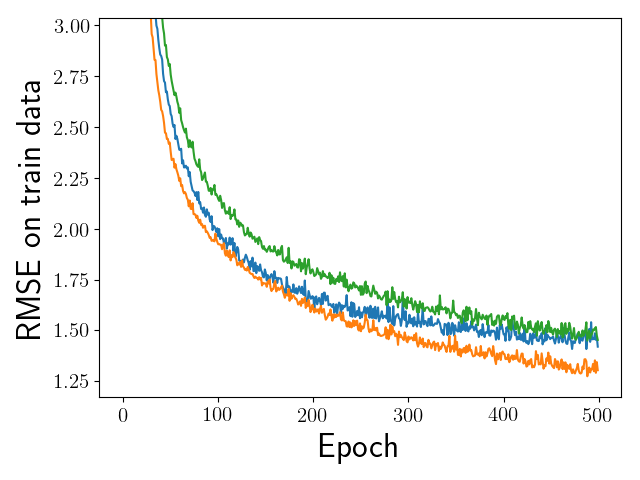}
\end{subfigure}
 \begin{subfigure}[b]{\figwidththree}
         \includegraphics[width=\textwidth]{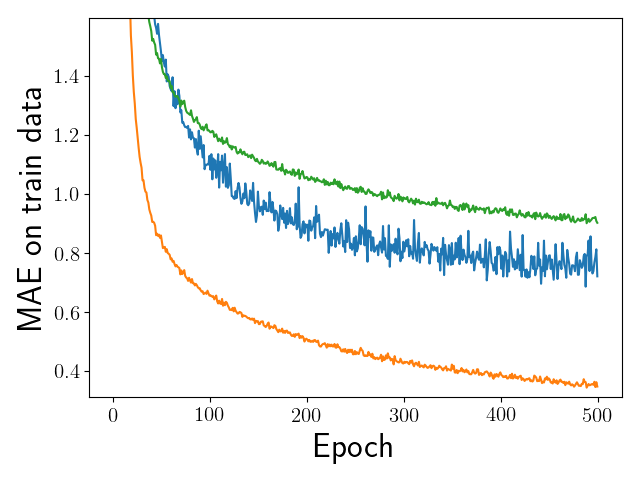}
\end{subfigure} 
 \begin{subfigure}[b]{\figwidththree}
         \includegraphics[width=\textwidth]{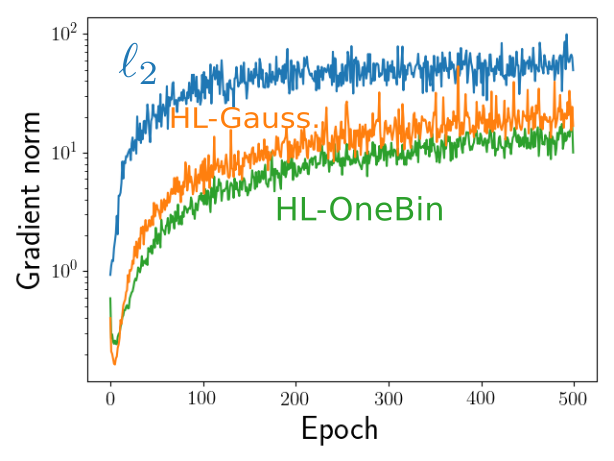}
\end{subfigure} 
\caption{Training process. The rows show CT Scan, Song Year, Bike Sharing, and Pole respectively. A logarithmic scale is used in the Y-axis of the rightmost plot. In all cases, training with HL-Gauss resulted in initial fast reduction in train MAE compared to $\ell_2$. The gradient norm of $\ell_2$ was highly varying through the training.}
\label{fig:conv_ctscan_appdx}
\end{figure*}

\begin{figure*}[!ht]
\centering
 \begin{subfigure}[b]{\figwidthtwo}
         \includegraphics[width=\textwidth]{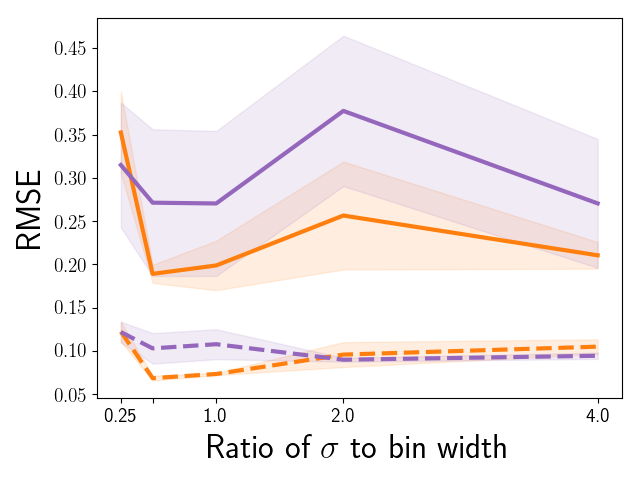}
\end{subfigure}
 \begin{subfigure}[b]{\figwidthtwo}
         \includegraphics[width=\textwidth]{figures/annealing,ctscan,mae.png}
\end{subfigure}

 \begin{subfigure}[b]{\figwidthtwo}
         \includegraphics[width=\textwidth]{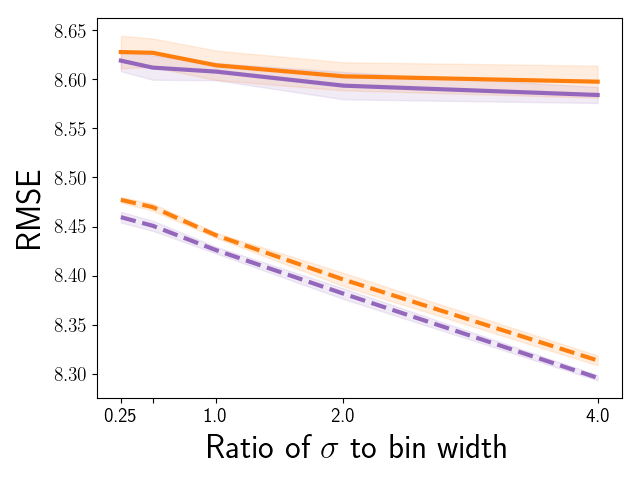}
\end{subfigure}
 \begin{subfigure}[b]{\figwidthtwo}
         \includegraphics[width=\textwidth]{figures/annealing,yearpred,mae.png}
\end{subfigure}

 \begin{subfigure}[b]{\figwidthtwo}
         \includegraphics[width=\textwidth]{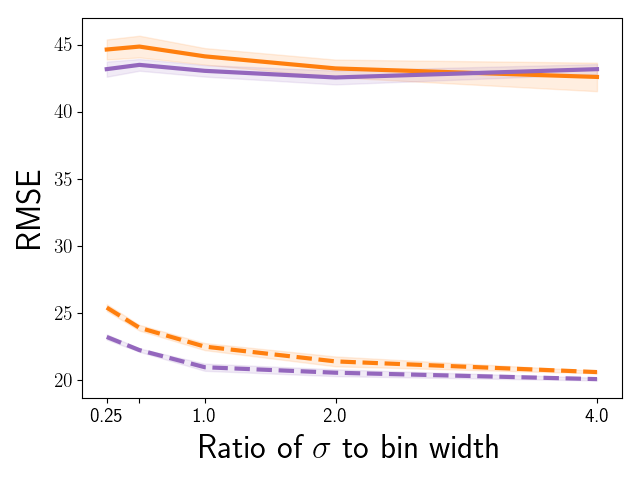}
\end{subfigure}
 \begin{subfigure}[b]{\figwidthtwo}
         \includegraphics[width=\textwidth]{figures/annealing,bike,mae.png}
\end{subfigure}

 \begin{subfigure}[b]{\figwidthtwo}
         \includegraphics[width=\textwidth]{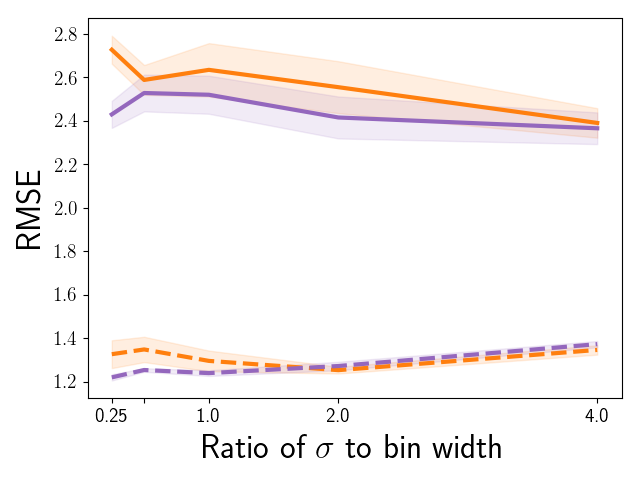}
\end{subfigure}
 \begin{subfigure}[b]{\figwidthtwo}
         \includegraphics[width=\textwidth]{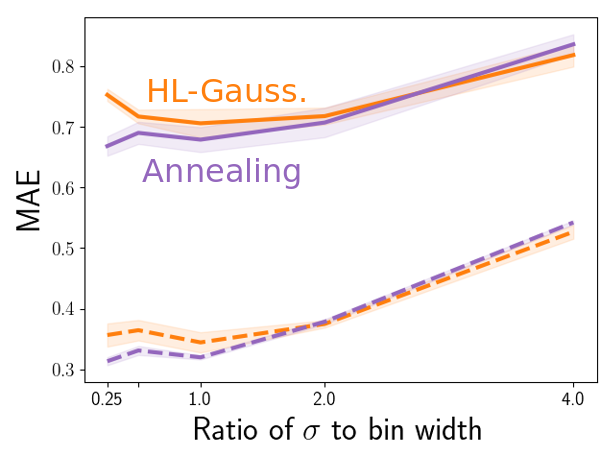}
\end{subfigure} 
\caption{Annealing results. The rows show CT Scan, Song Year, Bike Sharing, and Pole respectively. Dotted and solid lines show train and test errors. On three data sets, there was a small benefit in starting with a higher $\sigma$ and reducing it through the training.}
\label{fig:anneal_ctscan_appdx}
\end{figure*}

\clearpage
\subsection{Robustness to Corrupted Targets}

\begin{figure*}[!ht]
\centering
 \begin{subfigure}[b]{\figwidthtwo}
         \includegraphics[width=\textwidth]{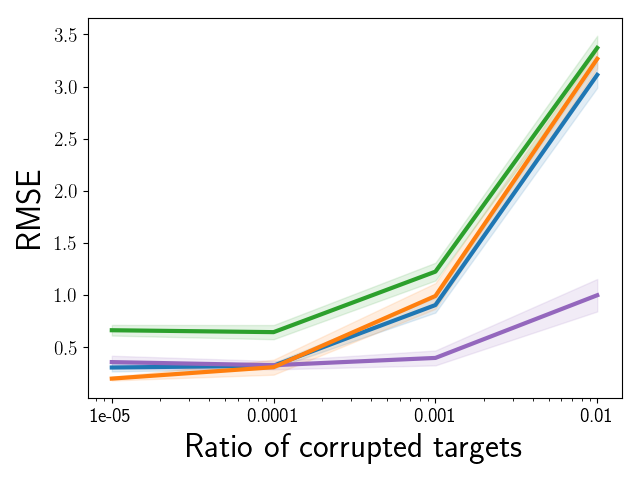}
\end{subfigure}
 \begin{subfigure}[b]{\figwidthtwo}
         \includegraphics[width=\textwidth]{figures/outliers,ctscan,mae.png}
\end{subfigure}

 \begin{subfigure}[b]{\figwidthtwo}
         \includegraphics[width=\textwidth]{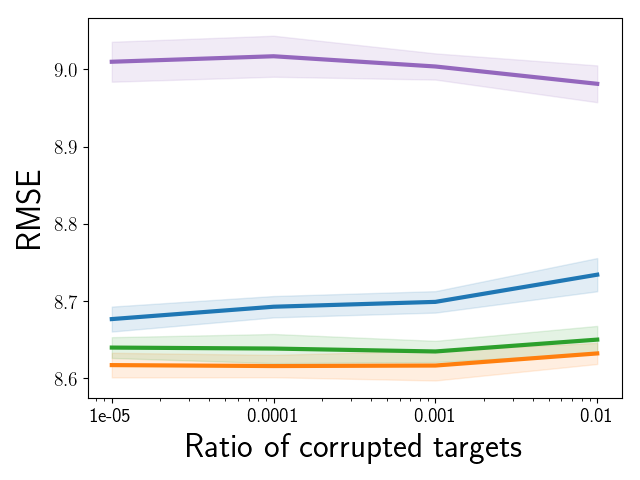}
\end{subfigure}
 \begin{subfigure}[b]{\figwidthtwo}
         \includegraphics[width=\textwidth]{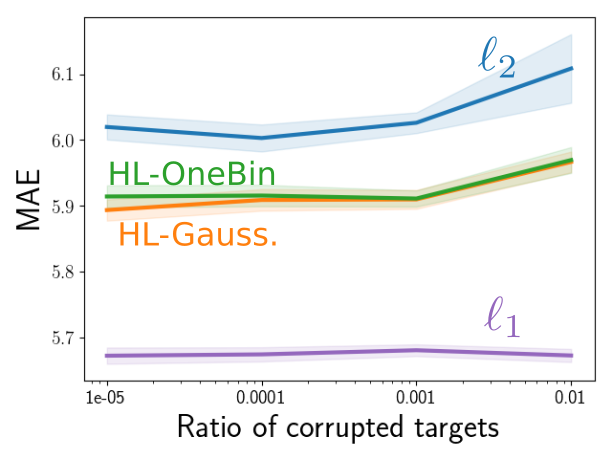}
\end{subfigure} 

 \begin{subfigure}[b]{\figwidthtwo}
         \includegraphics[width=\textwidth]{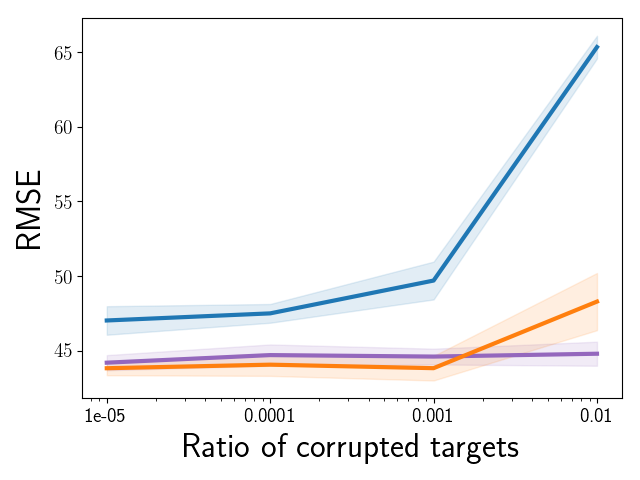}
\end{subfigure}
 \begin{subfigure}[b]{\figwidthtwo}
         \includegraphics[width=\textwidth]{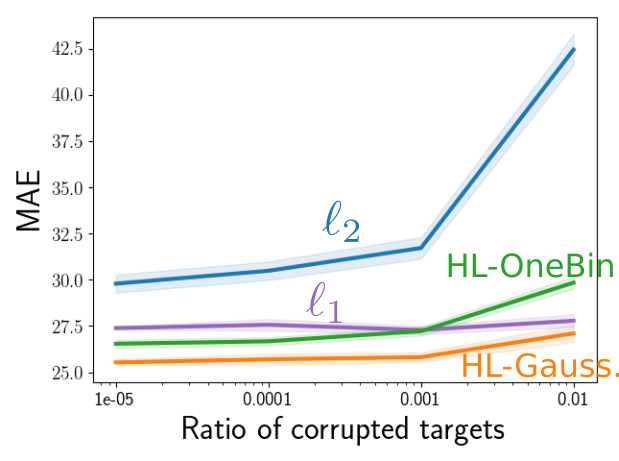}
\end{subfigure} 

 \begin{subfigure}[b]{\figwidthtwo}
         \includegraphics[width=\textwidth]{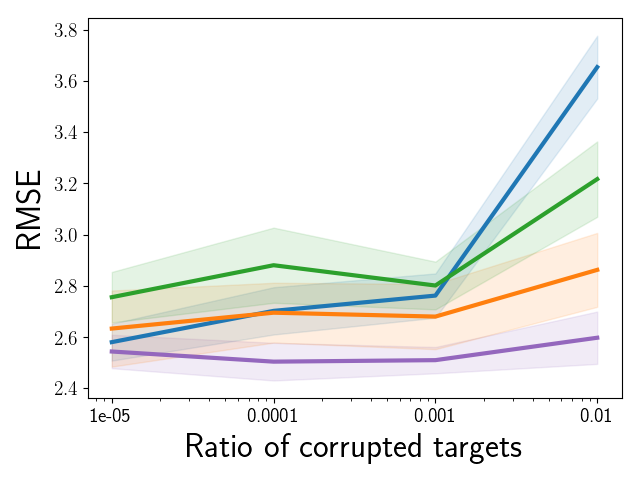}
\end{subfigure}
 \begin{subfigure}[b]{\figwidthtwo}
         \includegraphics[width=\textwidth]{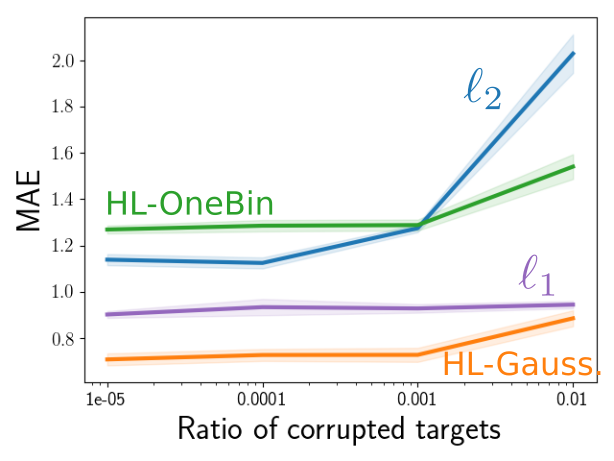}
\end{subfigure} 
\caption{Corrupted targets results. The rows show CT Scan, Song Year, Bike Sharing, and Pole respectively. The lines show test errors. The $\ell_1$ loss was robust to corrupted targets, and the $\ell_2$ loss was affected the most.}
\label{fig:outliers_ctscan_appdx}
\end{figure*}

\clearpage
\subsection{Sensitivity to Input Perturbations}

\begin{figure*}[!ht]
\centering
 \begin{subfigure}[b]{\figwidthtwo}
         \includegraphics[width=\textwidth]{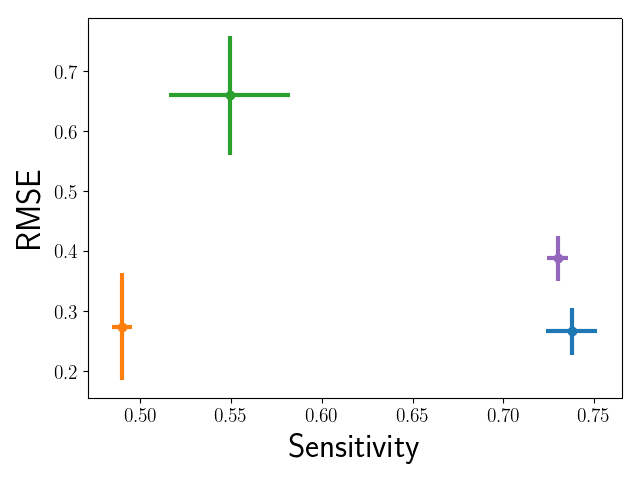}
\end{subfigure}
 \begin{subfigure}[b]{\figwidthtwo}
         \includegraphics[width=\textwidth]{figures/sensitivity,ctscan,mae.png}
\end{subfigure}

 \begin{subfigure}[b]{\figwidthtwo}
         \includegraphics[width=\textwidth]{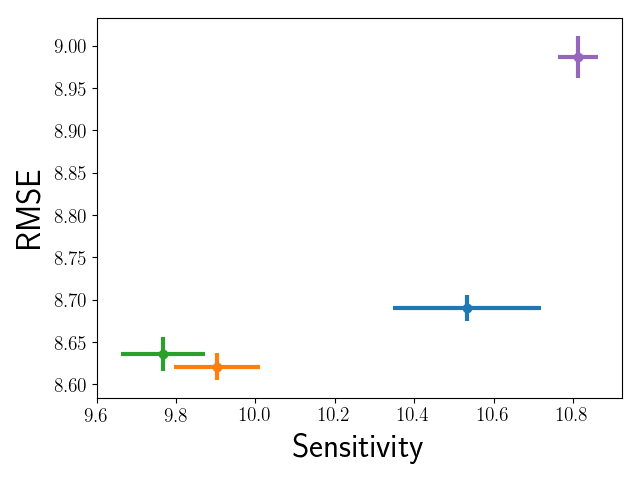}
\end{subfigure}
 \begin{subfigure}[b]{\figwidthtwo}
         \includegraphics[width=\textwidth]{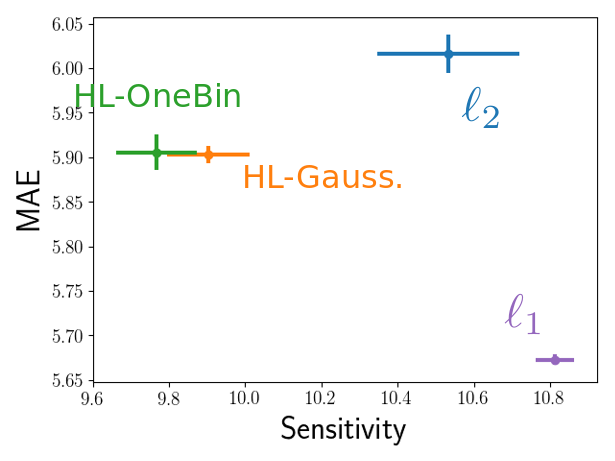}
\end{subfigure} 

 \begin{subfigure}[b]{\figwidthtwo}
         \includegraphics[width=\textwidth]{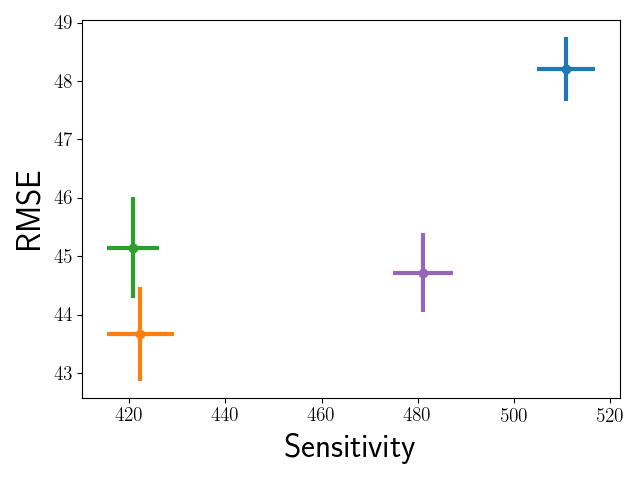}
\end{subfigure}
 \begin{subfigure}[b]{\figwidthtwo}
         \includegraphics[width=\textwidth]{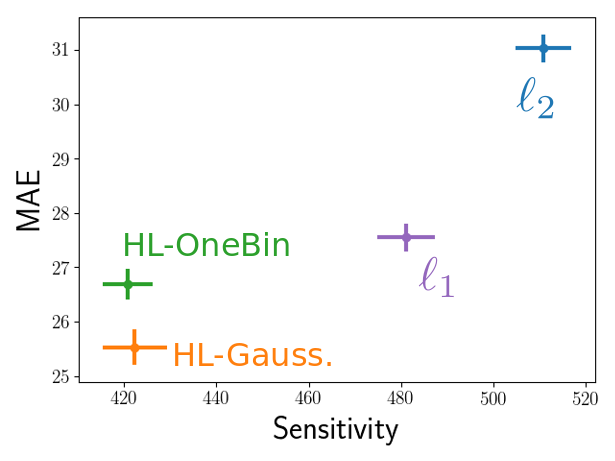}
\end{subfigure} 

 \begin{subfigure}[b]{\figwidthtwo}
         \includegraphics[width=\textwidth]{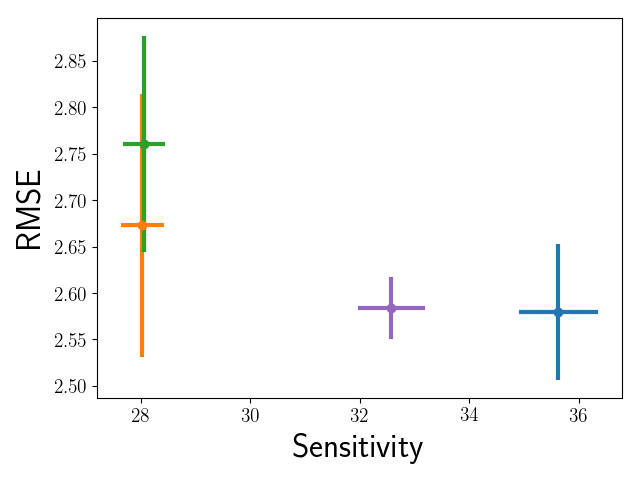}
\end{subfigure}
 \begin{subfigure}[b]{\figwidthtwo}
         \includegraphics[width=\textwidth]{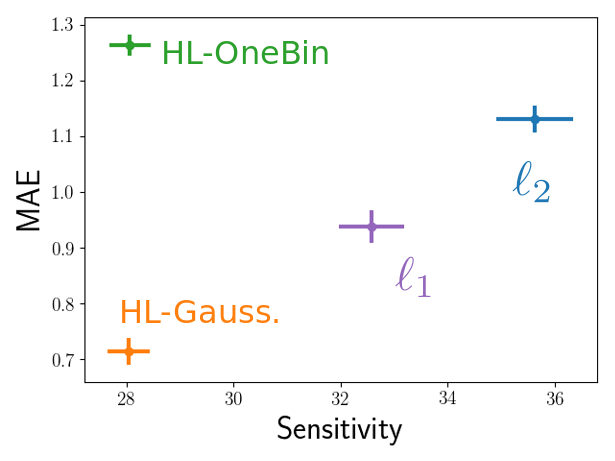}
\end{subfigure} 
\caption{Sensitivity results. The rows show CT Scan, Song Year, Bike Sharing, and Pole respectively. The horizontal axis shows the sensitivity of the model's output to input perturbations (left means less sensitive) and the vertical axis shows the test error (lower is better). HL-Gauss showed less sensitivity and lower test errors than $\ell_1$ and $\ell_2$.}
\label{fig:sensitivity_ctscan_appdx}
\end{figure*}

\end{document}